\documentclass[twoside,11pt]{article}

\usepackage{jmlr2e}
\usepackage{enumerate}
\usepackage{algorithm}
\usepackage{algorithmic}
\usepackage{pslatex}
\usepackage{amsopn}
\usepackage{bm}
\usepackage{epsf}
\usepackage{multirow}
\usepackage{wrapfig}
\usepackage{graphicx}
\usepackage{subfigure}
\usepackage{url}
\usepackage{picinpar}
\usepackage{latexsym}
\usepackage{amsfonts}
\usepackage{makecell}
\usepackage{mathrsfs}

\usepackage[english]{babel}

\newcommand{\xv}{\mathbf{x}}

\newcommand{\yv}{\mathbf{y}}

\newcommand{\zv}{\mathbf{z}}
\newcommand{\zvbar}{\bar{\zv}}
\newcommand{\Zv}{\mathbf{Z}}

\newcommand{\Vv}{\mathbf{V}}

\newcommand{\wv}{\mathbf{w}}
\newcommand{\Wv}{\mathbf{W}}

\newcommand{\ud}{\mathrm{d}}

\newcommand{\sign}{\mathrm{sign}}
\newcommand{\risk}{\mathcal{R}}
\newcommand{\etav}{\boldsymbol \eta}
\newcommand{\muv}{\boldsymbol \mu}
\newcommand{\Sigmav}{\boldsymbol \Sigma}

\newcommand{\xiv}{\boldsymbol \xi}
\newcommand{\alphav}{\boldsymbol \alpha}
\newcommand{\betav}{\boldsymbol \beta}
\newcommand{\thetav}{\boldsymbol \theta}
\newcommand{\Thetav}{\boldsymbol \Theta}
\newcommand{\kappav}{\boldsymbol \kappa}

\newcommand{\omegav}{\boldsymbol \omega}

\newcommand{\lambdav}{\boldsymbol \lambda }

\newcommand{\Phiv}{\boldsymbol \Phi }

\newcommand{\ep}{\mathbb{E}}

\newcommand{\IG}{\mathcal{I}\!\mathcal{G}}
\newcommand{\GIG}{\mathcal{G}\!\mathcal{I}\!\mathcal{G}}

\newcommand{\data}{\mathcal{D}}

\def\indicator{{\mathbb I}}

\newcommand{\argmax}{\operatornamewithlimits{argmax}}

\jmlrheading{14}{2013}{1-10}{1/13; Revised ?}{?}{Jun Zhu, Ning Chen, Hugh Perkins, and Bo Zhang}

\ShortHeadings{Gibbs Max-margin Topic Models with Data Augmentation}{Zhu, Chen, Perkins, and Zhang}
\firstpageno{1}

\begin{document}

\title{Gibbs Max-margin Topic Models with Data Augmentation}

\author{\name Jun Zhu \email dcszj@mail.tsinghua.edu.cn \\
        \name Ning Chen \email ningchen@mail.tsinghua.edu.cn \\
        \name Hugh Perkins \email ngls11@mails.tsinghua.edu.cn \\
        \name Bo Zhang \email dcszb@mail.tsinghua.edu.cn \\
        \addr State Key Lab of Intelligent Technology and Systems \\
        \addr Tsinghua National Lab for Information Science and Technology \\
        \addr Department of Computer Science and Technology \\
        \addr Tsinghua University \\
        Beijing, 100084, China
}

\editor{?}

\maketitle

\begin{abstract}
Max-margin learning is a powerful approach to building classifiers and structured output predictors. Recent work on max-margin supervised topic models has successfully integrated it with Bayesian topic models to discover discriminative latent semantic structures and make accurate predictions for unseen testing data. However, the resulting learning problems are usually hard to solve because of the non-smoothness of the margin loss. Existing approaches to building max-margin supervised topic models rely on an iterative procedure to solve multiple latent SVM subproblems with additional mean-field assumptions on the desired posterior distributions. This paper presents an alternative approach by defining a new max-margin loss. Namely, we present Gibbs max-margin supervised topic models, a latent variable Gibbs classifier to discover hidden topic representations for various tasks, including classification, regression and multi-task learning. Gibbs max-margin supervised topic models minimize an expected margin loss, which is an upper bound of the existing margin loss derived from an expected prediction rule. By introducing augmented variables and integrating out the Dirichlet variables analytically by conjugacy, we develop simple Gibbs sampling algorithms with no restricting assumptions and no need to solve SVM subproblems. Furthermore, each step of the ``augment-and-collapse" Gibbs sampling algorithms has an analytical conditional distribution, from which samples can be easily drawn. Experimental results demonstrate significant improvements on time efficiency. The classification performance is also significantly improved over competitors on binary, multi-class and multi-label classification tasks.
\end{abstract}


\section{Introduction}

As the availability and scope of complex data increase, developing statistical tools to discover latent structures and reveal hidden explanatory factors has become a major theme in statistics and machine learning. Topic models represent one type of such useful tools to discover latent semantic structures that are organized in an automatically learned latent topic space, where each topic (i.e., a coordinate of the latent space) is a unigram distribution over the terms in a vocabulary. Due to its nice interpretability and extensibility, the Bayesian formulation of topic models~\citep{Blei:03} has motivated substantially broad extensions and applications to various fields, such as document analysis, image categorization~\citep{Fei-Fei:05}, and network data analysis~\citep{airoldi2008mixed}. Besides discovering latent topic representations, many models usually have a goal to make good predictions, such as relational topic models~\citep{Chang:RTM09} whose major goal is to make accurate predictions on the link structures of a document network. Another example is supervised topic models, our focus in this paper, which learn a prediction model for regression and classification tasks. As supervising information (e.g., user-input rating scores for product reviews) gets easier to obtain on the Web, developing supervised latent topic models has attracted a lot of attention. Both maximum likelihood estimation (MLE) and max-margin learning have been applied to learn supervised topic models. Different from the MLE-based approaches~\citep{Blei:sLDA07}, which define a normalized likelihood model for response variables, max-margin supervised topic models, such as maximum entropy discrimination LDA (MedLDA)~\citep{Zhu:jmlr12}, directly minimize a margin-based loss derived from an expected (or averaging) prediction rule.


By performing discriminative learning, max-margin supervised topic models can discover predictive latent topic representations and have shown promising performance in various prediction tasks, such as text document categorization~\citep{Zhu:jmlr12} and image annotation~\citep{yang2010uai}. However, their learning problems are generally hard to solve due to the non-smoothness of the margin-based loss function. Most existing solvers rely on a variational approximation scheme with strict mean-field assumptions on posterior distributions, and they normally need to solve multiple latent SVM subproblems in an EM-type iterative procedure. By showing a new interpretation of MedLDA as a regularized Bayesian inference method, the recent work~\citep{Zhu:nips12} successfully developed Monte Carlo methods for such max-margin topic models, with a weaker mean-field assumption. Though the prediction performance is improved because of more accurate inference, the Monte Carlo methods still need to solve multiple SVM subproblems. Thus, their efficiency could be limited as learning SVMs is normally computationally demanding. Furthermore, due to the dependence on SVM solvers, it is not easy to parallelize these algorithms for large-scale data analysis tasks, although substantial efforts have been made to develop parallel Monte Carlo methods for unsupervised topic models~\citep{Newman08distributedinference,Smola:vldb10,Ahmed:wsdm12}.

This paper presents Gibbs MedLDA, an alternative formulation of max-margin supervised topic models, for which we can develop simple and efficient inference algorithms. Technically, instead of minimizing the margin loss of an expected (averaging) prediction rule as adopted in existing max-margin topic models, Gibbs MedLDA minimizes the expected margin loss of many latent prediction rules, of which each rule corresponds to a configuration of topic assignments and the prediction model, drawn from a post-data posterior distribution. Theoretically, the expected margin loss is an upper bound of the existing margin loss of an expected prediction rule. Computationally, although the expected margin loss can be hard in developing variational algorithms, we successfully develop simple and fast collapsed Gibbs sampling algorithms without any restricting assumptions on the posterior distribution and without solving multiple latent SVM subproblems. Each of the sampling substeps has a closed-form conditional distribution, from which a sample can be efficiently drawn. Our algorithms represent an extension of the classical ideas of data augmentation~\citep{Dempster1977,Tanner:1987,DykMeng2001} and its recent developments in learning fully observed max-margin classifiers \citep{Polson:BA11} to learn the sophisticated latent topic models. We further generalize the ideas to develop a Gibbs MedLDA regression model and a multi-task Gibbs MedLDA model, and we also develop efficient collapsed Gibbs sampling algorithms for them with data augmentation. Empirical results on real data sets demonstrate significant improvements in time efficiency. The classification performance is also significantly improved in binary, multi-class, and multi-label classification tasks.

The rest of the paper is structured as follows. Section~\ref{section:related-work} summarizes some related work. Section~\ref{section:MedLdA} reviews MedLDA and its EM-type algorithms. Section~\ref{section:GibbsMedLDA} presents Gibbs MedLDA and its sampling algorithms for classification. Section~\ref{section:Extensions} presents two extensions of Gibbs MedLDA for regression and multi-task learning. Section~\ref{section:experiments} presents empirical results. Finally, Section~\ref{section:conclusions} concludes and discusses future directions.

\section{Related Work}\label{section:related-work}

Max-margin learning has been very successful in building classifiers~\citep{Vapnik:95} and structured output prediction models~\citep{Taskar:03} in the last decade. Recently, research on learning max-margin models in the presence of latent variable models has received increasing attention because of the promise of using latent variables to capture the underlying structures of the complex problems. Deterministic approaches~\citep{Yu:latentSVM09} fill in the unknown values of the hidden structures by using some estimates (e.g., MAP estimates), and then a max-margin loss function is defined with the filled-in hidden structures, while probabilistic approaches aim to infer an entire distribution profile of the hidden structures given evidence and some prior distribution, following the Bayes' way of thinking. Though the former is powerful, we focus on Bayesian approaches, which can naturally incorporate prior beliefs, maintain the entire distribution profile of latent structures, and be extensible to nonparametric methods. One representative work along this line is maximum entropy discrimination (MED)~\citep{jaakkola1999maximum,Jebara:thesis}, which learns a distribution of model parameters given a set of labeled training data.

MedLDA~\citep{Zhu:jmlr12} is one extension of MED to infer hidden topical structures from data and MMH (max-margin Harmoniums)~\citep{Chen:PAMI12} is another extension that infers the hidden semantic features from multi-view data. Along similar lines, recent work has also successfully developed nonparametric Bayesian max-margin models, such as infinite SVMs (iSVM)~\citep{Zhu:ICML11} for discovering clustering structures when building SVM classifiers and infinite latent SVMs (iLSVM)~\citep{Zhu:nips11} for automatically learning predictive features for SVM classifiers. Both iSVM and iLSVM can automatically resolve the model complexity (e.g., the number of components in a mixture model or the number of latent features in a factor analysis model). The nonparametric Bayesian max-margin ideas have been proven to be effective in dealing with more challenging problems, such as link prediction in social networks~\citep{Zhu:ICML12} and low-rank matrix factorization for collaborative recommendation~\citep{Xu:NIPS12,Xu:ICML13}.

One common challenge of these Bayesian max-margin latent variable models is on the posterior inference, which is normally intractable. Almost all the existing work adopts a variational approximation scheme, with some mean-field assumptions. Very little research has been done on developing Monte Carlo methods, except the work~\citep{Zhu:nips12} which still makes mean-field assumptions. The work in the present paper provides a novel way to formulate Bayesian max-margin models and we show that these new formulations can have very simple and efficient Monte Carlo inference algorithms without making restricting assumptions. The key step to deriving our algorithms is a data augmentation formulation of the expected margin-based loss. Other work on inferring the posterior distributions of latent variables includes max-margin min-entropy models~\citep{Miller:2012} which learn a single set of model parameters, different from our focus of inferring the model posterior distribution.

Data augmentation refers to methods of augmenting the observed data so as to make it easy to analyze with an iterative optimization or sampling algorithm. For deterministic algorithms, the technique has been popularized in the statistics community by the seminal expectation-maximization (EM) algorithm~\citep{Dempster1977} for maximum likelihood estimation (MLE) with missing values. For stochastic algorithms, the technique has been popularized in statistics by Tanner and Wong's data augmentation algorithm for posterior sampling~\citep{Tanner:1987} and in physics by Swendsen and Wang's sampling algorithms for Ising and Potts models~\citep{SwendsenWang:87}. When using the idea to solve estimation or posterior inference problems, the key step is to find a set of augmented variables, conditioned on which the distribution of our models can be easily sampled. The speed of mixing or convergence is another important concern when designing a data augmentation method. While the conflict between simplicity and speed is a common phenomenon with many standard augmentation schemes, some work has demonstrated that with more creative augmentation schemes it is possible to construct EM-type algorithms~\citep{MengDyk:DA97} or Markov Chain Monte Carlo methods (known as slice sampling)~\citep{Neal:97} that are both fast and simple. We refer the readers to~\citep{DykMeng2001} for an excellent review of the broad literature of data augmentation and an effective search strategy for selecting good augmentation schemes.

For our focus on max-margin classifiers, the recent work~\citep{Polson:BA11} provides an elegant data augmentation formulation for support vector machines (SVM) with fully observed input data, which leads to analytical conditional distributions that are easy to sample from and fast to mix. Our work in the present paper builds on the method of Polson et al. and presents a successful implementation of data augmentation to deal with the challenging posterior inference problems of Bayesian max-margin latent topic models. Our approach can be generalized to deal with other Bayesian max-margin latent variable models, e.g., max-margin matrix factorization~\citep{Xu:ICML13}, as reviewed above.

Finally, some preliminary results were presented in a conference paper~\citep{Zhu:GibbsMedLDA13}. This paper presents a full extension.

\section{MedLDA}\label{section:MedLdA}

We begin with a brief overview of MedLDA and its learning algorithms, which motivate our developments of Gibbs MedLDA.

\subsection{MedLDA: a Regularized Bayesian Model}
We consider binary classification with a labeled training set $\mathcal{D} = \{ (\wv_d, y_d) \}_{d=1}^{D}$, where the response variable $Y$ takes values from the output space $\mathcal{Y}=\{-1, +1\}$. Basically, MedLDA consists of two parts --- an LDA model for describing input documents $\Wv = \{\wv_d\}_{d=1}^D$,~where $\wv_d=\{w_{dn}\}_{n=1}^{N_d}$ denote the words appearing in document $d$, and an expected classifier for considering the supervising signal $\yv =\{ y_d \}_{d=1}^D$. Below, we introduce each of them in turn.

{\bf LDA}: Latent Dirichlet allocation (LDA)~\citep{Blei:03} is a hierarchical Bayesian model that posits each document as an admixture of $K$ topics, where each topic $\Phiv_k$ is a multinomial distribution over a $V$-word vocabulary. For document $d$, the generating process can be described as
\begin{enumerate}
\item draw a topic proportion $\thetav_d \sim \mathrm{Dir}(\alphav)$
\item for each word $n~(1 \leq n \leq N_d)$:
\begin{enumerate}
\item draw a topic assignment\footnote{A $K$-dimension binary vector with only one nonzero entry.} $z_{dn} | \thetav_d \sim \mathrm{Mult}(\thetav_d)$
\item draw the observed word $w_{dn} | z_{dn}, \Phiv  \sim \mathrm{Mult}(\Phiv_{z_{dn}})$
\end{enumerate}
\end{enumerate}
where $\mathrm{Dir}(\cdot)$ is a Dirichlet distribution; $\mathrm{Mult}(\cdot)$ is multinomial; and $\Phiv_{z_{dn}}$ denotes the topic selected by the non-zero entry of $z_{dn}$. For a fully-Bayesian LDA, the topics are random samples drawn from a prior, e.g., $\Phiv_k \sim \textrm{Dir}(\betav)$.

Given a set of documents $\Wv$, we let $\zv_d = \{z_{dn}\}_{n=1}^{N_d}$ denote the set of topic assignments for document $d$ and let $\Zv = \{\zv_d\}_{d=1}^D$ and $\Thetav = \{\thetav_d\}_{d=1}^D$ denote all the topic assignments and mixing proportions for the whole corpus, respectively. Then, LDA infers the posterior distribution using Bayes' rule $$ p(\Thetav, \Zv, \Phiv | \Wv) = \frac{p_0(\Thetav, \Zv, \Phiv) p(\Wv| \Zv, \Phiv)}{ p(\Wv)},$$ where $p_0(\Thetav, \Zv, \Phiv) = \prod_{k=1}^K p_0( \Phiv_k | \betav ) \prod_{d=1}^D p_0(\thetav_d | \alphav) \prod_{n=1}^{N_d} p(z_{dn} | \thetav_d)$ according to the generating process of LDA; and $p(\Wv)$ is the marginal evidence. We can show that the posterior distribution by Bayes' rule is the solution of an information theoretical optimization problem
\setlength\arraycolsep{1pt} \begin{eqnarray}\label{problem:RegBayesLDA}
\min_{q(\Thetav, \Zv, \Phiv) } &&  \mathrm{KL}\left[ q(\Thetav, \Zv, \Phiv) \Vert p_0(\Thetav, \Zv, \Phiv) \right] - \ep_q\left[ \log p(\Wv |  \Zv, \Phiv) \right] \nonumber \\
\mathrm{s.t.}:&& q(\Thetav, \Zv, \Phiv) \in \mathcal{P},
\end{eqnarray}
where $\mathrm{KL}(q||p)$ is the Kullback-Leibler divergence and $\mathcal{P}$ is the space of probability distributions with an appropriate dimension. In fact, if we add the constant $\log p(\Wv)$ to the objective, the problem is the minimization of the KL-divergence $\textrm{KL}(q(\Thetav, \Zv, \Phiv) \Vert p(\Thetav, \Zv, \Phiv | \Wv))$, whose solution is the desired posterior distribution by Bayes' rule. One advantage of this variational formulation of Bayesian inference is that it can be naturally extended to include some regularization terms on the desired post-data posterior distribution $q$. This insight has been taken to develop regularized Bayesian inference (RegBayes)~\citep{Zhu:nips11}, a computational framework for doing Bayesian inference with posterior regularization\footnote{Posterior regularization was first used in~\citep{ganchev2010posterior} for maximum likelihood estimation and was later extended in~\citep{Zhu:nips11} to Bayesian and nonparametric Bayesian methods.}. As shown in~\citep{Zhu:nips12} and detailed below, MedLDA is one example of RegBayes models. Moreover, as we shall see in Section~\ref{section:GibbsMedLDA}, our Gibbs max-margin topic models follow this similar idea too.

{\bf Expected Classifier}: Given a training set $\data$, an expected (or averaging) classifier chooses a posterior distribution $q(h|\data)$ over a hypothesis space $\mathcal{H}$ of classifiers such that the $q$-weighted (expected) classifier $$h_q(\wv) = \sign ~\ep_q[h(\wv)]$$ will have the smallest possible risk. MedLDA follows this principle to learn a posterior distribution $q(\etav, \Thetav, \Zv, \Phiv | \data)$ such that the expected classifier
\begin{eqnarray}\label{eqn:prediction_rule}
\hat{y} = \sign~ F(\wv)
\end{eqnarray}
has the smallest possible risk, approximated by the training error $\risk_\data(q) = \sum_{d=1}^D \indicator(\hat{y}_d \neq y_d)$. The discriminant function is defined as
\begin{eqnarray}\label{eqn:linear expected discrimination function}
F(\wv) = \mathbb{E}_{q(\etav, \zv|\data)}[ F(\etav, \zv; \wv) ],~F(\etav, \zv; \wv) = \etav^\top \bar{\zv} \nonumber
\end{eqnarray}
where $\bar{\zv}$ is the average topic assignment associated with the words $\wv$, a vector with element $\bar{z}_k = \frac{1}{N}\sum_{n=1}^{N} z_{n}^k$, and $\etav$ is the classifier weights. 
Note that the expected classifier and the LDA likelihood are coupled via the latent topic assignments $\Zv$. The strong coupling makes it possible for MedLDA to learn a posterior distribution that can describe the observed words well and make accurate predictions.

{\bf Regularized Bayesian Inference}: To integrate the above two components for hybrid learning, 
MedLDA regularizes the properties of the topic representations by imposing the following max-margin constraints derived from the classifier~(\ref{eqn:prediction_rule}) to a standard LDA inference problem~(\ref{problem:RegBayesLDA})
\begin{eqnarray}
y_d F(\wv_d) \geq \ell - \xi_d, ~\forall d,
\end{eqnarray}
where $\ell$ ($\geq \! 1$) is the cost of making a wrong prediction; and $\xiv = \{\xi_d\}_{d=1}^D$ are non-negative slack variables for inseparable cases. Let $\mathcal{L}(q) = \mathrm{KL}(q||p_0(\etav, \Thetav, \Zv, \Phiv)) - \mathbb{E}_q[\log p(\Wv | \Zv, \Phiv)]$ be the objective for doing standard Bayesian inference with the classifier $\etav$ and $p_0(\etav, \Thetav, \Zv, \Phiv) = p_0(\etav) p_0(\Thetav, \Zv, \Phiv)$. MedLDA solves the regularized Bayesian inference~\citep{Zhu:nips11} problem
\begin{eqnarray}\label{problem:MedLDA}
\min_{q(\etav, \Thetav, \Zv, \Phiv) \in \mathcal{P}, \xiv } && \mathcal{L}\left( q(\etav, \Thetav, \Zv, \Phiv) \right) + 2 c \sum_{d=1}^D \xi_d \\
\textrm{s.t.:} && ~ y_d F(\wv_d) \geq \ell - \xi_d,~\xi_d \geq 0, \forall d, \nonumber
\end{eqnarray}
where the margin constraints directly regularize the properties of the post-data distribution and $c$ is the positive regularization parameter. Equivalently, MedLDA solves the unconstrained problem\footnote{If not specified, $q$ is subject to the constraint $q \in \mathcal{P}$.}
\setlength\arraycolsep{-3pt}\begin{eqnarray}\label{problem:MedLDA_Unconstrained}
&&\min_{q(\etav, \Thetav, \Zv, \Phiv)  }  \mathcal{L}\left(q(\etav, \Thetav, \Zv, \Phiv)\right) + 2 c \mathcal{R}\left( q(\etav, \Thetav, \Zv, \Phiv) \right),
\end{eqnarray}
where $\mathcal{R}(q) = \sum_{d=1}^D \max(0, \ell - y_d F(\wv_d))$ is the hinge loss that upper-bounds the training error $\risk_\data(q)$ of the expected classifier~(\ref{eqn:prediction_rule}). Note that the constant $2$ is included simply for convenience.

\subsection{Existing Iterative Algorithms}

Since it is difficult to solve problem~(\ref{problem:MedLDA}) or~(\ref{problem:MedLDA_Unconstrained}) directly because of the non-conjugacy (between priors and likelihood) and the max-margin constraints, corresponding to a non-smooth posterior regularization term in~(\ref{problem:MedLDA_Unconstrained}), both variational and Monte Carlo methods have been developed for approximate solutions. It can be shown that the variational method~\citep{Zhu:jmlr12} is a coordinate descent algorithm to solve problem~(\ref{problem:MedLDA_Unconstrained}) with the fully-factorized assumption that $$q(\etav, \Theta, \Zv, \Phiv) = q(\etav) \left( \prod_{d=1}^D q(\thetav_d) \prod_{n=1}^{N_d} q(z_{dn}) \right) \prod_k q(\Phiv_k);$$ while the Monte Carlo methods~\citep{Zhu:nips12} make a weaker assumption that $$q(\etav, \Theta, \Zv, \Phiv) = q(\etav)q(\Theta, \Zv, \Phiv).$$
All these methods have a similar EM-type iterative procedure, which solves many latent SVM subproblems, as outlined below.

{\bf Estimate $q(\etav)$}: Given $q(\Thetav, \Zv, \Phiv)$, we solve problem~(\ref{problem:MedLDA_Unconstrained}) with respect to $q(\etav)$. In the equivalent constrained form, this step solves
\setlength\arraycolsep{1pt} \begin{eqnarray}\label{problem:MedLDA_subEta}
 \min_{q(\etav), \xiv} ~&& ~ \mathrm{KL}\left( q(\etav) \Vert p_0(\etav) \right) + 2c\sum_{d=1}^D \xi_d \\
\mathrm{s.t.}: ~&& ~y_d \mathbb{E}_q[\etav]^\top \ep_q[\bar{\zv}_d] \geq \ell - \xi_d,~\xi_d \geq 0, \forall d. \nonumber
\end{eqnarray}
This problem is convex and can be solved with Lagrangian methods. Specifically, let $\mu_d$ be the Lagrange multipliers, one per constraint. When the prior $p_0(\etav)$ is the commonly used standard normal distribution, we have the optimum solution $q(\etav)=\mathcal{N}(\kappav, I)$, 
where $\kappav = \sum_{d=1}^D y_d \mu_d \ep_q[\bar{\zv}_d]$. 
It can be shown that the dual problem of (\ref{problem:MedLDA_subEta}) is the dual of a standard binary linear SVM and we can solve it or its primal form efficiently using existing high-performance SVM learners. We denote the optimum solution of this problem by $(q^\ast(\etav), \kappav^\ast, \xiv^\ast, \muv^\ast)$.

{\bf Estimate $q(\Thetav, \Zv, \Phiv)$}: Given $q(\etav)$, we solve problem~(\ref{problem:MedLDA_Unconstrained}) with respect to $q(\Thetav, \Zv, \Phiv)$. In the constrained form, this step solves
\begin{eqnarray}\label{problem:MedLDA_subZ}
 \min_{q(\Thetav, \Zv, \Phiv), \xiv} &&~\mathcal{L}\left( q(\Thetav, \Zv, \Phiv) \right) + 2 c\sum_{d=1}^D \xi_d \\
\mathrm{s.t.}: &&~ y_d (\kappav^\ast)^\top \ep_q[\bar{\zv}_d] \geq \ell - \xi_d,~\xi_d \geq 0, \forall d. \nonumber
\end{eqnarray}
Although we can solve this problem using Lagrangian methods, it would be hard to derive the dual objective. An effective approximation strategy was used in~\citep{Zhu:jmlr12,Zhu:nips12}, which updates $q(\Thetav, \Zv, \Phiv)$ for only one step with $\xiv$ fixed at $\xiv^\ast$. By fixing $\xiv$ at $\xiv^\ast$, we have the solution $$q(\Thetav, \Zv, \Phiv)  \propto p(\Wv, \Thetav, \Zv, \Phiv) \exp\left\{ (\kappav^\ast)^\top \sum_{d} \mu_d^\ast \bar{\zv}_d \right\},$$
where the second term indicates the regularization effects due to the max-margin posterior constraints. For those data with non-zero Lagrange multipliers (i.e., support vectors), the second term will bias MedLDA towards a new posterior distribution that favors more discriminative representations on these ``hard" data points. The Monte Carlo methods~\citep{Zhu:nips12} directly draw samples from the posterior distribution $q(\Thetav, \Zv, \Phiv)$ or its collapsed form using Gibbs sampling to estimate $\ep_q[\bar{\zv}_d]$, the expectations required to learn $q(\etav)$. In contrast, the variational methods~\citep{Zhu:jmlr12} solve problem~(\ref{problem:MedLDA_subZ}) using coordinate descent to estimate $\ep_q[\bar{\zv}_d]$ with a fully factorized assumption.


\section{Gibbs MedLDA}\label{section:GibbsMedLDA}

Now, we present Gibbs max-margin topic models for binary classification and their ``augment-and-collapse" sampling algorithms. We will discuss further extensions in the next section.

\subsection{Learning with an Expected Margin Loss}

As stated above, MedLDA chooses the strategy to minimize the hinge loss of an expected classifier. In learning theory, an alternative approach to building classifiers with a posterior distribution of models is to minimize an expected loss, under the framework known as Gibbs classifiers (or stochastic classifiers)~\citep{McAllester2003,Catoni2007,Germain:icml09} which have nice theoretical properties on generalization performance.


For our case of inferring the distribution of latent topic assignments $\Zv = \{\zv_d\}_{d=1}^D$ and the classification model $\etav$, the expected margin loss is defined as follows. If we have drawn a sample of the topic assignments $\Zv$ and the prediction model $\etav$ from a posterior distribution $q(\etav, \Zv)$, we can define the linear discriminant function $$F(\etav, \zv; \wv) = \etav^\top \bar{\zv}$$ as before and make prediction using the {\it latent prediction rule}
\begin{eqnarray}\label{eq:GibbsRule}
\hat{y}(\etav, \zv) = \sign ~F(\etav, \zv; \wv).
\end{eqnarray}
Note that the prediction is a function of the configuration $(\etav, \zv)$. Let $\zeta_d = \ell - y_d \etav^\top \zvbar_d$, where $\ell$ is a cost parameter as defined before. The hinge loss of the stochastic classifier is $$\risk(\etav, \Zv) = \sum_{d=1}^D \max(0, \zeta_d),$$ a function of the latent variables $(\etav, \Zv)$, and the expected hinge loss is
\setlength\arraycolsep{1pt} \begin{eqnarray}
\risk^\prime(q) =  \ep_q[ \risk(\etav, \Zv) ] = \sum_{d=1}^D \ep_q\left[ \max(0, \zeta_d) \right], \nonumber
\end{eqnarray}
a function of the posterior distribution $q(\etav, \Zv)$. Since for any $(\etav, \Zv)$, the hinge loss $\risk(\etav, \Zv)$ is an upper bound of the training error of the latent Gibbs classifier~(\ref{eq:GibbsRule}), that is,
$$\risk(\etav, \Zv) \geq \sum_{d=1}^D \indicator\left(y_d \neq \hat{y}_d(\etav, \zv_d)\right),$$
we have $$ \risk^\prime(q) \geq \sum_{d=1}^D \ep_q\left[\indicator(y_d \neq \hat{y}_d(\etav, \zv_d))\right],$$
where $\indicator(\cdot)$ is an indicator function that equals to 1 if the predicate holds otherwise 0. In other words, the expected hinge loss $\risk^\prime(q)$ is an upper bound of the expected training error of the Gibbs classifier (\ref{eq:GibbsRule}). Thus, it is a good surrogate loss for learning a posterior distribution which could lead to a low training error in expectation.

Then, with the same goal as MedLDA of finding a posterior distribution $q(\etav, \Thetav, \Zv, \Phiv)$ that on one hand describes the observed data and on the other hand predicts as well as possible on training data, we define Gibbs MedLDA as solving the new regularized Bayesian inference problem
\begin{eqnarray}\label{problem:GibbsMedLDA_Unconstrained}
\min_{q(\etav, \Thetav, \Zv, \Phiv) }  \mathcal{L}\left( q(\etav, \Thetav, \Zv, \Phiv) \right) + 2 c \risk^\prime\left( q(\etav, \Thetav, \Zv, \Phiv) \right).
\end{eqnarray}
Note that we have written the expected margin loss $\risk^\prime$ as a function of the complete distribution $q(\etav, \Thetav, \Zv, \Phiv)$. This doesn't conflict with our definition of $\risk^\prime$ as a function of the marginal distribution $q(\etav, \Zv)$ because the other irrelevant variables (i.e., $\Thetav$ and $\Phiv$) are integrated out when we compute the expectation.

Comparing to MedLDA in problem (\ref{problem:MedLDA_Unconstrained}), we have the following lemma by applying Jensen's inequality.
\begin{lemma}\label{lemma:upper-bound-medlda} The expected hinge loss $\risk^\prime$ is an upper bound of the hinge loss of the expected classifier~(\ref{eqn:prediction_rule}):
\begin{eqnarray}
\risk^\prime(q) \geq \risk(q) = \sum_{d=1}^D \max\left( 0, \ep_q[\zeta_d] \right); \nonumber
\end{eqnarray}
and thus the objective in (\ref{problem:GibbsMedLDA_Unconstrained}) is an upper bound of that in (\ref{problem:MedLDA_Unconstrained}) when $c$ values are the same.
\end{lemma}

\subsection{Formulation with Data Augmentation}
If we directly solve  problem~(\ref{problem:GibbsMedLDA_Unconstrained}), the expected hinge loss $\risk^\prime$ is hard to deal with because of the non-differentiable max function. Fortunately, we can develop a simple collapsed Gibbs sampling algorithm with analytical forms of local conditional distributions, based on a data augmentation formulation of the expected hinge-loss.

Let $\phi(y_d | \zv_d, \etav) = \exp\{ -2 c \max(0, \zeta_d) \}$ be the unnormalized likelihood of the response variable for document $d$. Then, problem~(\ref{problem:GibbsMedLDA_Unconstrained}) can be written as
\begin{equation}\label{problem:GibbsMedLDA}
\min_{q(\etav, \Thetav, \Zv, \Phiv) } \mathcal{L}\left( q(\etav, \Thetav, \Zv, \Phiv) \right) - \ep_q \left[ \log \phi(\yv | \Zv, \etav) \right],
\end{equation}
where $\phi(\yv|\Zv, \etav) = \prod_{d=1}^D \phi(y_d|\zv_d, \etav)$. Solving problem~(\ref{problem:GibbsMedLDA}) with the constraint that $q(\etav, \Thetav, \Zv, \Phiv) \in \mathcal{P}$, we can get the normalized posterior distribution
\begin{eqnarray}
q(\etav, \Thetav, \Zv, \Phiv) = \frac{ p_0(\etav, \Thetav, \Zv, \Phiv) p(\Wv | \Zv, \Phiv) \phi(\yv | \Zv, \etav) }{\psi(\yv, \Wv)}, \nonumber
\end{eqnarray}
where $\psi(\yv, \Wv)$ is the normalization constant. Due to the complicated form of $\phi$, it will not have simple conditional distributions if we want to derive a Gibbs sampling algorithm for $q(\etav, \Thetav, \Zv, \Phiv)$ directly. This motivates our exploration of data augmentation techniques. Specifically, using the ideas of data augmentation~\citep{Tanner:1987,Polson:BA11}, we have Lemma~\ref{lemma:SoM}.
\begin{lemma}[Scale Mixture Representation] \label{lemma:SoM} The unnormalized likelihood can be expressed as
\setlength\arraycolsep{1pt} \begin{eqnarray}
&& \phi(y_d|\zv_d, \etav) = \int_0^\infty \frac{1}{ \sqrt{ 2\pi \lambda_d} } \exp\left( -\frac{ (\lambda_d + c \zeta_d)^2 }{2 \lambda_d}  \right) \ud \lambda_d \nonumber
\end{eqnarray}
\end{lemma}
\begin{proof}
Due to the fact that $a \max(0, x) = \max(0, a x)$ if $a \geq 0$, we have $-2c\max(0, \zeta_d) = -2\max(0, c\zeta_d)$. Then, we can follow the proof in~\citep{Polson:BA11} to get the results.
\end{proof}
Lemma~\ref{lemma:SoM} indicates that the posterior distribution of Gibbs MedLDA can be expressed as the marginal of a higher-dimensional distribution that includes the augmented variables $\lambdav = \{ \lambda_d \}_{d=1}^D$, that is,
\begin{eqnarray}
q(\etav, \Thetav, \Zv, \Phiv) = \int_0^\infty \cdots \int_0^\infty q(\etav, \lambdav, \Thetav, \Zv, \Phiv) \ud \lambda_1 \cdots \ud \lambda_D = \int_{\mathbb{R}_{+}^D} q(\etav, \lambdav, \Thetav, \Zv, \Phiv) \ud \lambdav,
\end{eqnarray}
where $\mathbb{R}_{+} = \{ x: x \in \mathbb{R},~x > 0 \}$ is the set of positive real numbers; the complete posterior distribution is
\setlength\arraycolsep{1pt}\begin{eqnarray}
q(\etav, \lambdav, \Thetav, \Zv, \Phiv) = \frac{ p_0(\etav, \Thetav, \Zv, \Phiv) p(\Wv | \Zv, \Phiv) \phi(\yv, \lambdav | \Zv, \etav) }{\psi(\yv, \Wv)}; \nonumber
\end{eqnarray}
and the unnormalized joint distribution of $\yv$ and $\lambdav$ is
\begin{eqnarray}
\phi(\yv, \lambdav | \Zv, \etav) = \prod_{d=1}^D \frac{1}{ \sqrt{ 2\pi \lambda_d} } \exp\left( -\frac{ (\lambda_d + c \zeta_d )^2 }{2 \lambda_d}  \right). \nonumber
\end{eqnarray}
In fact, we can show that the complete posterior distribution is the solution of the data augmentation problem of Gibbs MedLDA
\setlength\arraycolsep{0pt}\begin{eqnarray}\label{eq:JointInf}
&&\min_{q(\etav, \lambdav, \Thetav, \Zv, \Phiv) } \mathcal{L}\left( q(\etav, \lambdav, \Thetav, \Zv, \Phiv) \right) - \ep_q \left[ \log \phi(\yv, \lambdav | \Zv, \etav) \right], \nonumber
\end{eqnarray}
which is again subject to the normalization constraint that $q(\etav, \lambdav, \Thetav, \Zv, \Phiv) \in \mathcal{P}$. The first term in the objective is $\mathcal{L}( q(\etav, \lambdav, \Thetav, \Zv, \Phiv) ) = \mathrm{KL}(q(\etav, \lambdav, \Thetav, \Zv, \Phiv) || p_0(\etav, \Thetav, \Zv, \Phiv)) - \mathbb{E}_q[\log p(\Wv | \Zv, \Phiv)]$. If we like to impose a prior distribution on the augmented variables $\lambdav$, one good choice can be an improper uniform prior.
\begin{remark}
The objective of this augmented problem is an upper bound of the objective in~(\ref{problem:GibbsMedLDA}) (thus, also an upper bound of MedLDA's objective due to Lemma~\ref{lemma:upper-bound-medlda}). 
This is because by using the data augmentation we can show that
\begin{eqnarray}
\ep_{q(\Vv)} [ \log \phi(\yv | \Zv, \etav) ]
&=& \ep_{q(\Vv)} \left[ \log \int_{\mathbb{R}_{+}^D} \phi(\yv, \lambdav | \Zv, \etav) \ud \lambdav \right] \nonumber \\
&=& \ep_{q(\Vv)} \left[ \log \int_{\mathbb{R}_{+}^D} \frac{q(\lambdav|\Vv)}{q(\lambdav|\Vv)} \phi(\yv, \lambdav | \Zv, \etav) \ud \lambdav \right] \nonumber  \\
&\geq& \ep_{q(\Vv)} \left[ \ep_{q(\lambdav|\Vv)} \left[   \log \phi(\yv, \lambdav | \Zv, \etav) \right] - \ep_{q(\lambdav|\Vv)}\left[ \log q(\lambdav|\Vv) \right] \right] \nonumber \\
&=& \ep_{q(\Vv, \lambdav)} \left[ \log \phi(\yv, \lambdav | \Zv, \etav) \right] - \ep_{q(\Vv, \lambdav)}\left[ \log q(\lambdav|\Vv) \right] \nonumber
\end{eqnarray}
where $\Vv = \{\etav, \Thetav, \Zv, \Phiv\}$ denotes all the random variables in MedLDA. Therefore, we have
\begin{eqnarray}
\mathcal{L}(q(\Vv)) - \ep_{q(\Vv)} [ \log \phi(\yv | \Zv, \etav) ]
&\leq& \mathcal{L}(q(\Vv)) - \ep_{q(\Vv, \lambdav)} \left[ \log \phi(\yv, \lambdav | \Zv, \etav) \right] + \ep_{q(\Vv, \lambdav)}\left[ \log q(\lambdav|\Vv) \right] \nonumber \\
&=& \mathcal{L}(q(\Vv, \lambdav)) - \ep_q [ \log \phi(\yv, \lambdav | \Zv, \etav) ].\nonumber
\end{eqnarray}
\end{remark}

\subsection{Inference with Collapsed Gibbs Sampling}


Although with the above data augmentation formulation we can do Gibbs sampling to infer the complete posterior distribution $q(\etav, \lambdav,\Thetav, \Zv, \Phiv)$ and thus $q(\etav, \Thetav, \Zv, \Phiv)$ by ignoring $\lambdav$, the mixing rate would be slow because of the large sample space of the latent variables. One way to effectively reduce the sample space and improve mixing rates is to integrate out the intermediate Dirichlet variables $(\Thetav, \Phiv)$ and build a Markov chain whose equilibrium distribution is the resulting marginal distribution $q(\etav, \lambdav, \Zv)$. We propose to use collapsed Gibbs sampling, which has been successfully used in LDA~\citep{Griffiths:04}. With the data augmentation representation, this leads to an ``augment-and-collapse" sampling algorithm for Gibbs MedLDA, as detailed below.

For the data augmented formulation of Gibbs MedLDA, by integrating out the Dirichlet variables $(\Thetav, \Phiv)$, we get the collapsed posterior distribution:
\setlength\arraycolsep{1pt} \begin{eqnarray}
q(\etav, \lambdav, \Zv)  && \propto  p_0(\etav) p(\Wv, \Zv|\alphav, \betav) \phi(\yv, \lambdav|\Zv, \etav)  \nonumber \\
             && = p_0(\etav) \left[ \prod_{d=1}^{D} \frac{\delta(\mathbf{C}_d + \alphav)}{\delta(\alphav)} \right] \prod_{k=1}^{K}\frac{\delta(\mathbf{C}_k + \betav)}{\delta(\betav)} \prod_{d=1}^{D} \frac{1}{ \sqrt{ 2\pi \lambda_d} } \exp\left( - \frac{ (\lambda_d + c \zeta_d )^2 }{2 \lambda_d}  \right), \nonumber
\end{eqnarray}
where $$\delta(\xv)=\frac{\prod_{i=1}^{\mathrm{dim}(\xv)}\Gamma(x_i)}{\Gamma(\sum_{i=1}^{\mathrm{dim}(\xv)} x_i)};$$ $\Gamma(\cdot)$ is the Gamma function; $C_k^t$ is the number of times that the term $t$ is assigned to topic $k$ over the whole corpus; $\mathbf{C}_k=\{C_k^t\}_{t=1}^{V}$ is the set of word counts associated with topic $k$; $C_d^k$ is the number of times that terms are associated with topic $k$ within the $d$-th document; and $\mathbf{C}_d=\{C_d^k\}_{k=1}^{K}$ is the set of topic counts for document $d$. Then, the conditional distributions used in collapsed Gibbs sampling are as follows.

{\bf For {\boldmath $\etav$}}: Let us assume its prior is the commonly used isotropic Gaussian distribution $p_0(\etav) = \prod_{k=1}^K \mathcal{N}(\eta_k; 0, \nu^2)$, where $\nu$ is a non-zero parameter. Then, we have the conditional distribution of $\etav$ given the other variables:
\begin{eqnarray}\label{eq:GibbsEta}
q(\etav | \Zv, \lambdav ) && \propto p_0(\etav) \prod_{d=1}^D \exp\left( - \frac{ (\lambda_d + c \zeta_d )^2 }{2 \lambda_d}  \right) \nonumber \\
                         &&\propto \exp\left( -\sum_{k=1}^K \frac{\eta_k^2}{2 \nu^2} - \sum_{d=1}^D \frac{ (\lambda_d + c \zeta_d )^2 }{2 \lambda_d}  \right) \nonumber \\
                         &&= \exp\left( -\frac{1}{2}\etav^\top \left( \frac{1}{\nu^2} I + c^2 \sum_{d=1}^D \frac{\zvbar_d \zvbar_d^\top}{\lambda_d} \right) \etav + \left( c\sum_{d=1}^D y_d \frac{\lambda_d + c\ell }{\lambda_d}\zvbar_d \right)^\top \etav \right) \nonumber \\
                         &&= \mathcal{N}(\etav; \muv, \Sigmav),
\end{eqnarray}
a $K$-dimensional Gaussian distribution, where the posterior mean and the covariance matrix are $$\muv = \Sigmav \left( c \sum_{d=1}^D y_d \frac{\lambda_d + c\ell}{\lambda_d}\zvbar_d \right),~~ \Sigmav = \left( \frac{1}{\nu^2}I + c^2 \sum_{d=1}^D \frac{\zvbar_d \zvbar_d^\top}{\lambda_d} \right)^{-1}.$$ Therefore, we can easily draw a sample from this multivariate Gaussian distribution. The inverse can be robustly done using Cholesky decomposition, an $O(K^3)$ procedure. Since $K$ is normally not large, the inversion can be done efficiently, especially in applications where the number of documents is much larger than the number of topics. 

{\bf For $\Zv$}: The conditional distribution of $\Zv$ given the other variables is
\begin{eqnarray}
q(\Zv | \etav, \lambdav ) && \propto \prod_{d=1}^{D} \frac{\delta(\mathbf{C}_d + \alphav)}{\delta(\alphav)} \exp\left( - \frac{ (\lambda_d + c \zeta_d)^2 }{2 \lambda_d}  \right)  \prod_{k=1}^{K}\frac{\delta(\mathbf{C}_k + \betav)}{\delta(\betav)}. \nonumber
\end{eqnarray}
By canceling common factors, we can derive the conditional distribution of one variable $z_{dn}$ given others $\Zv_{\neg}$ as:
\setlength\arraycolsep{1pt} \begin{eqnarray}\label{eqn:transitionProb}
 q(z_{dn}^k = 1 | \Zv_{\neg}, \etav, \lambdav, w_{dn}=t ) && \propto \frac{ (C_{k,\neg n}^{t}+\beta_t) (C_{d,\neg n}^{k}+\alpha_k) }{\sum_{t=1}^V C_{k,\neg n}^t + \sum_{t=1}^V \beta_t} \exp\Big( \frac{ \gamma y_d (c\ell + \lambda_d)\eta_k }{\lambda_d} \nonumber \\
&& ~~  - c^2 \frac{\gamma^2 \eta_k^2 + 2 \gamma(1-\gamma)\eta_k \Lambda_{dn}^k }{2 \lambda_d} \Big),
\end{eqnarray}
where $C_{\cdot,\neg n}^{\cdot}$ indicates that term $n$ is excluded from the corresponding document or topic; $\gamma = \frac{1}{N_d}$; and $$\Lambda_{dn}^k = \frac{1}{N_d-1} \sum_{k^\prime = 1}^K \eta_{k^\prime} C_{d, \neg n}^{k^\prime}$$ is the discriminant function value without word $n$. We can see that the first term is from the LDA model for observed word counts and the second term is from the supervised signal $\yv$.

\begin{algorithm}[t]
\caption{Collapsed Gibbs Sampling Algorithm for GibbsMedLDA Classification Models}\label{alg:GibbsAlg}
\begin{algorithmic}[1]
   \STATE {\bfseries Initialization:} set $\lambdav = 1$ and randomly draw $z_{dk}$ from a uniform distribution.
   \FOR{$m=1$ {\bfseries to} $M$}
   \STATE draw the classifier from the normal distribution~(\ref{eq:GibbsEta})
    \FOR{$d=1$ {\bfseries to} $D$}
        \FOR{each word $n$ in document $d$}
             \STATE draw a topic from the multinomial distribution~(\ref{eqn:transitionProb})
        \ENDFOR
       \STATE draw $\lambda_d^{-1}$ (and thus $\lambda_d$) from the inverse Gaussian distribution~(\ref{eq:GibbsLambda}).
   \ENDFOR
   \ENDFOR
\end{algorithmic}
\end{algorithm}

{\bf For $\lambdav$}: Finally, the conditional distribution of the augmented variables $\lambdav$ given the other variables is factorized and we can derive the conditional distribution for each $\lambda_d$ as
\begin{eqnarray}
 q(\lambda_d | \Zv, \etav) &&\propto \frac{1}{\sqrt{ 2 \pi \lambda_d }} \exp\left( - \frac{(\lambda_d + c \zeta_d )^2 }{2\lambda_d}  \right) \nonumber \\
                         &&\propto \frac{1}{\sqrt{ 2 \pi \lambda_d }} \exp\left( -\frac{c^2 \zeta_d^2}{2 \lambda_d} - \frac{\lambda_d}{2} \right) \nonumber \\
                         && =  \GIG \left( \lambda_d; \frac{1}{2}, 1, c^2\zeta_d^2 \right), \nonumber
\end{eqnarray}
where $$\GIG(x; p, a,b) = C(p, a,b)x^{p-1}\exp\left( -\frac{1}{2}\big( \frac{b}{x} + a x \big) \right)$$ is a generalized inverse Gaussian distribution~\citep{Devroye:book1986} and $C(p,a,b)$ is a normalization constant. Therefore, we can derive that $\lambda_d^{-1}$ follows an inverse Gaussian distribution
\begin{eqnarray}\label{eq:GibbsLambda}
p(\lambda_d^{-1} | \Zv, \etav ) = \IG \left(\lambda_d^{-1}; \frac{1}{c |\zeta_d|}, 1 \right),
\end{eqnarray}
where $$\IG(x; a, b) = \sqrt{ \frac{b}{2\pi x^3} }\exp\left( -\frac{b(x - a)^2}{2 a^2 x} \right)$$ for $a>0$ and $b>0$.

With the above conditional distributions, we can construct a Markov chain which iteratively draws samples of the classifier weights $\etav$ using Eq. (\ref{eq:GibbsEta}), the topic assignments $\Zv$ using Eq. (\ref{eqn:transitionProb}) and the augmented variables $\lambdav$ using Eq. (\ref{eq:GibbsLambda}), with an initial condition. To sample from an inverse Gaussian distribution, we apply the transformation method with multiple roots~\citep{Michael:IG76} which is very efficient with a constant time complexity. Overall, the per-iteration time complexity is $\mathcal{O}(K^3 + N_{total} K)$, where $N_{total} = \sum_{d=1}^D N_d$ is the total number of words in all documents. If $K$ is not very large (e.g., $K \ll \sqrt{N_{total}}$), which is the common case in practice as $N_{total}$ is often very large, the per-iteration time complexity is $\mathcal{O}(N_{total} K)$; if $K$ is large (e.g., $K \gg \sqrt{N_{total}}$), which is not common in practice, drawing the global classifier weights will dominate and the per-iteration time complexity is $\mathcal{O}(K^3)$. In our experiments, we initially set $\lambda_d=1,~\forall d$ and randomly draw $\Zv$ from a uniform distribution. In training, we run this Markov chain to finish the burn-in stage with $M$ iterations, as outlined in Algorithm~\ref{alg:GibbsAlg}. Then, we draw a sample $\hat{\etav}$ as the Gibbs classifier to make predictions on testing data.

In general, there is no theoretical guarantee that a Markov chain constructed using data augmentation can converge to the target distribution (See~\citep{Hobert2011} for a failure example). However, for our algorithms, we can justify that the Markov transition distribution of the chain satisfies the condition $\mathcal{K}$ from~\citep{Hobert2011}, i.e., the transition probability from one state to any other state is larger than 0. Condition $\mathcal{K}$ implies that the Markov chain is Harris ergodic~\citep[Lemma 1]{Tan2009}. Therefore, no matter how the chain is started, our sampling algorithms can be employed to effectively explore the intractable posterior distribution. In practice, the sampling algorithm as well as the ones to be presented require only a few iterations to get stable prediction performance, as we shall see in Section~\ref{sec:sensitivity-burn-in}. More theoretical analysis such as convergence rates requires a good bit of technical Markov chain theory and is our future work.


\subsection{Prediction}\label{sec:prediction}

To apply the Gibbs classifier $\hat{\etav}$, we need to infer the topic assignments for testing document, denoted by $\wv$. A fully Bayesian treatment needs to compute an integral in order to get the posterior distribution of the topic assignment given the training data $\data$ and the testing document content $\wv$: $$p(\zv | \wv, \data) \propto \int_{\mathcal{P}_V^K} p(\zv, \wv, \Phi | \data) \ud \Phi = \int_{\mathcal{P}_V^K} p(\zv, \wv | \Phi) p(\Phi | \data) \ud \Phi,$$ where $\mathcal{P}_V$ is the $V-1$ dimensional simplex; and the second equality holds due to the conditional independence assumption of the documents given the topics. Various approximation methods can be applied to compute the integral. Here, we take the approach applied in~\citep{Zhu:jmlr12,Zhu:nips12}, which uses a point estimate of topics $\Phiv$ from training data and makes predictions based on them. Specifically, we use a point estimate $\hat{\Phiv}$ (a Dirac measure) to approximate the probability distribution $p(\Phiv | \data)$. For the collapsed Gibbs sampler, an estimate of $\hat{\Phiv}$ using the samples is the posterior mean $$\hat{\phi}_{kt} \propto {C_k^t} + \beta_t.$$ Then, given a testing document $\wv$, we infer its latent components $\zv$ using $\hat{\Phiv}$ by drawing samples from the local conditional distribution
\begin{eqnarray}\label{eq:GibbsTest}
p(z_n^k=1|\mathbf{z}_{\neg n}, \wv, \data) \propto \hat{\phi}_{kw_{n}}\left( C_{\neg n}^{k} + \alpha_k \right),
\end{eqnarray}
where $C_{\neg n}^{k}$ is the number of times that the terms in this document $\wv$ assigned to topic $k$ with the $n$-th term excluded. To start the sampler, we randomly set each word to one topic. Then, we run the Gibbs sampler for a few iterations until some stop criterion is satisfied, e.g., after a few burn-in steps or the relative change of data likelihood is lower than some threshold. Here, we adopt the latter, the same as in~\citep{Zhu:nips12}. After this burn-in stage, we keep one sample of $\zv$ for prediction using the stochastic classifier. Empirically, using the average of a few (e.g., 10) samples of $\zv$ could lead to slightly more robust predictions, as we shall see in Section~\ref{sec:sensitivity-test-lag}.

\section{Extensions to Regression and Multi-task Learning}\label{section:Extensions}

The above ideas can be naturally generalized to develop Gibbs max-margin supervised topic models for various prediction tasks. In this section, we present two examples for regression and multi-task learning, respectively.

\subsection{Gibbs MedLDA Regression Model}

We first discuss how to generalize the above ideas to develop a regression model, where the response variable $Y$ takes real values. Formally, the Gibbs MedLDA regression model also has two components --- an LDA model to describe input bag-of-words documents and a Gibbs regression model for the response variables. Since the LDA component is the same as in the classification model, we focus on presenting the Gibbs regression model.

\subsubsection{The Models with Data Augmentation}
If a sample of the topic assignments $\Zv$ and the prediction model $\etav$ are drawn from the posterior distribution $q(\etav, \Zv)$, we define the latent regression rule as
\begin{eqnarray}\label{eq:GibbsMedLDAReg}
\hat{y}(\etav, \zv) = \etav^\top \zvbar.
 \end{eqnarray}
To measure the goodness of the prediction rule~(\ref{eq:GibbsMedLDAReg}), we adopt the widely used $\epsilon$-insensitive loss $$\risk_\epsilon(\etav, \Zv) = \sum_{d=1}^D \max\left( 0, |\Delta_d| - \epsilon \right),$$
where $\Delta_d = y_d - \etav^\top \zvbar_d$ is the margin between the true score and the predicted score. The $\epsilon$-insensitive loss has been successfully used in learning fully observed support vector regression~\citep{Smola:03}. In our case, the loss is a function of predictive model $\etav$ as well as the topic assignments $\Zv$ which are hidden from the input data. To resolve this uncertainty, we define the expected $\epsilon$-insensitive loss $$\risk_\epsilon(q) = \ep_q\left[ \risk_\epsilon(\etav, \Zv) \right] = \sum_{d=1}^D \ep_q\left[ \max(0, |\Delta_d| - \epsilon) \right],$$ a function of the desired posterior distribution $q(\etav, \Zv)$.

With the above definitions, we can follow the same principle as Gibbs MedLDA to define the Gibbs MedLDA regression model as solving the regularized Bayesian inference problem
\setlength\arraycolsep{-4pt}\begin{eqnarray}\label{problem:GibbsMedLDAr_Unconstrained}
&&\min_{q(\etav, \Thetav, \Zv, \Phiv) }  \mathcal{L}\left( q(\etav, \Thetav, \Zv, \Phiv) \right) + 2 c \risk_\epsilon\left( q(\etav, \Thetav, \Zv, \Phiv) \right).
\end{eqnarray}
Note that as in the classification model, we have put the complete distribution $q(\etav, \Thetav, \Zv, \Phiv)$ as the argument of the expected loss $\risk_\epsilon$, which only depends on the marginal distribution $q(\etav, \Zv)$. This does not affect the results because we are taking the expectation to compute $\risk_\epsilon$ and any irrelevant variables will be marginalized out.

As in the Gibbs MedLDA classification model, we can show that $\risk_\epsilon$ is an upper bound of the $\epsilon$-insensitive loss of MedLDA's expected prediction rule, by applying Jensen's inequality to the convex function $h(x) = \max(0, |x|-\epsilon)$.
\begin{lemma} We have $\risk_\epsilon \geq  \sum_{d=1}^D \max(0, |\ep_q[\Delta_d]| - \epsilon)$.
\end{lemma}
We can reformulate problem (\ref{problem:GibbsMedLDAr_Unconstrained}) in the same form as problem~(\ref{problem:GibbsMedLDA}), with the unnormalized likelihood $$\phi(y_d|\etav,\zv_d) = \exp\left(-2c \max(0, |\Delta_d | - \epsilon)\right).$$
Then, we have the dual scale of mixture representation, by noting that
\begin{eqnarray}\label{eq:Eqaulity}
\max(0, |x| - \epsilon) = \max(0, x - \epsilon) + \max(0, -x - \epsilon).
\end{eqnarray}
\begin{lemma}[Dual Scale Mixture Representation]  For regression, the unnormalized likelihood can be expressed as
\setlength\arraycolsep{1pt} \begin{eqnarray}
\phi(y_d|\etav,\zv_d) && =  \int_0^\infty  \frac{1}{ \sqrt{ 2\pi \lambda_d} } \exp\left( -\frac{ (\lambda_d + c(\Delta_d - \epsilon))^2 }{2 \lambda_d}  \right) \ud\lambda_d \nonumber \\
&& \times \int_0^\infty \frac{1}{ \sqrt{ 2\pi \omega_d} } \exp\left( -\frac{ (\omega_d - c(\Delta_d + \epsilon))^2 }{2 \omega_d}  \right) \ud\omega_d \nonumber
\end{eqnarray}
\end{lemma}
\begin{proof}
By the equality (\ref{eq:Eqaulity}), we have $\phi(y_d|\etav,\zv_d) = \exp\{-2c\max(0, \Delta_d - \epsilon)\} \exp\{-2c\max(0, -\Delta_d - \epsilon)\}.$
Each of the exponential terms can be formulated as a scale mixture of Gaussians due to Lemma 2.
\end{proof}
Then, the data augmented learning problem of the Gibbs MedLDA regression model is
\begin{eqnarray}\label{eq:JointInfRegression}
\min_{q(\etav, \lambdav, \omegav, \Thetav, \Zv, \Phiv) }  \mathcal{L}\left( q(\etav, \lambdav, \omegav, \Thetav, \Zv, \Phiv) \right) - \ep_q \left[ \log \phi(\yv, \lambdav, \omegav | \Zv, \etav) \right] \nonumber
\end{eqnarray}
where $\phi(\yv, \lambdav, \omegav | \Zv, \etav) = \prod_{d=1}^D \phi(y_d, \lambda_d, \omega_d | \Zv, \omegav)$ and
\setlength\arraycolsep{1pt}\begin{eqnarray}
\phi(y_d, \lambda_d, \omega_d | \Zv, \etav) &&= \frac{1}{ \sqrt{ 2\pi \lambda_d} } \exp\left( -\frac{ (\lambda_d + c(\Delta_d - \epsilon))^2 }{2 \lambda_d}  \right)
 \frac{1}{ \sqrt{ 2\pi \omega_d} } \exp\left( -\frac{ (\omega_d - c(\Delta_d + \epsilon))^2 }{2 \omega_d}  \right). \nonumber
\end{eqnarray}
Solving the augmented problem and integrating out $(\Thetav, \Phiv)$, we can get the collapsed posterior distribution
$$q(\etav, \lambdav, \omegav, \Zv ) \propto p_0(\etav) p(\Wv, \Zv|\alphav, \betav) \phi(\yv, \lambdav, \omegav|\Zv, \etav).$$

\subsubsection{A Collapsed Gibbs Sampling Algorithm}

Following similar derivations as in the classification model, the Gibbs sampling algorithm to infer the posterior has the following conditional distributions, with an outline in Algorithm~\ref{alg:GibbsAlgReg}.

{\bf For $\etav$}: Again, with the isotropic Gaussian prior $p_0(\etav) = \prod_{k=1}^K \mathcal{N}(\eta_k; 0, \nu^2)$, we have
\begin{eqnarray}\label{eq:GibbsEtaReg}
q(\etav | \Zv, \lambdav, \omegav)  &\propto& p_0(\etav) \prod_{d=1}^D \exp\left( - \frac{ (\lambda_d + c(\Delta_d - \epsilon))^2 }{2 \lambda_d}  \right) \exp\left( - \frac{ (\omega_d - c(\Delta_d + \epsilon))^2 }{2 \omega_d}  \right)  \nonumber \\
                         &\propto& \exp\left( -\sum_{k=1}^K \frac{\eta_k^2}{2 \nu^2} - \sum_{d=1}^D \left( \frac{ (\lambda_d + c(\Delta_d - \epsilon))^2  }{2 \lambda_d} + \frac{ (\omega_d - c(\Delta_d + \epsilon))^2 }{2 \omega_d} \right) \right) \nonumber \\
                        &=& \exp\left( -\frac{1}{2}\etav^\top \left( \frac{1}{\nu^2} I + c^2 \sum_{d=1}^D \rho_d \zvbar_d \zvbar_d^\top \right)\etav + c \left(\sum_{d=1}^D \psi_d \zvbar_d \right)^\top \etav \right) \nonumber \\
                         =&& \mathcal{N}(\etav; \muv, \Sigmav),
\end{eqnarray}
where the posterior covariance matrix and the posterior mean are $$\Sigmav = \left(\frac{1}{\nu^2}I + c^2\sum_{d=1}^D \rho_d \zvbar_d \zvbar_d^\top \right)^{-1},~~\muv = c \Sigmav \left( \sum_{d=1}^D \psi_d \zvbar_d \right),$$ and $\rho_d = \frac{1}{\lambda_d} + \frac{1}{\omega_d}$ and $\psi_d = \frac{y_d - \epsilon}{\lambda_d} + \frac{y_d + \epsilon}{\omega_d}$ are two parameters. We can easily draw a sample from a $K$-dimensional multivariate Gaussian distribution. The inverse can be robustly done using Cholesky decomposition.

\begin{algorithm}[t]
\caption{Collapsed Gibbs Sampling Algorithm for GibbsMedLDA Regression Models}\label{alg:GibbsAlgReg}
\begin{algorithmic}[1]
   \STATE {\bfseries Initialization:} set $\lambdav = 1$ and randomly draw $z_{dk}$ from a uniform distribution.
   \FOR{$m=1$ {\bfseries to} $M$}
   \STATE draw the classifier from the normal distribution~(\ref{eq:GibbsEtaReg})
    \FOR{$d=1$ {\bfseries to} $D$}
        \FOR{each word $n$ in document $d$}
             \STATE draw a topic from the multinomial distribution~(\ref{eqn:transitionProbReg})
        \ENDFOR
       \STATE draw $\lambda_d^{-1}$ (and thus $\lambda_d$) from the inverse Gaussian distribution~(\ref{eq:GibbsLambdaReg}).
       \STATE draw $\omega_d^{-1}$ (and thus $\omega_d$) from the inverse Gaussian distribution~(\ref{eq:GibbsOmegaReg}).
   \ENDFOR
   \ENDFOR
\end{algorithmic}
\end{algorithm}
{\bf For $\Zv$}: 
We can derive the conditional distribution of one variable $z_{dn}$ given others $\Zv_{\neg}$ as:
\setlength\arraycolsep{1pt} \begin{eqnarray}\label{eqn:transitionProbReg}
q(z_{dn}^k = 1 | \Zv_{\neg}, \etav, \lambdav, \omegav, w_{dn}=t) && \propto \frac{ (C_{k,\neg n}^{t}+\beta_t) (C_{d,\neg n}^{k}+\alpha_k) }{\sum_{t=1}^V C_{k,\neg n}^t + \sum_{t=1}^V \beta_t} \exp\Big( c\gamma \psi_d \eta_k  \nonumber \\
&& ~~  - c^2 ( \frac{\gamma^2 \rho_d \eta_k^2}{2}  + \gamma(1-\gamma) \rho_d \eta_k \Upsilon_{dn}^k) \Big),
\end{eqnarray}
where $\gamma = \frac{1}{N_d}$; and $\Upsilon_{dn}^k = \frac{1}{N_d-1} \sum_{k^\prime = 1}^K \eta_{k^\prime} C_{d, \neg n}^{k^\prime}$ is the discriminant function value without word $n$. The first term is from the LDA model for observed word counts. The second term is from the supervised signal $\yv$.

{\bf For $\lambdav$ and $\omegav$}: Finally,
we can derive that $\lambda_d^{-1}$ and $\omega_d^{-1}$ follow the inverse Gaussian distributions:
\begin{eqnarray}
q(\lambda_d^{-1} | \Zv, \etav, \omegav) &=& \IG \left( \lambda_d^{-1}; \frac{1}{c|\Delta_d - \epsilon|}, 1 \right), \label{eq:GibbsLambdaReg} \\
q(\omega_d^{-1} | \Zv, \etav, \lambdav) &=& \IG \left( \omega_d^{-1}; \frac{1}{c |\Delta_d + \epsilon|}, 1 \right). \label{eq:GibbsOmegaReg}
\end{eqnarray}
The per-iteration time complexity of this algorithm is similar to that of the binary Gibbs MedLDA model, i.e., linear to the number of documents and number of topics if $K$ is not too large.


\subsection{Multi-task Gibbs MedLDA}\label{section:GibbsMedLDA-MT}

The second extension is a multi-task Gibbs MedLDA. Multi-task learning is a scenario where multiple potentially related tasks are learned jointly with the hope that their performance can be boosted by sharing some statistic strength among these tasks, and it has attracted a lot of research attention. In particular, learning a common latent representation shared by all the related tasks has proven to be an effective way to capture task relationships~\citep{AndoTong:05,Argyriou:nips07,Zhu:nips11}. Here, we take the similar approach to learning multiple predictive models which share the common latent topic representations. As we shall see in Section~\ref{sec:multi-task-multi-class}, one natural application of our approach is to do multi-label classification~\citep{Tsoumakas:10}, where each document can belong to multiple categories, by defining each task as a binary classifier to determine whether a data point belongs to a particular category; and it can also be applied to multi-class classification, where each document belongs to only one of the many categories, by defining a single output prediction rule (See Section~\ref{sec:multi-task-multi-class} for details).

\subsubsection{The Model with Data Augmentation}

We consider $L$ binary classification tasks and each task $i$ is associated with a classifier with weights $\etav_i$. We assume that all the tasks work on the same set of input data $\Wv = \{ \wv_d \}_{d=1}^D$, but each data $d$ has different binary labels $\{ y_d^i \}_{i=1}^L$ in different tasks. A multi-task Gibbs MedLDA model has two components --- an LDA model to describe input words (the same as in Gibbs MedLDA); and multiple Gibbs classifiers sharing the same topic representations. When we have the classifier weights $\etav$ and the topic assignments $\Zv$, drawn from a posterior distribution $q(\etav, \Zv)$, we follow the same principle as in Gibbs MedLDA and define the latent Gibbs rule for each task as
\begin{eqnarray}\label{eq:GibbsRule-MT}
\forall i=1, \dots L:~~\hat{y}^i(\etav_i, \zv) = \sign~ F(\etav_i, \zv; \wv) = \sign( \etav_i^\top \zvbar ).
\end{eqnarray}
Let $\zeta_d^i = \ell - y_d^i \etav_i^\top \zvbar_d$. The hinge loss of the stochastic classifier $i$ is $$\risk_i(\etav_i, \Zv) = \sum_{d=1}^D \max(0, \zeta_d^i)$$ and the expected hinge loss is
\setlength\arraycolsep{1pt} \begin{eqnarray}
\risk_i^\prime(q) =  \ep_q[ \risk_i(\etav_i, \Zv) ] = \sum_{d=1}^D \ep_q\left[ \max(0, \zeta_d^i) \right]. \nonumber
\end{eqnarray}
For each task $i$, we can follow the argument as in Gibbs MedLDA to show that the expected loss $\risk_i^\prime(q)$ is an upper bound of the expected training error $\sum_{d=1}^D \ep_q[\indicator(y_d^i \neq \hat{y}_d^i(\etav_i, \zv_d) )]$ of the Gibbs classifier (\ref{eq:GibbsRule-MT}). Thus, it is a good surrogate loss for learning a posterior distribution which could lead to a low expected training error.

Then, following a similar procedure of defining the binary GibbsMedLDA classifier, we define the multi-task GibbsMedLDA model as solving the following RegBayes problem:
\begin{eqnarray}\label{problem:MT-GibbsMedLDA_Unconstrained}
\min_{q(\etav, \Thetav, \Zv, \Phiv) }  \mathcal{L}\left( q(\etav, \Thetav, \Zv, \Phiv) \right) + 2 c \risk_{MT}^\prime\left( q(\etav, \Thetav, \Zv, \Phiv) \right),
\end{eqnarray}
where the multi-task expected hinge loss is defined as a summation of the expected hinge loss of all the tasks:
\begin{eqnarray}
\risk_{MT}^\prime\left( q(\etav, \Thetav, \Zv, \Phiv) \right) = \sum_{i=1}^L \risk_i^\prime\left( q(\etav, \Thetav, \Zv, \Phiv) \right).
\end{eqnarray}

Due to the separability of the multi-task expected hinge loss, we can apply Lemma~\ref{lemma:SoM} to reformulate each task-specific expected hinge loss $\risk_i^\prime$ as a scale mixture by introducing a set of augmented variables $\{ \lambda_d^i \}_{d=1}^D$. More specifically, let $\phi_i(y_d^i | \zv_d, \etav) = \exp\{ -2 c \max(0, \zeta_d^i) \}$ be the unnormalized likelihood of the response variable for document $d$ in task $i$. Then, we have
\setlength\arraycolsep{1pt} \begin{eqnarray}
&& \phi_i(y_d^i | \zv_d, \etav) = \int_0^\infty \frac{1}{ \sqrt{ 2\pi \lambda_d^i} } \exp\left( -\frac{ (\lambda_d^i + c \zeta_d^i)^2 }{2 \lambda_d^i}  \right) \ud \lambda_d^i. \nonumber
\end{eqnarray}

\subsubsection{A Collapsed Gibbs Sampling Algorithm}

Similar to the binary Gibbs MedLDA classification model, we can derive the collapsed Gibbs sampling algorithm, as outlined in Algorithm~\ref{alg:GibbsAlg-MT}. Specifically, let $$\phi_i(\yv^i, \lambdav^i | \Zv, \etav) = \prod_{d=1}^D \frac{1}{ \sqrt{ 2\pi \lambda_d^i} } \exp\left( -\frac{ (\lambda_d^i + c \zeta_d^i)^2 }{2 \lambda_d^i}  \right) $$ be the joint unnormalized likelihood of the class labels $\yv^i = \{ y_d^i \}_{d=1}^D$ and the augmentation variables $\lambdav^i = \{ \lambda_d^i \}_{d=1}^D$. Then, for the multi-task Gibbs MedLDA, we can integrate out the Dirichlet variables ($\Thetav$, $\Phiv$) and get the collapsed posterior distribution
\setlength\arraycolsep{1pt} \begin{eqnarray}
q(\etav, \lambdav, \Zv)  && \propto  p_0(\etav) p(\Wv, \Zv|\alphav, \betav) \prod_{i=1}^L \phi_i(\yv^i, \lambdav^i | \Zv, \etav)  \nonumber \\
             && = p_0(\etav) \left[ \prod_{d=1}^{D} \frac{\delta(\mathbf{C}_d + \alphav)}{\delta(\alphav)} \right] \prod_{k=1}^{K}\frac{\delta(\mathbf{C}_k + \betav)}{\delta(\betav)} \prod_{i=1}^L \prod_{d=1}^{D} \frac{1}{ \sqrt{ 2\pi \lambda_d^i} } \exp\left( - \frac{ (\lambda_d^i + c \zeta_d^i )^2 }{2 \lambda_d^i}  \right). \nonumber
\end{eqnarray}
Then, we can derive the conditional distributions used in collapsed Gibbs sampling as follows.

\begin{algorithm}[t]
\caption{Collapsed Gibbs Sampling Algorithm for Multi-task GibbsMedLDA}\label{alg:GibbsAlg-MT}
\begin{algorithmic}[1]
   \STATE {\bfseries Initialization:} set $\lambdav = 1$ and randomly draw $z_{dk}$ from a uniform distribution.
   \FOR{$m=1$ {\bfseries to} $M$}
        \FOR{$i=1$ {\bfseries to} $L$}
            \STATE draw the classifier $\etav_i$ from the normal distribution~(\ref{eq:GibbsEta-MT})
        \ENDFOR
        \FOR{$d=1$ {\bfseries to} $D$}
            \FOR{each word $n$ in document $d$}
                 \STATE draw a topic from the multinomial distribution~(\ref{eqn:transitionProb-MT})
            \ENDFOR
            \FOR{$i=1$ {\bfseries to} $L$}
                \STATE draw $(\lambda_d^i)^{-1}$ (and thus $\lambda_d^i$) from the inverse Gaussian distribution~(\ref{eq:GibbsLambda-MT}).
            \ENDFOR
        \ENDFOR
   \ENDFOR
\end{algorithmic}
\end{algorithm}

{\bf For $\etav$}: We also assume its prior is an isotropic Gaussian $p_0(\etav) = \prod_{i=1}^L \prod_{k=1}^K \mathcal{N}(\eta_{ik}; 0, \nu^2)$. Then, we have $q(\etav | \Zv, \lambdav) = \prod_{i=1}^L q(\etav_i | \Zv, \lambdav)$, where
\begin{eqnarray}\label{eq:GibbsEta-MT}
q(\etav_i | \Zv, \lambdav ) && \propto p_0(\etav_i) \prod_{d=1}^D \exp\left( - \frac{ (\lambda_d^i + c \zeta_d^i )^2 }{2 \lambda_d^i}  \right)
                         = \mathcal{N}(\etav_i; \muv_i, \Sigmav_i),
\end{eqnarray}
where the posterior covariance matrix and posterior mean are $$\Sigmav_i = \left( \frac{1}{\nu^2}I + c^2 \sum_{d=1}^D \frac{\zvbar_d \zvbar_d^\top}{\lambda_d^i} \right)^{-1},~\muv_i = \Sigmav_i \left( c \sum_{d=1}^D y_d^i \frac{\lambda_d^i + c\ell}{\lambda_d^i}\zvbar_d \right).$$
Similarly, the inverse can be robustly done using Cholesky decomposition, an $O(K^3)$ procedure. Since $K$ is normally not large, the inversion can be done efficiently. 

{\bf For $\Zv$}: The conditional distribution of $\Zv$ is
\begin{eqnarray}
q(\Zv | \etav, \lambdav ) && \propto \prod_{d=1}^{D} \frac{\delta(\mathbf{C}_d + \alphav)}{\delta(\alphav)} \left[ \prod_{i=1}^L \exp\left( - \frac{ (\lambda_d^i + c \zeta_d^i)^2 }{2 \lambda_d^i}  \right) \right]  \prod_{k=1}^{K}\frac{\delta(\mathbf{C}_k + \betav)}{\delta(\betav)}. \nonumber
\end{eqnarray}
By canceling common factors, we can derive the conditional distribution of one variable $z_{dn}$ given others $\Zv_{\neg}$ as:
\setlength\arraycolsep{1pt} \begin{eqnarray}\label{eqn:transitionProb-MT}
 q(z_{dn}^k = 1 | \Zv_{\neg}, \etav, \lambdav, w_{dn}=t ) && \propto \frac{ (C_{k,\neg n}^{t}+\beta_t) (C_{d,\neg n}^{k}+\alpha_k) }{\sum_{t=1}^V C_{k,\neg n}^t + \sum_{t=1}^V \beta_t} \prod_{i=1}^L \exp\Big( \frac{ \gamma y_d^i (c\ell + \lambda_d^i)\eta_{ik} }{\lambda_d^i} \nonumber \\
&& ~~  - c^2 \frac{\gamma^2 \eta_{ik}^2 + 2 \gamma(1-\gamma)\eta_{ik} \Lambda_{dn}^i }{2 \lambda_d^i} \Big),
\end{eqnarray}
where $\Lambda_{dn}^i = \frac{1}{N_d-1} \sum_{k^\prime = 1}^K \eta_{i k^\prime} C_{d, \neg n}^{k^\prime}$ is the discriminant function value without word $n$. We can see that the first term is from the LDA model for observed word counts and the second term is from the supervised signal $\{ y_d^i \}$ from all the multiple tasks.

{\bf For $\lambdav$}: Finally, the conditional distribution of the augmented variables $\lambdav$ is fully factorized, $q(\lambdav | \Zv, \etav) = \prod_{i=1}^L \prod_{d=1}^D q(\lambda_d^i | \Zv, \etav)$, and each variable follows a generalized inverse Gaussian distribution
\begin{eqnarray}
 q(\lambda_d^i | \Zv, \etav) &&\propto \frac{1}{\sqrt{ 2 \pi \lambda_d^i }} \exp\left( - \frac{(\lambda_d^i + c \zeta_d^i )^2 }{2\lambda_d^i}  \right)
                          =  \GIG \left(\lambda_d^i; \frac{1}{2}, 1, c^2 (\zeta_d^i)^2 \right). \nonumber
\end{eqnarray}
Therefore, we can derive that $(\lambda_d^i)^{-1}$ follows an inverse Gaussian distribution
\begin{eqnarray}\label{eq:GibbsLambda-MT}
p( (\lambda_d^i)^{-1} | \Zv, \etav ) = \IG \left( (\lambda_d^i)^{-1}; \frac{1}{c |\zeta_d^i|}, 1 \right),
\end{eqnarray}
from which a sample can be efficiently drawn with a constant time complexity.

The per-iteration time complexity of the algorithm is $\mathcal{O}(LK^3 + N_{total} K + D L)$. For common large-scale applications where $K$ and $L$ are not too large while $D$ (thus $N_{total})$ is very large, the step of sampling latent topic assignments takes most of the time. If $L$ is very large, e.g., in the PASCAL large-scale text/image categorization challenge tasks which have tens of thousands of categories\footnote{http://lshtc.iit.demokritos.gr/;~~http://www.image-net.org/challenges/LSVRC/2012/index}, the step of drawing global classifier weights may dominate. A nice property of the algorithm is that we can easily parallelize this step since there is no coupling among these classifiers once the topic assignments are given. A preliminary investigation of the parallel algorithm is presented in~\citep{Zhu:ParallelMedLDA13}.


\section{Experiments} \label{section:experiments}
We present empirical results to demonstrate the efficiency and prediction performance of Gibbs MedLDA (denoted by GibbsMedLDA) on the 20Newsgroups data set for classification, a hotel review data set for regression, and a Wikipedia data set with more than 1 million documents for multi-label classification. We also analyze its sensitivity to key parameters and examine the learned latent topic representations qualitatively. The 20Newsgroups data set contains about 20K postings within 20 groups. We follow the same setting as in~\citep{Zhu:jmlr12} and remove a standard list of stop words for both binary and multi-class classification. For all the experiments, we use the standard normal prior $p_0(\etav)$ (i.e., $\nu^2 = 1$) and the symmetric Dirichlet priors $\alphav  = \frac{\alpha}{K} {\boldsymbol 1},~\betav = 0.01 \times {\boldsymbol 1}$, where ${\boldsymbol 1}$ is a vector with all entries being $1$. For each setting, we report the average performance and standard deviation with five randomly initialized runs. All the experiments, except the those on the large Wikipedia data set, are done on a standard desktop computer.

\subsection{Binary classification}\label{sec:binary-classification}
The binary classification task is to distinguish postings of the newsgroup \emph{alt.atheism} and postings of the newsgroup \emph{talk.religion.misc}. The training set contains 856 documents, and the test set contains 569 documents. We compare Gibbs MedLDA with the MedLDA model that uses variational methods (denoted by vMedLDA)~\citep{Zhu:jmlr12} and the MedLDA that uses collapsed Gibbs sampling algorithms (denoted by gMedLDA)~\citep{Zhu:nips12}. We also include unsupervised LDA using collapsed Gibbs sampling as a baseline, denoted by GibbsLDA. For GibbsLDA, we learn a binary linear SVM on its topic representations using SVMLight~\citep{joachims1999making}. The results of other supervised topic models, such as sLDA and DiscLDA~\citep{Simon:nips09}, were reported in~\citep{Zhu:jmlr12}. For Gibbs MedLDA, we set $\alpha=1$, $\ell=164$ and $M=10$. As we shall see in Section~\ref{sec:sensitivity-analysis}, Gibbs MedLDA is insensitive to $\alpha$, $\ell$ and $M$ in a wide range. Although tuning $c$ (e.g., via cross-validation) can produce slightly better results, we fix $c=1$ for simplicity.

\begin{figure}
\centering
\includegraphics[height=2.6in]{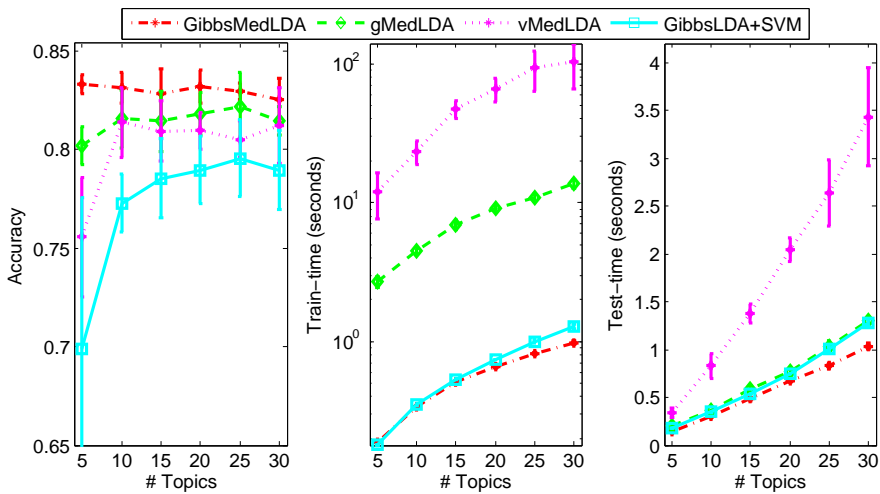}
\caption{Classification accuracy, training time (in log-scale) and testing time (in linear scale) on the 20Newsgroups binary classification data set.}\label{fig:20ngBinary}
\end{figure}

Figure~\ref{fig:20ngBinary} shows the accuracy, training time, and testing time of different methods with various numbers of topics. We can see that by minimizing an expected hinge-loss and making no restricting assumptions on the posterior distributions, GibbsMedLDA achieves higher accuracy than other max-margin topic models, which make some restricting mean-field assumptions. Similarly, as gMedLDA makes a weaker mean-field assumption, it achieves slightly higher accuracy than vMedLDA, which assumes that the posterior distribution is fully factorized. For the training time, GibbsMedLDA is about two orders of magnitudes faster than vMedLDA, and about one order of magnitude faster than gMedLDA. This is partly because both vMedLDA and gMedLDA need to solve multiple SVM problems. For the testing time, GibbsMedLDA is comparable with gMedLDA and the unsupervised GibbsLDA, but faster than the variational algorithm used by vMedLDA, especially when the number of topics $K$ is large. There are several possible reasons for the faster testing than vMedLDA, though they use the same stopping criterion. For example, vMedLDA performs mean-field inference in a full space which leads to a low convergence speed, while GibbsMedLDA carries out Gibbs sampling in a collapsed space. Also, the sparsity of the sampled topics in GibbsMedLDA could save time, while vMedLDA needs to carry out computation for each dimension of the variational parameters.

\subsection{Regression}
We use the hotel review data set~\citep{Zhu:icml10} built by randomly crawling hotel reviews from the TripAdvisor website\footnote{http://www.tripadvisor.com/} where each review is associated with a global rating score ranging from 1 to 5. In these experiments, we focus on predicting the global rating scores for reviews using the bag-of-words features only, with a vocabulary of 12,000 terms, though the other manually extracted features (e.g.,, part-of-speech tags) are provided. All the reviews have character lengths between 1,500 and 6,000. The data set consists of 5,000 reviews, with 1,000 reviews per rating. The data set is uniformly partitioned into training and testing sets. We compare the Gibbs MedLDA regression model with the MedLDA regression model that uses variational inference and supervised LDA (sLDA) which also uses variational inference. For Gibbs MedLDA and vMedLDA, the precision is set at $\epsilon = 1e^{-3}$ and $c$ is selected via 5 fold cross-validation during training. Again, we set the Dirichlet parameter $\alpha = 1$ and the number of burn-in $M=10$.

\begin{figure}
\centering
\includegraphics[height=2.6in]{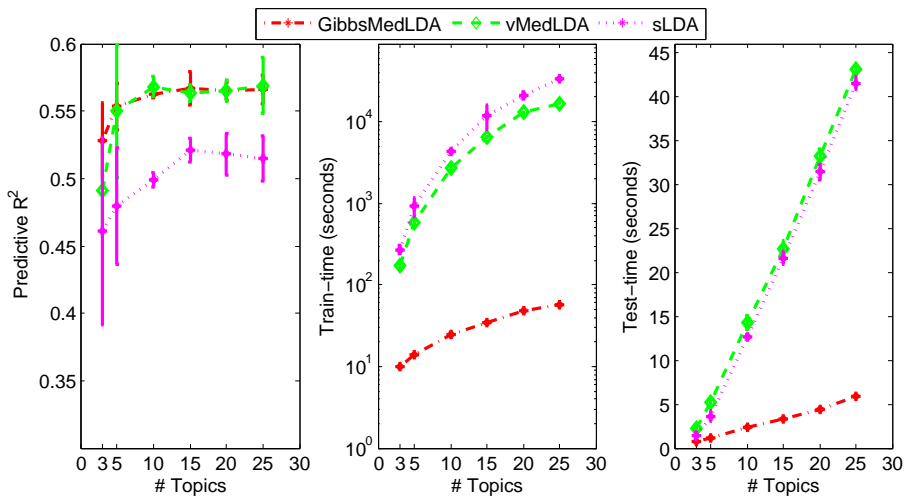}
\caption{Predictive R2, training time and testing time on the hotel review data set.}
\label{fig:hotelReview}
\end{figure}

Figure~\ref{fig:hotelReview} shows the predictive R$^2$~\citep{Blei:sLDA07} of different methods. We can see that GibbsMedLDA achieves comparable prediction performance with vMedLDA, which is better than sLDA. Note that vMedLDA uses a full likelihood model for both input words and response variables, while GibbsMedLDA uses a simpler likelihood model for words only\footnote{The MedLDA with a simple likelihood on words only doesn't perform well for regression.}. For training time, GibbsMedLDA is about two orders of magnitudes faster than vMedLDA (as well as sLDA), again due to the fact that GibbsMedLDA does not need to solve multiple SVM problems. For testing time, GibbsMedLDA is also much faster than vMedLDA and sLDA, especially when the number of topics is large, due to the same reasons as stated in Section~\ref{sec:binary-classification}.

\subsection{Multi-class classification}
We perform multi-class classification on 20Newsgroups with all 20 categories. The data set has a balanced distribution over the categories. The test set consists of 7,505 documents, in which the smallest category has 251 documents and the largest category has $399$ documents. The training set consists of 11,269 documents, in which the smallest and the largest categories contain 376 and 599 documents, respectively. We consider two approaches to doing multi-class classification --- one is to build multiple independent binary Gibbs MedLDA models, one for each category, and the other one is to build multiple dependent binary Gibbs MedLDA models under the framework of multi-task learning, as presented in Section~\ref{section:GibbsMedLDA-MT}.


\subsubsection{Multiple One-vs-All Classifiers}

Various methods exist to apply binary classifiers to do multi-class classification, including the popular ``one-vs-all" and ``one-vs-one" strategies. Here we choose the ``one-vs-all" strategy, which has shown effective~\citep{Rifkin:jmlr04}, to provide some preliminary analysis. Let $\hat{\etav}_i$ be the sampled classifier weights of the 20 ``one-vs-all" binary classifiers after the burn-in stage. For a test document $\wv$, we need to infer the latent topic assignments $\zv_i$ under each ``one-vs-all" binary classifier using a Gibbs sampler with the conditional distribution~(\ref{eq:GibbsTest}). Then, we predict the document as belonging to the single category which has the largest discriminant function value, i.e., $$\hat{y} = \argmax_{i = 1, \dots ,L} \big( \hat{\etav}_i^\top \zvbar_i \big),$$ where $L$ is the number of categories (i.e., 20 in this experiment). Again, since GibbsMedLDA is insensitive to $\alpha$ and $\ell$, we set $\alpha=1$ and $\ell=64$. We also fix $c=1$ for simplicity. The number of burn-in iterations is set as $M = 20$, which is sufficiently large as will be shown in Figure~\ref{fig:BurnIn-20ng}.

Figure~\ref{fig:20ng} shows the classification accuracy and training time, where GibbsMedLDA builds 20 binary Gibbs MedLDA classifiers. Note that for GibbsMedLDA the horizontal axis denotes the number of topics used by each single binary classifier. Since there is no coupling among these 20 binary classifiers, we can learn them in parallel, which we denote by pGibbsMedLDA.
\begin{figure}
\centering
{\hfill \subfigure[]{\includegraphics[height=2in]{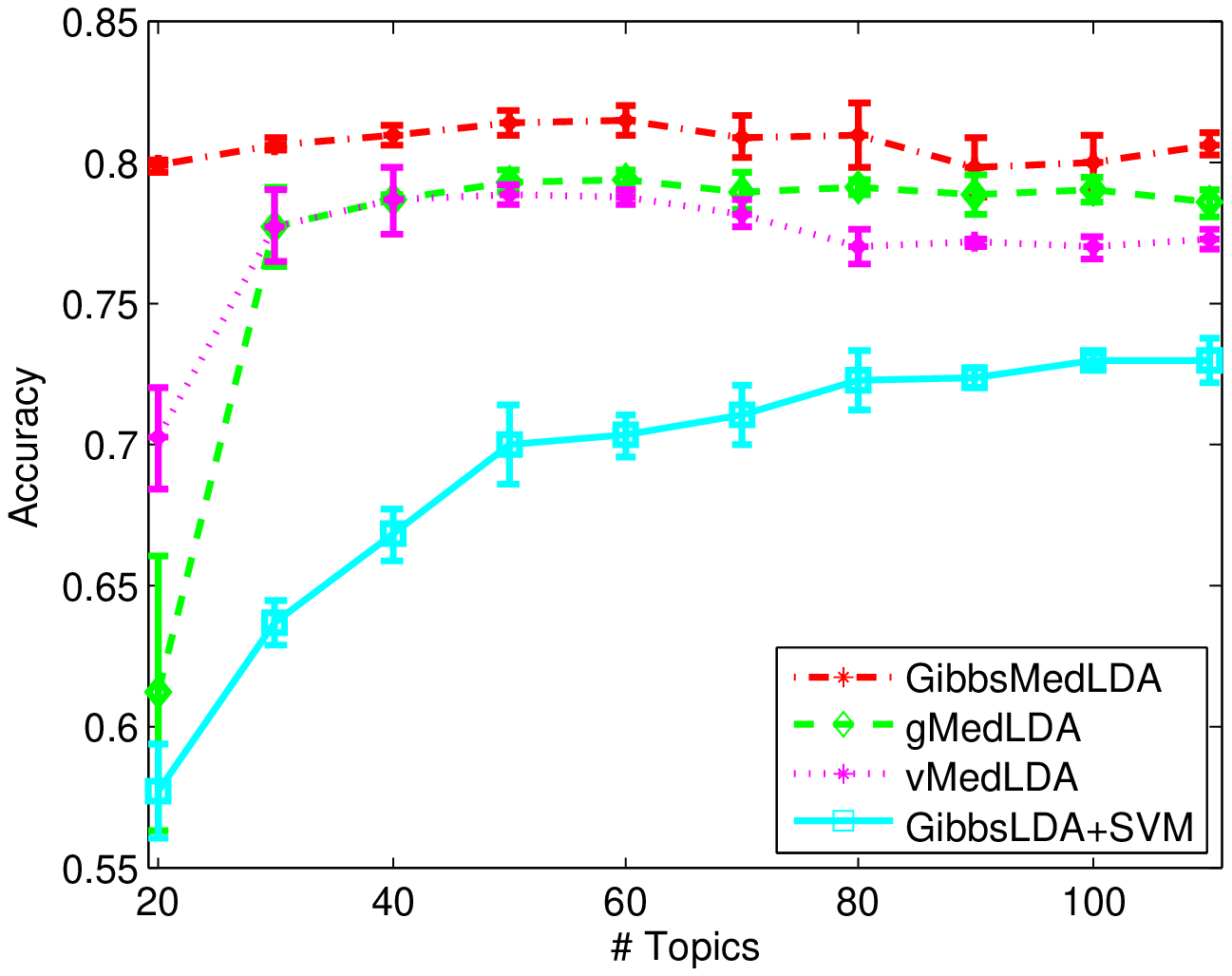}}
\hfill \subfigure[]{\includegraphics[height=2in]{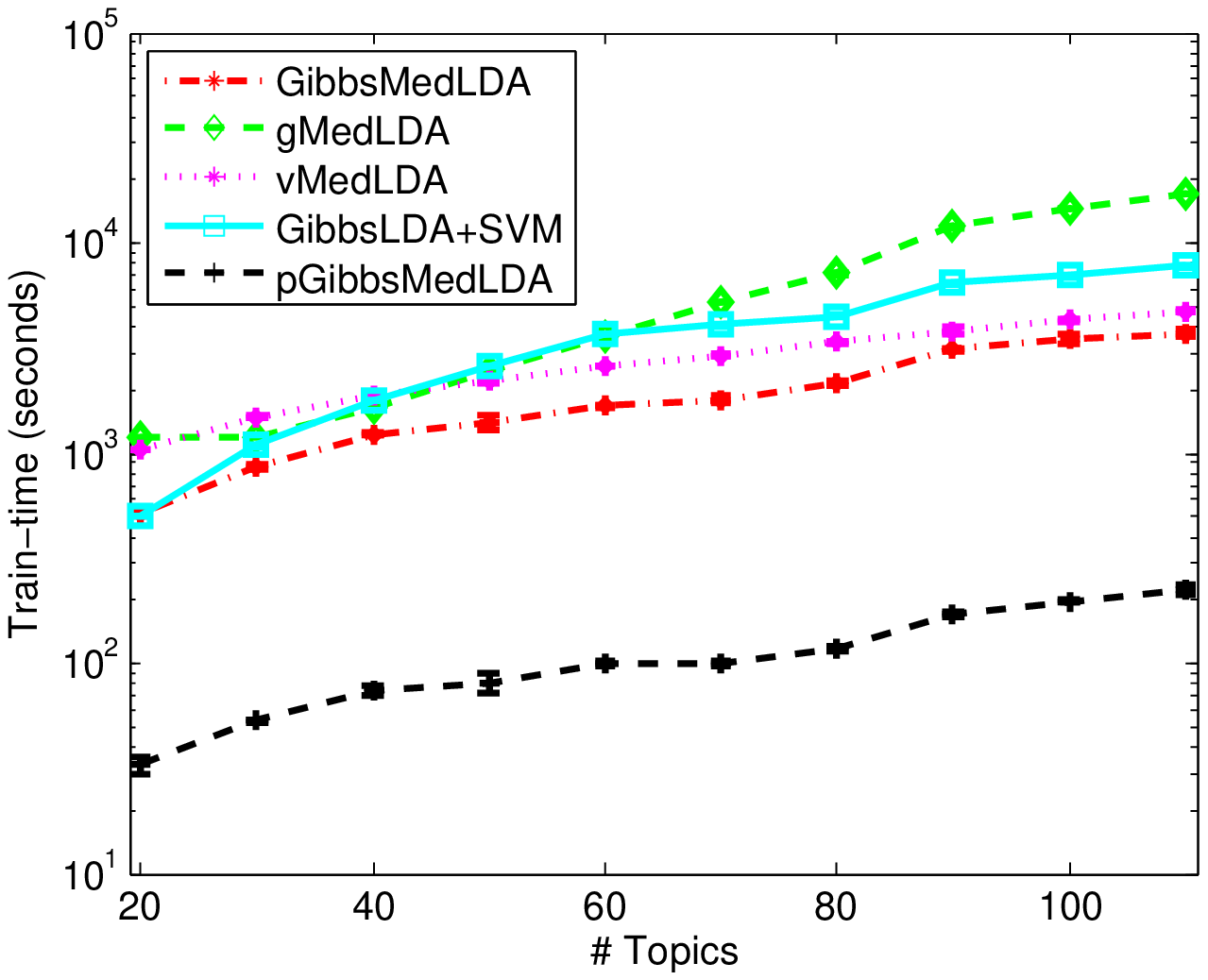}}\hfill}\vspace{-.4cm}
\caption{(a) classification accuracy and (b) training time of the one-vs-all Gibbs MedLDA classifiers for multi-class classification on the whole 20Newsgroups data set.}
\label{fig:20ng}\vspace{-.2cm}
\end{figure}
We can see that GibbsMedLDA clearly improves over other competitors on the classification accuracy, which may be due to the different strategies on building the multi-class classifiers\footnote{MedLDA learns multi-class SVM~\citep{Zhu:jmlr12}.}. However, given the performance gain on the binary classification task, we believe that the Gibbs sampling algorithm without any restricting factorization assumptions is another factor leading to the improved performance. For training time, GibbsMedLDA takes slightly less time than the variational MedLDA as well as gMedLDA. But if we train the 20 binary GibbsMedLDA classifiers in parallel, we can save a lot of training time. These results are promising since it is now not uncommon to have a desktop computer with multiple processors or a cluster with tens or hundreds of computing nodes.

\subsubsection{Multi-class Classification as a Multi-task Learning Problem}\label{sec:multi-task-multi-class}

The second approach to performing multi-class classification is to formulate it as a multiple task learning problem, with a single output prediction rule. Specifically, let the label space be $\mathcal{Y} = \{1, \dots, L\}$. We can define one binary classification task for each category $i$ and the task is to distinguish whether a data example belongs to the class $i$ (with binary label $+1$) or not (with binary label $-1$). All the binary tasks share the same topic representations. To apply the model as we have presented in Section~\ref{section:GibbsMedLDA-MT}, we need to determine the true binary label of each document in a task. Given the multi-class label $y_d$ of document $d$, this can be easily done by defining
\begin{eqnarray}
\forall i=1, \dots, L:~ y_d^i = \left\{ \begin{array}{cl}
+1 & \textrm{if}~ y_d = i \\
-1 & \textrm{otherwise}
\end{array} \right. . \nonumber
\end{eqnarray}
Then, we can learn a multi-task Gibbs MedLDA model using the data with transferred multiple labels. Let $\hat{\etav}_i$ be the sampled classifier weights of task $i$ after the burn-in stage. For a test document $\wv$, once we have inferred the latent topic assignments $\zv$ using a Gibbs sampler with the conditional distribution~(\ref{eq:GibbsTest}), we compute the discriminant function value $\hat{\etav}_i^\top \zvbar$ for each task $i$, and predict the document as belonging to the single category which has the largest discriminant function value, i.e., $$\hat{y} = \argmax_{i = 1, \dots ,L} \big( \hat{\etav}_i^\top \zvbar \big).$$
\begin{figure}
\centering
{\hfill \subfigure[]{\includegraphics[height=2in,width=1.9in]{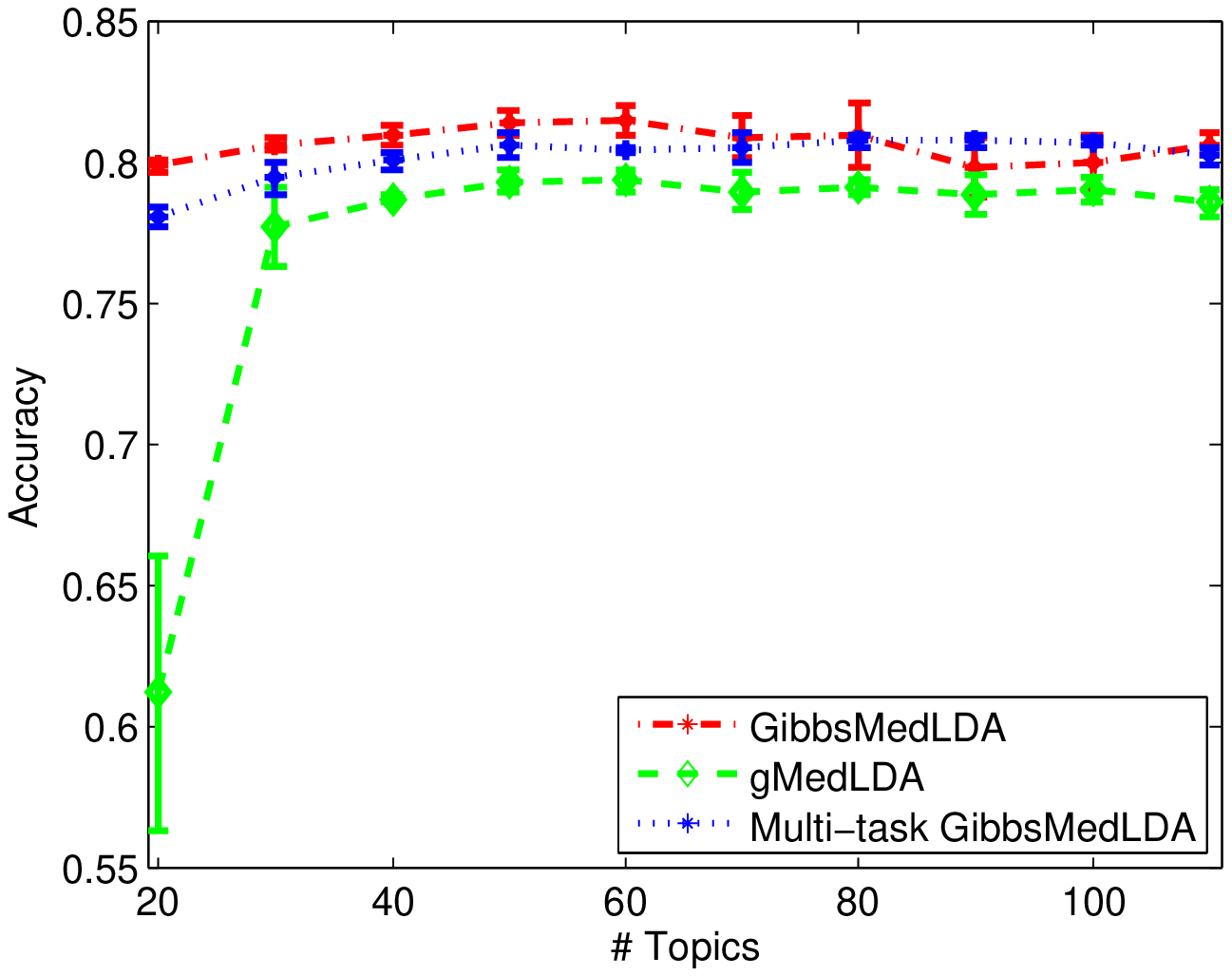}}
\hfill \subfigure[]{\includegraphics[height=2in,width=1.9in]{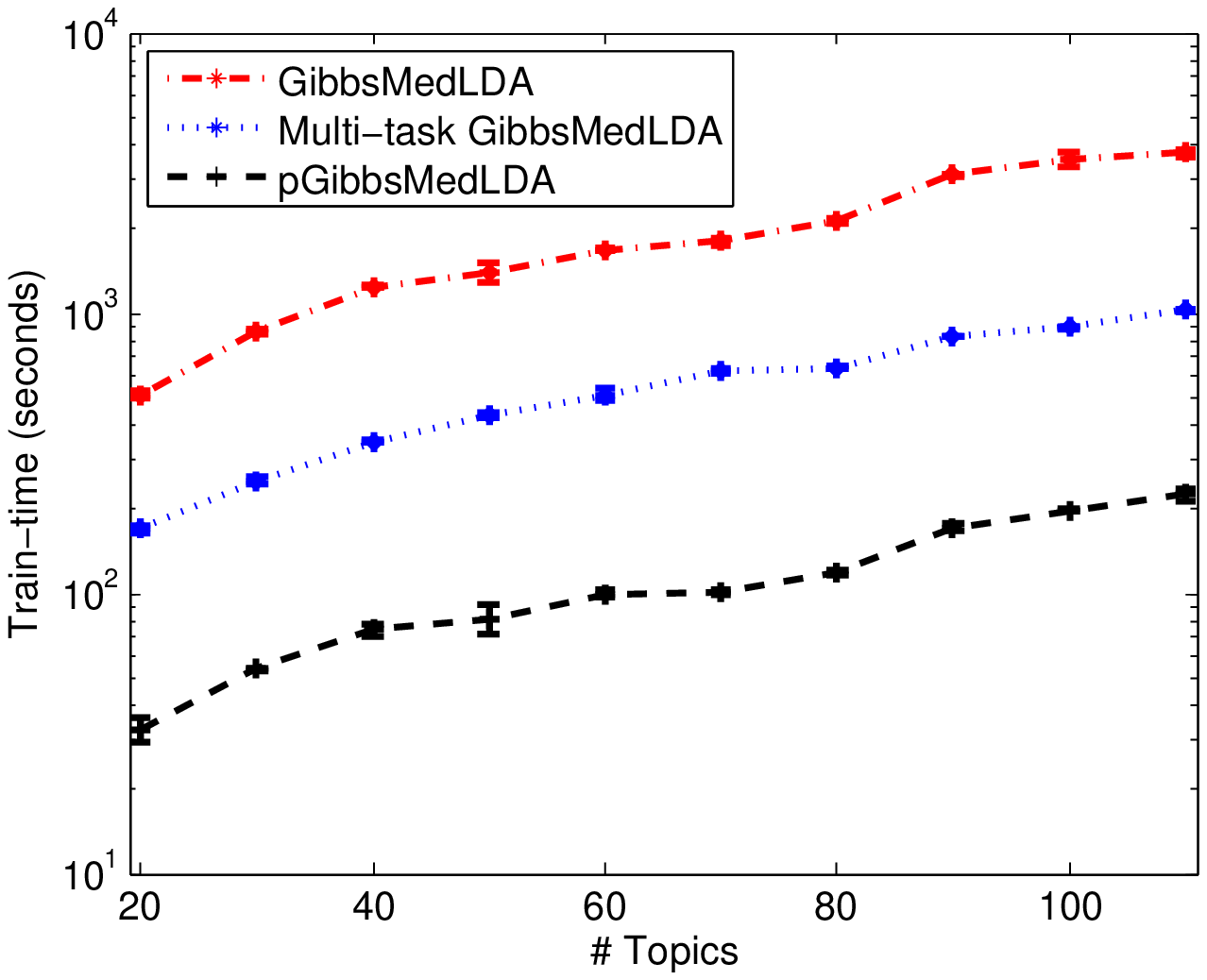}}
\hfill \subfigure[]{\includegraphics[height=2in,width=1.9in]{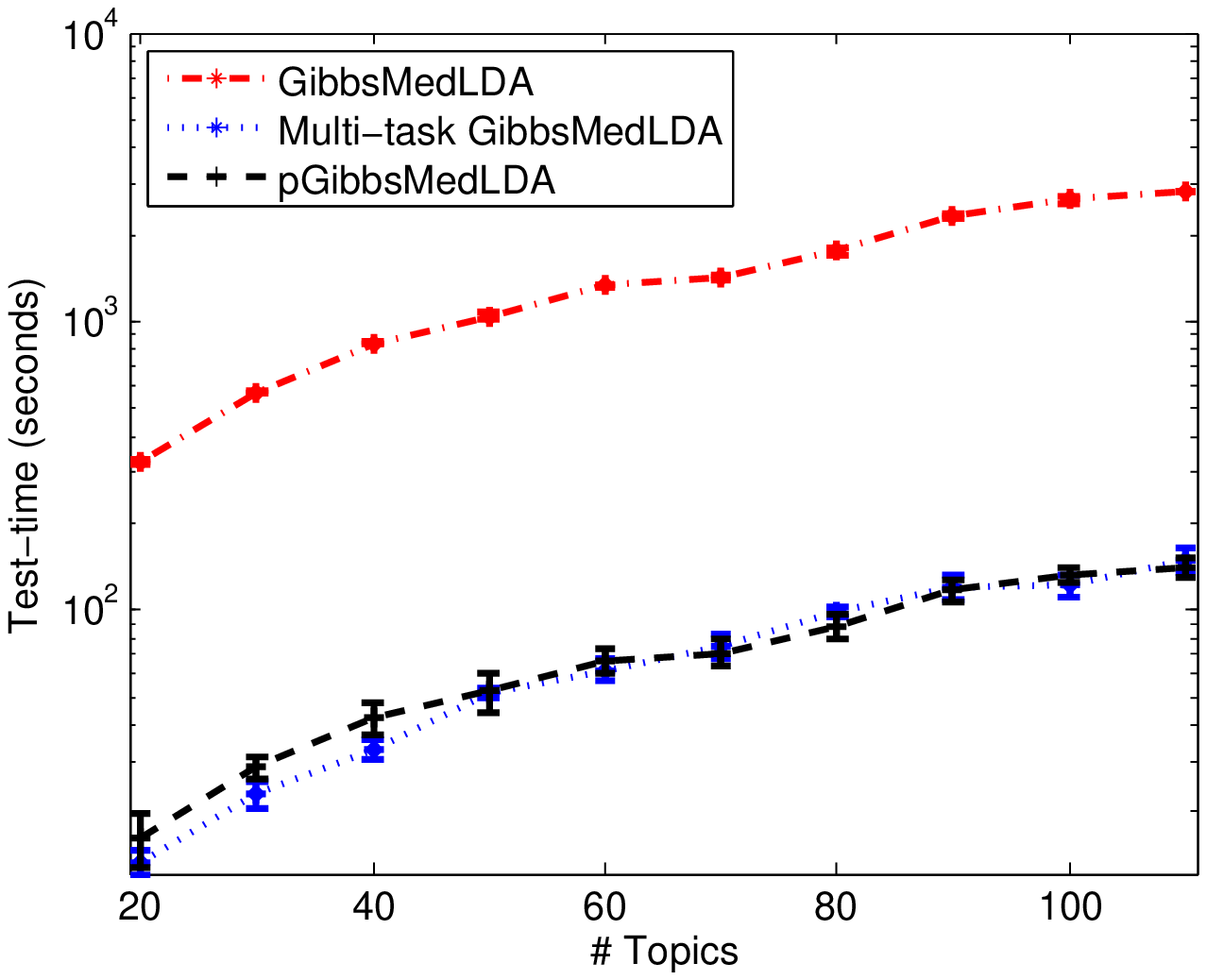}}\hfill}
\caption{(a) classification accuracy; (b) training time; and (c) testing time of the multi-task Gibbs MedLDA classifiers for multi-class classification on the whole 20Newsgroups data set.}
\label{fig:20ng-multi-task}
\end{figure}

Figure~\ref{fig:20ng-multi-task} shows the performance of the multi-task Gibbs MedLDA with comparison to the high-performance methods of the one-vs-all GibbsMedLDA and gMedLDA. Note again that for the one-vs-all GibbsMedLDA the horizontal axis denotes the number of topics used by each single binary classifier. We can see that although the multi-task GibbsMedLDA uses 20 times fewer topics than the one-vs-all GibbsMedLDA, their prediction accuracy scores are comparable when the multi-task GibbsMedLDA uses a reasonable number of topics (e.g., larger than 40). Both implementations of Gibbs MedLDA yield higher performance than gMedLDA. Looking at training time, when there is only a single processor core available, the multi-task GibbsMedLDA is about 3 times faster than the one-vs-all GibbsMedLDA. When there are multiple processor cores available, the naive parallel one-vs-all Gibbs MedLDA is faster. In this case, using 20 processor cores, the parallel one-vs-all GibbsMedLDA is about 7 times faster than the multi-task GibbsMedLDA. In some scenarios, the testing time is significant. We can see that using a single core, the multi-task GibbsMedLDA is about 20 times faster than the one-vs-all GibbsMedLDA. Again however, in the presence of multiple processor cores, in this case 20, the parallel one-vs-all GibbsMedLDA tests at least as fast, at the expense of using more processor resources. So, depending on the processor cores available, both the parallel one-vs-all GibbsMedLDA and the multi-task GibbsMedLDA can be excellent choices. Where high efficiency single-core processing is key, then the multi-task GibbsMedLDA is a great choice. When there are many processor cores available, then the parallel one-vs-all GibbsMedLDA might be an appropriate choice.


\subsection{Multi-label Classification}

\begin{figure}
\centering
\includegraphics[height=2.5in,width=5.3in]{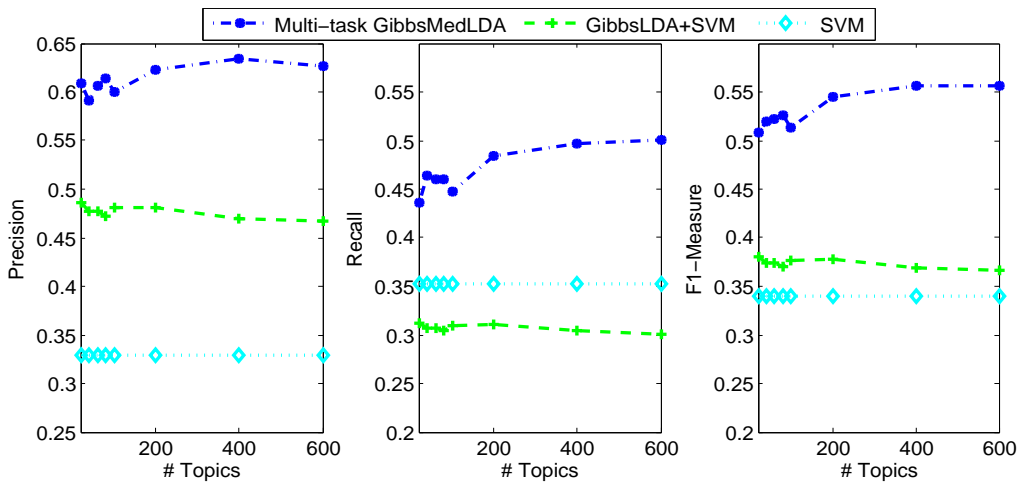}
\caption{Precision, recall and F1-measure of the multi-label classification using different models on the Wiki data set.}
\label{fig:wiki-multi-label}
\end{figure}

We also present some results on a multi-label classification task. We use the Wiki data set which is built from the large Wikipedia set used in the PASCAL LSHC challenge 2012, and where each document has multiple labels. The original data set\footnote{Available at: http://lshtc.iit.demokritos.gr/} is extremely imbalanced. We built our data set by selecting the 20 categories that have the largest numbers of documents and keeping all the documents that are labeled by at least one of these 20 categories. The training set consists of 1.1 millions of documents and the testing set consists of 5,000 documents. The vocabulary has 917,683 terms in total. To examine the effectiveness of Gibbs MedLDA, which performs topic discovery and classifier learning jointly, we compare it with a linear SVM classifier built on the raw bag-of-words features and a two-step approach denoted by GibbsLDA+SVM. The GibbsLDA+SVM method first uses LDA with collapsed Gibbs sampling to discover latent topic representations for all the documents and then builds 20 separate binary SVM classifiers using the training documents with their discovered topic representations. For multi-task Gibbs MedLDA, we use 40 burn-in steps, which is sufficiently large. The model is insensitive to other parameters, similar to the multi-class classification task.


Figure~\ref{fig:wiki-multi-label} shows the precision, recall and F1 measure (i.e., the harmonic mean of precision and recall) of various models running on a distributed cluster with 20 nodes (each node is equipped with two 6-core CPUs)\footnote{For GibbsLDA, we use the parallel implementation in Yahoo-LDA, which is publicly available at: https://github.com/shravanmn/Yahoo\_LDA. For Gibbs MedLDA, the parallel implementation of our Gibbs sampler is presented in~\citep{Zhu:ParallelMedLDA13}.}. We can see that the multi-task Gibbs MedLDA performs much better than other competitors. There are several reasons for the improvements. Since the vocabulary has about 1 million terms, the raw features are in a high-dimensional space and each document gives rise to a sparse feature vector (i.e., only a few elements are nonzero). Thus, learning SVM classifiers on the raw data leads not just to over-fitting but a wider failure to generalize. For example, two documents from the same category might contain non-intersecting sets of words, yet contain similar latent topics. Using LDA to discover latent topic representations can produce dense features. Building SVM classifiers using the latent topic features improves the overall F1 measure, by improving the ability to generalize, and reducing overfitting. But, due to its two-step procedure, the discovered topic representations may not be very predictive. By doing max-margin learning and topic discovery jointly, the multi-task GibbsMedLDA can discover more discriminative topic features, thus improving significantly over the two-step GibbsLDA+SVM algorithm.

\subsection{Sensitivity analysis}\label{sec:sensitivity-analysis}
We now provide a more careful analysis of the various Gibbs MedLDA models on their sensitivity to some key parameters in the classification tasks. Specifically, we will look at the effects of the number of burn-in steps, the Dirichlet prior $\alpha$, the loss penalty $\ell$, and the number of testing samples.

\subsubsection{Burn-in Steps}\label{sec:sensitivity-burn-in}

\begin{figure}[t]
\centering
\includegraphics[height=2.45in]{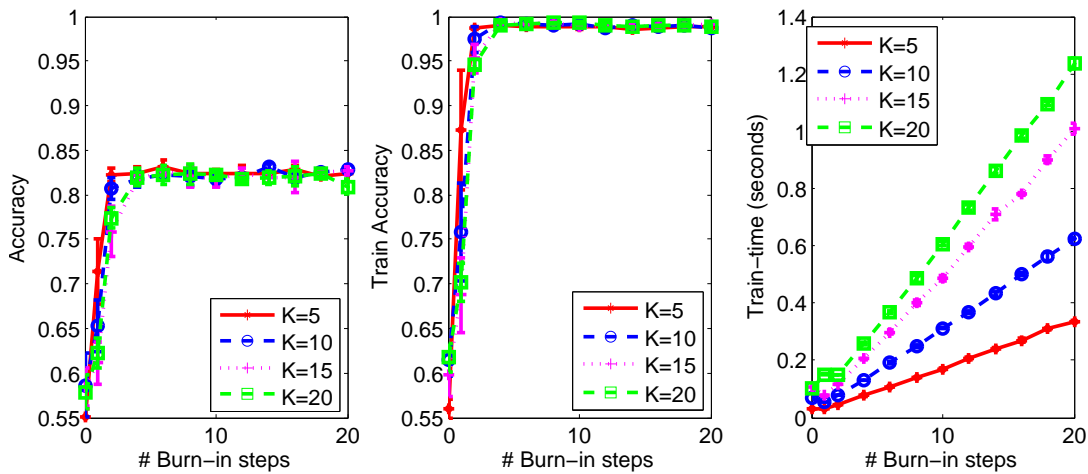}
\caption{(Left) testing accuracy, (Middle) training accuracy, and (Right) training time of GibbsMedLDA with different numbers of burn-in steps for binary classification.}
\label{fig:BurnIn}
\end{figure}

Figure~\ref{fig:BurnIn} shows the classification accuracy, training accuracy and training time of GibbsMedLDA with different numbers of burn-in samples in the binary classification task. When $M=0$, the model is essentially random, for which we draw a classifier with the randomly initialized topic assignments for training data. We can see that both the training accuracy and testing accuracy increase very quickly and converge to their stable values with 5 to 10 burn-in steps. As expected, the training time increases about linearly in general when using more burn-in steps. Moreover, the training time increases linearly as $K$ increases. In the previous experiments, we have chosen $M=10$, which is sufficiently large.

\begin{figure}[t]
\centering
\includegraphics[height=2.45in]{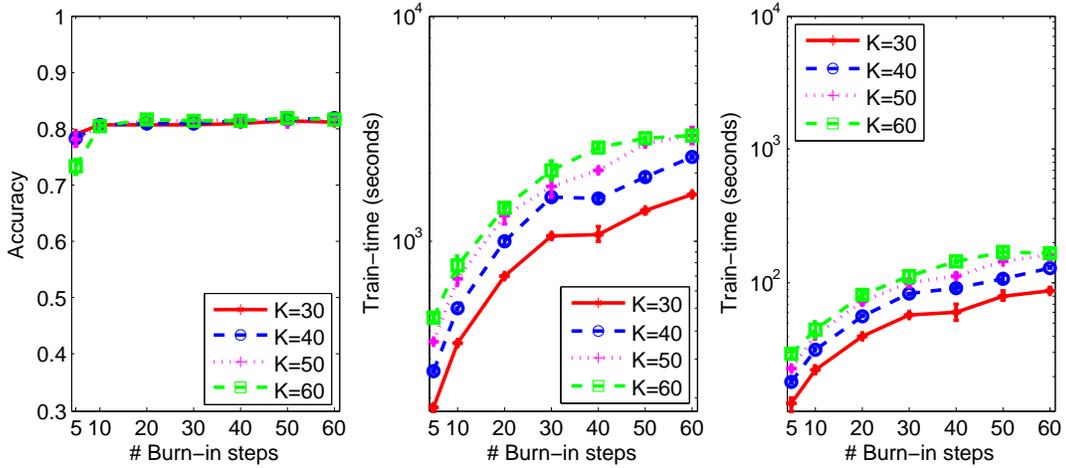}
\caption{(Left) classification accuracy of GibbsMedLDA, (Middle) training time of GibbsMedLDA and (Right) training time of the parallel pGibbsMedLDA with different numbers of burn-in steps for multi-class classification.}\vspace{-.2cm}
\label{fig:BurnIn-20ng}
\end{figure}

\begin{figure}[t]
\centering
\includegraphics[height=2.45in]{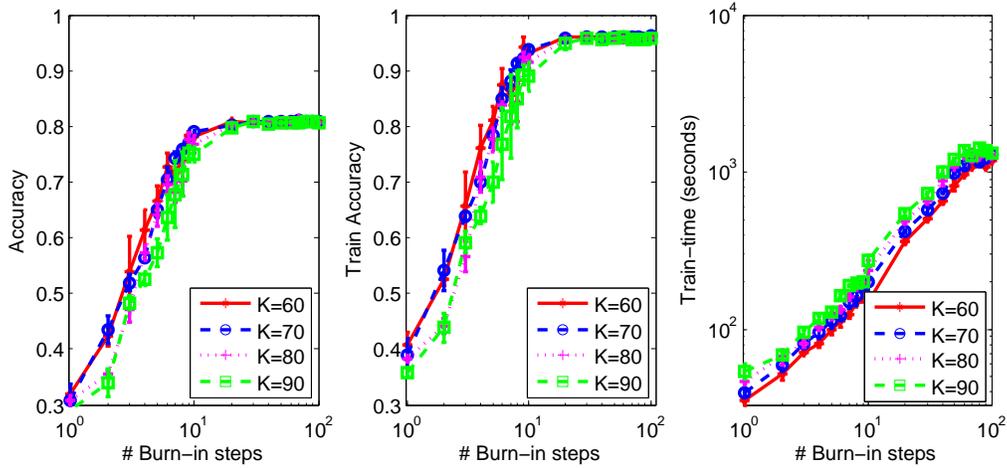}
\caption{(Left) test accuracy, (Middle) training accuracy, and (Right) training time of the multi-task GibbsMedLDA with different numbers of burn-in steps for multi-class classification.}
\label{fig:BurnIn-20ng-Multi-task}
\end{figure}

Figure~\ref{fig:BurnIn-20ng} shows the performance of GibbsMedLDA for multi-class classification with different numbers of burn-in steps when using the one-vs-all strategy. We show the total training time as well as the training time of the naive parallel implementation of pGibbsMedLDA. We can see that when the number of burn-in steps is larger than 20, the performance is quite stable, especially when $K$ is large. Again, the training time grows about linearly as the number of burn-in steps increases. Even if we use 40 or 60 steps of burn-in, the training time is still competitive, compared with the variational MedLDA, especially considering that GibbsMedLDA can be naively parallelized by learning different binary classifiers simultaneously.

Figure~\ref{fig:BurnIn-20ng-Multi-task} shows the testing classification accuracy, training accuracy and training time of the multi-task Gibbs MedLDA for multi-class classification with different numbers of burn-in steps. We can see that again both the training accuracy and testing accuracy increase fast and converge to their stable scores after about 30 burn-in steps. Also, the training time increases about linearly as the number of burn-in steps increases.

\subsubsection{ Dirichlet prior $\alphav$}

For topic models with a Dirichlet prior, the Dirichlet hyper-parameter can be automatically estimated, such as using the Newton-Raphson method~\citep{Blei:03}. Here, we analyze its effects on the performance by setting different values. Figure~\ref{fig:Sensitivity-Alpha} shows the classification performance of GibbsMedLDA on the binary task with different $\alpha$ values for the symmetric Dirichlet prior $\alphav = \frac{\alpha}{K} {\boldsymbol 1}$. For the three different topic numbers, we can see that the performance is quite stable in a wide range of $\alpha$ values, e.g., from $0.1$ to $10$. We can also see that it generally needs a larger $\alpha$ in order to get the best results when $K$ becomes larger (e.g., when $\alpha < 0.1$, using fewer topics results in slightly higher performance). This is mainly because a large $K$ tends to produce sparse topic representations and an appropriately large $\alpha$ is needed to smooth the representations, as the effective Dirichlet prior is $\alpha_k = \alpha / K$.

\begin{figure}
\centering
\includegraphics[height=2.1in]{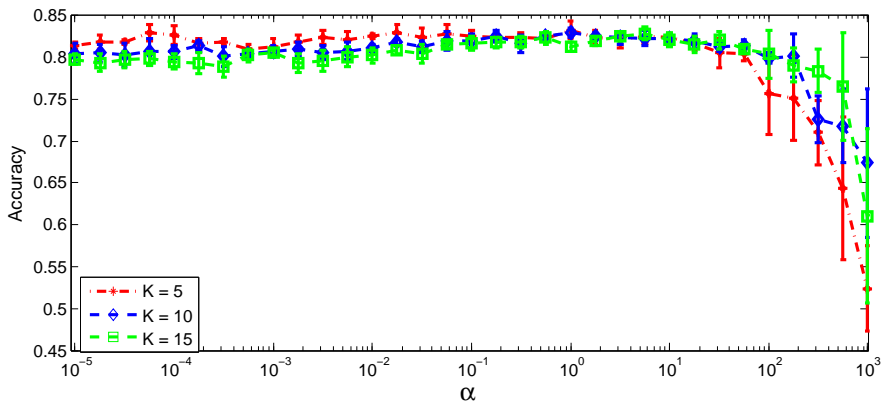}
\caption{Classification accuracy of GibbsMedLDA on the binary classification data set with different $\alpha$ values.}\label{fig:Sensitivity-Alpha}
\end{figure}

\begin{figure}[t]
\centering
\includegraphics[height=2.2in]{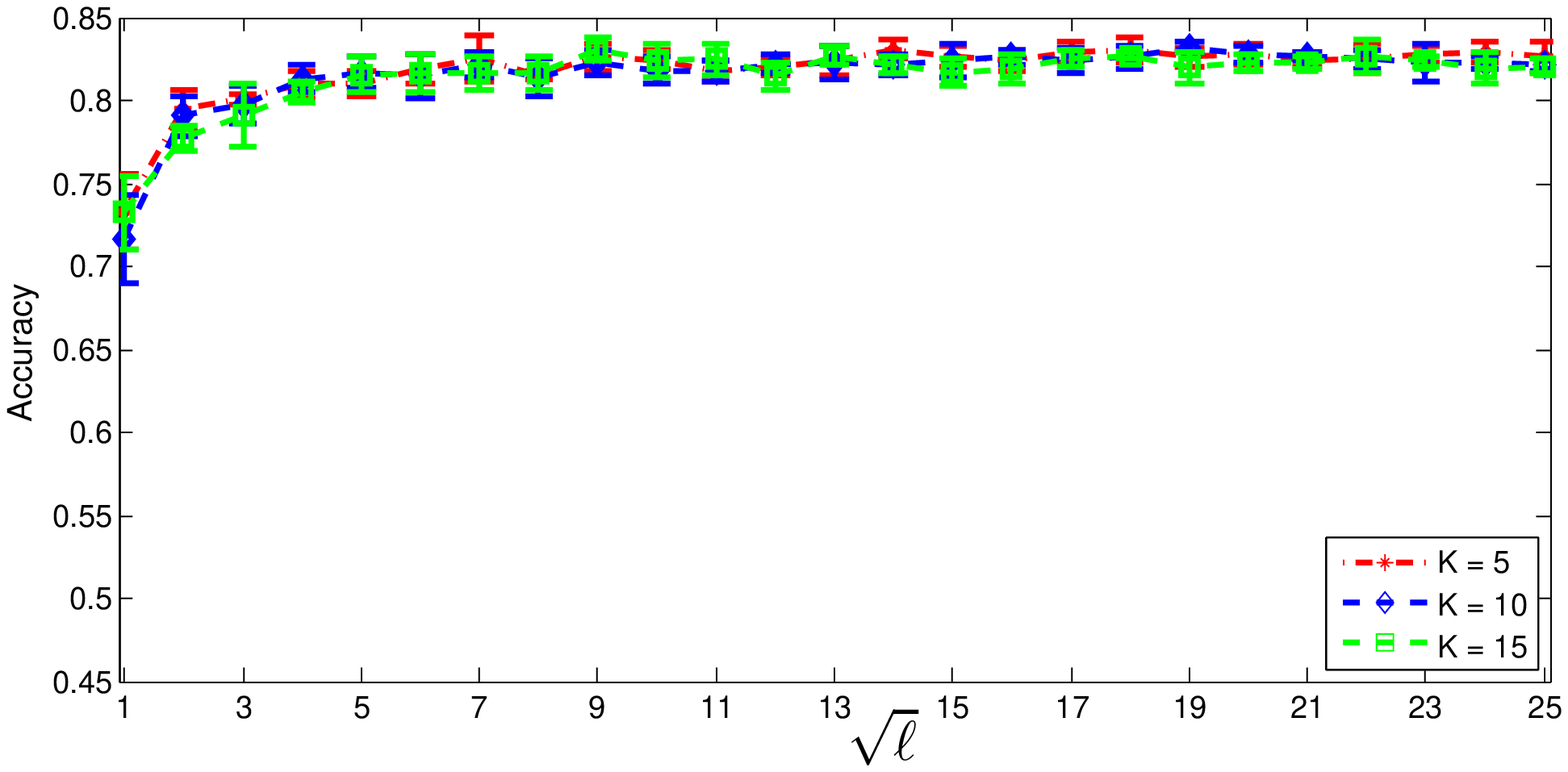}
\caption{Classification accuracy of GibbsMedLDA on the binary classification data set with different $\ell$ values.}\label{fig:Sensitivity-Ell}
\end{figure}

\begin{figure}[t]
\centering
\includegraphics[height=2.2in]{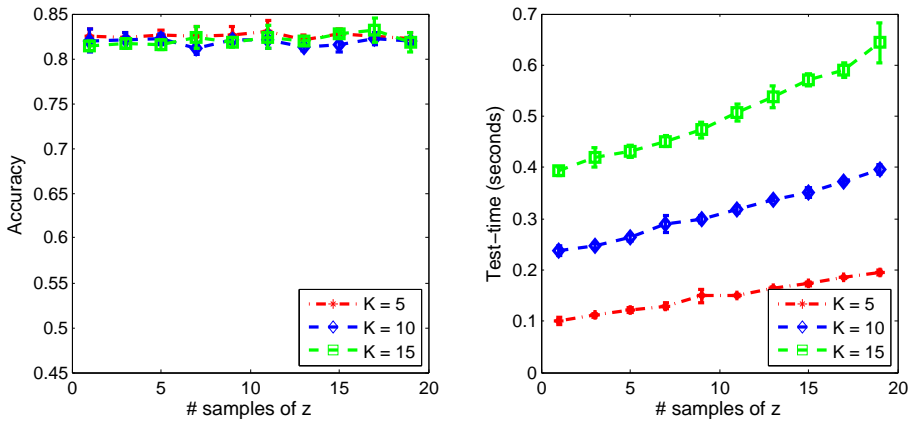}
\caption{(Left) classification accuracy and (Right) testing time of GibbsMedLDA on the binary classification data set with different numbers of $\zv$ samples in making predictions.}\label{fig:Sensitivity-Test-Lag}
\end{figure}

\subsubsection{Loss penalty $\ell$}
Figure~\ref{fig:Sensitivity-Ell} shows the classification performance of GibbsMedLDA on the binary classification task with different $\ell$ values. Again, we can see that in a wide range, e.g., from 25 to 625, the performance is quite stable for all the three different $K$ values. In the above experiments, we set $\ell = 164$. For the multi-class classification task, we have similar observations, and we set $\ell=64$ in the previous experiments.

\subsubsection{The number of testing samples}\label{sec:sensitivity-test-lag}
Figure~\ref{fig:Sensitivity-Test-Lag} shows the classification performance and testing time of GibbsMedLDA in the binary classification task with different numbers of $\zv$ samples when making predictions, as stated in Section~\ref{sec:prediction}. We can see that in a wide range, e.g., from 1 to 19, the classification performance is quite stable for all the three different $K$ values we have tested; and the testing time increases about linearly as the number of $\zv$ samples increases. For the multi-class classification task, we have similar observations.

\subsection{Topic Representations}

Finally, we also visualize the discovered latent topic representations of Gibbs MedLDA on the 20Newsgroup data set. We choose the multi-task Gibbs MedLDA, since it learns a single common topic space shared by multiple classifiers. We set the number of topics at 40. Figure~\ref{fig:20ng-topics} shows the average topic representations of the documents from each category, and Table~\ref{table:gibbs-medlda-topics} presents the 10 most probable words in each topic. We can see that for different categories, the average representations are quite different, indicating that the topic representations are good at distinguishing documents from different classes. We can also see that on average the documents in each category have very few salient topics (i.e., topics with a high probability of describing the documents). For example, the first two most salient topics for describing the documents in the category {\it alt.atheism} are topic 20 and topic 29, whose top-ranked words (see Table~\ref{table:gibbs-medlda-topics}) reflect the semantic meaning of the category. For {\it graphics} category, the documents have the most salient topic 23, which has topic words {\it image}, {\it graphics}, {\it file}, {\it jpeg}, and etc., all of which are closely related to the semantic of graphics. For other categories, we have similar observations.

\begin{figure}\vspace{-.4cm}
\centering
{\hfill \subfigure[alt.atheism]{\includegraphics[height=1.25in,width=1.4in]{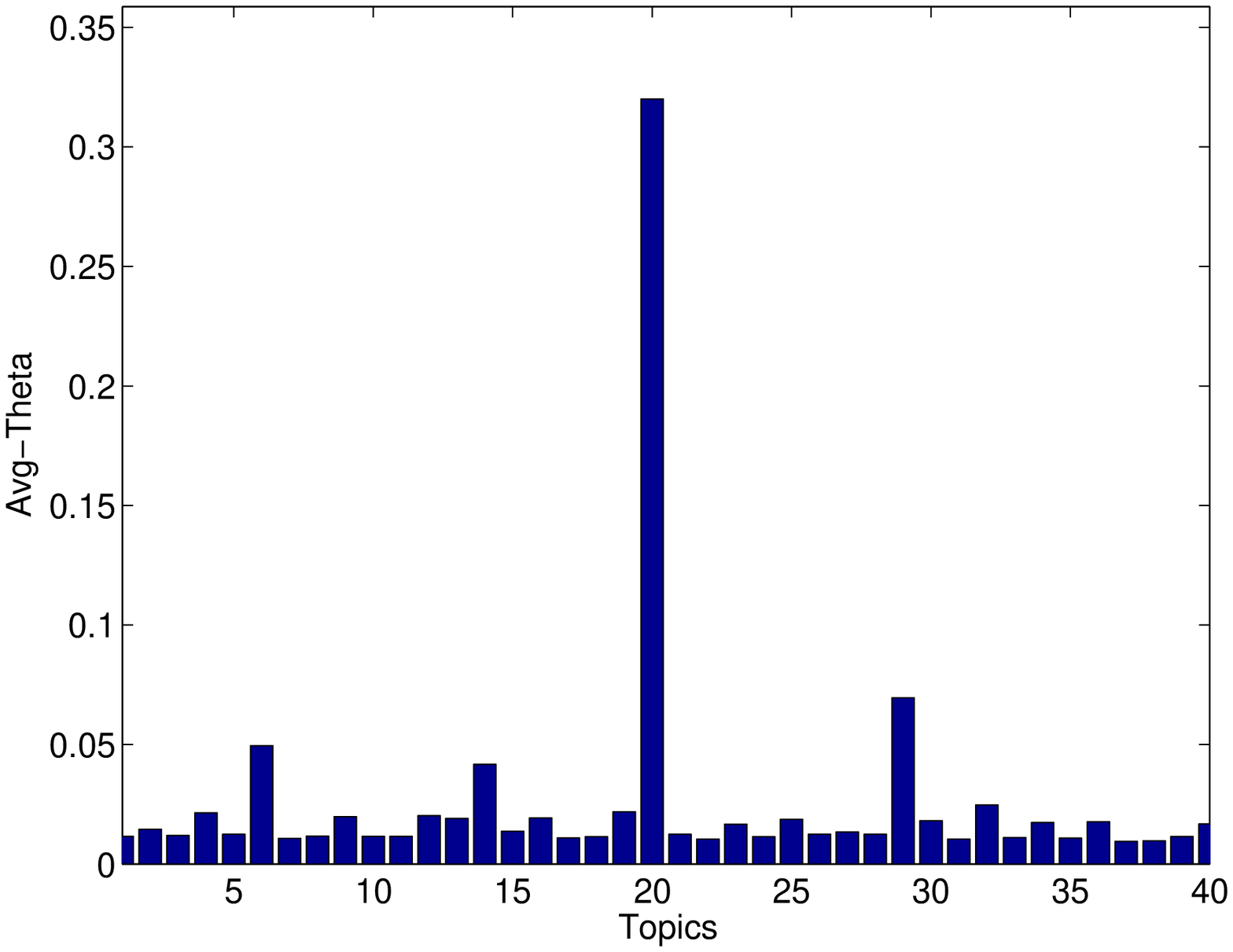}}\vspace{-.13cm}
\hfill \subfigure[graphics]{\includegraphics[height=1.25in,width=1.4in]{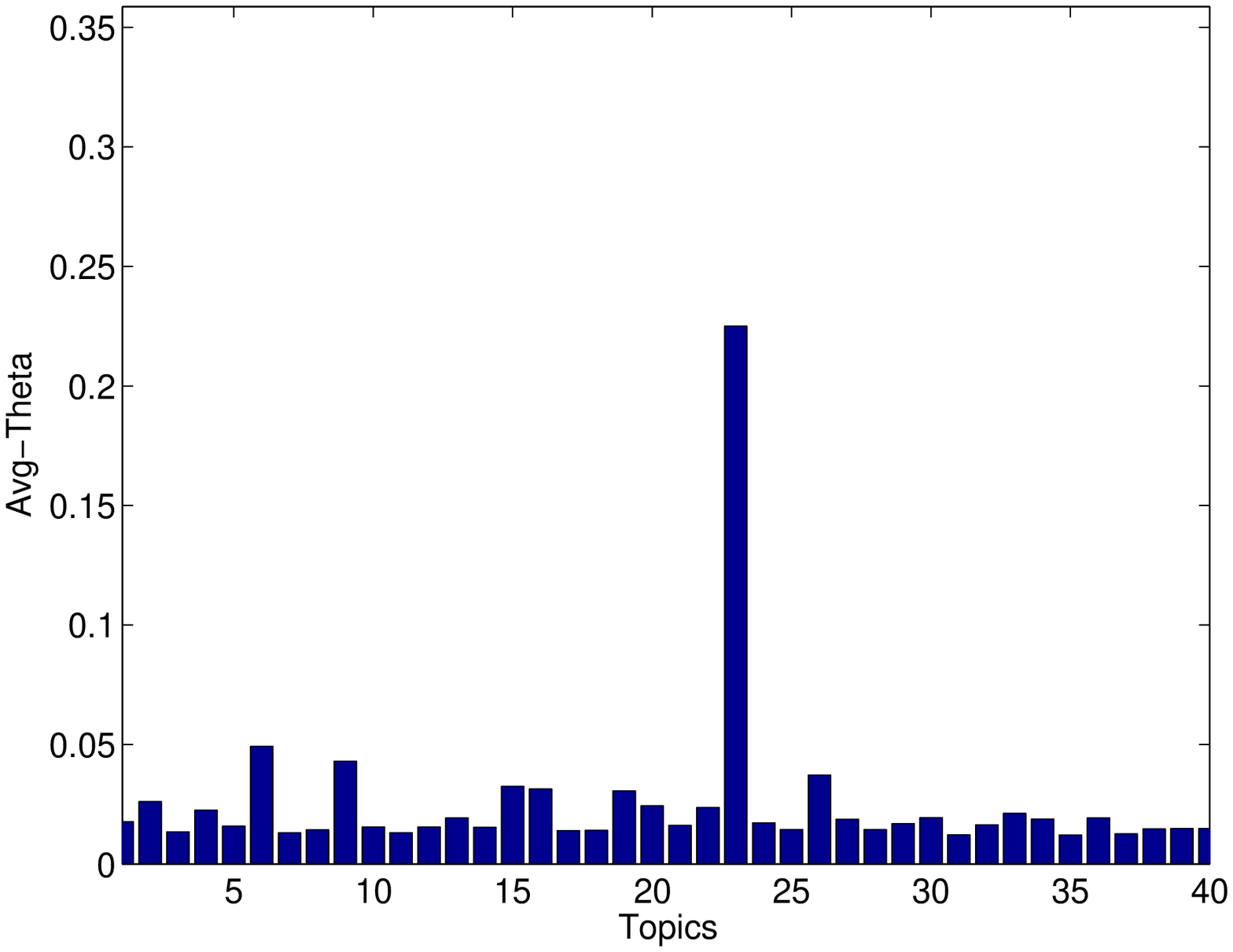}}\vspace{-.13cm}
\hfill \subfigure[ms-windows.misc]{\includegraphics[height=1.25in,width=1.4in]{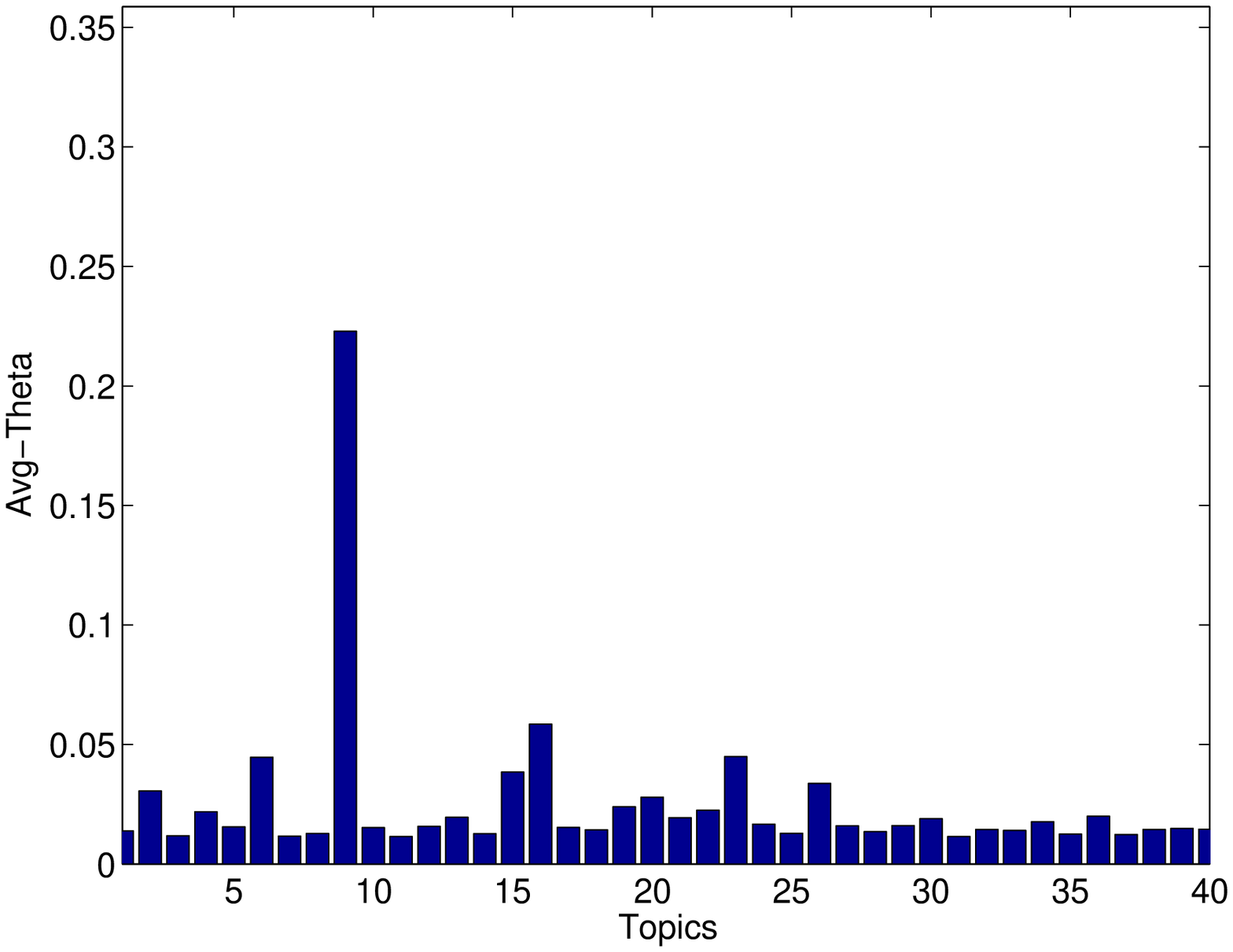}}\vspace{-.13cm}
\hfill \subfigure[ibm.pc.hardware]{\includegraphics[height=1.25in,width=1.4in]{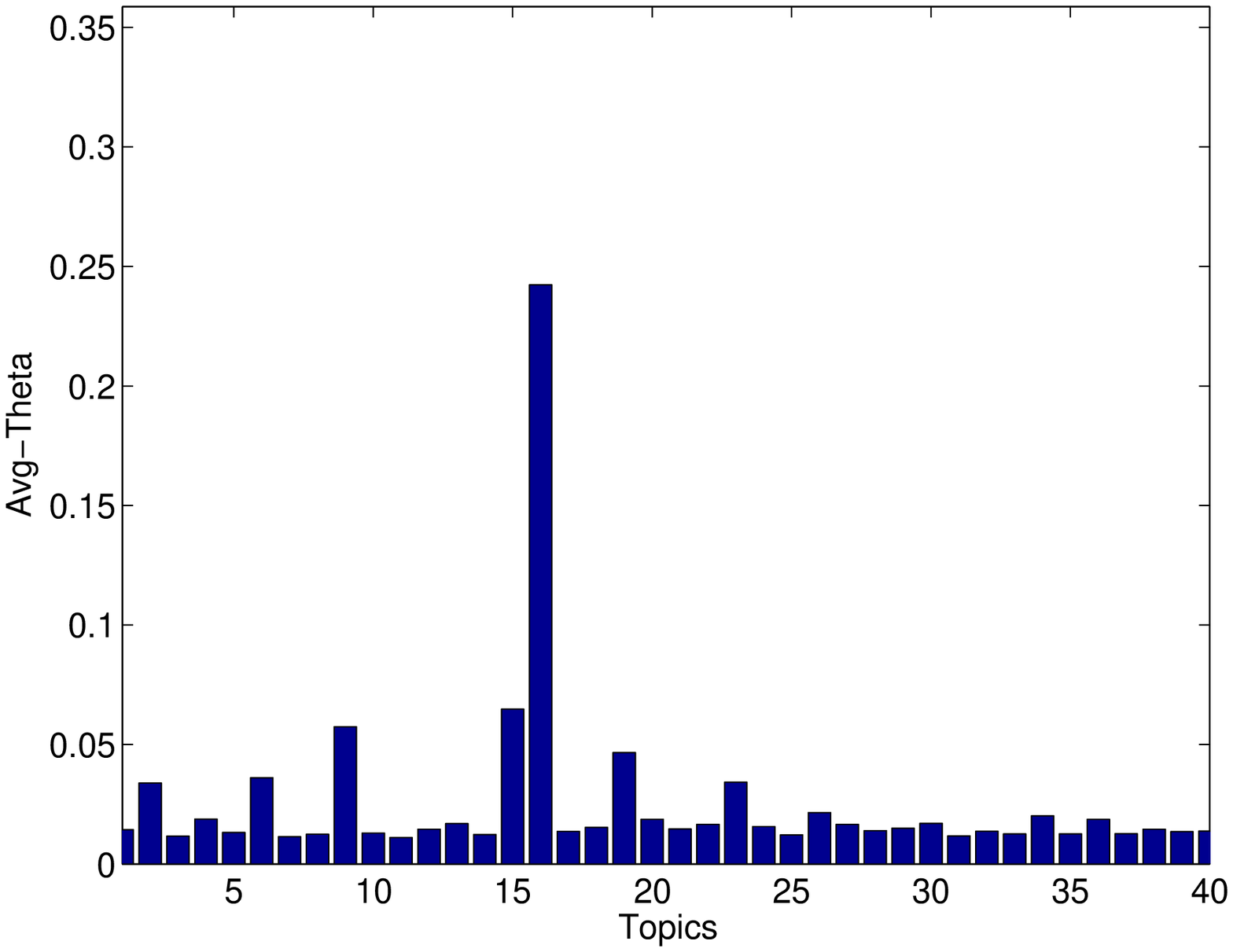}}\vspace{-.13cm}
\hfill \subfigure[sys.mac.hardware]{\includegraphics[height=1.25in,width=1.4in]{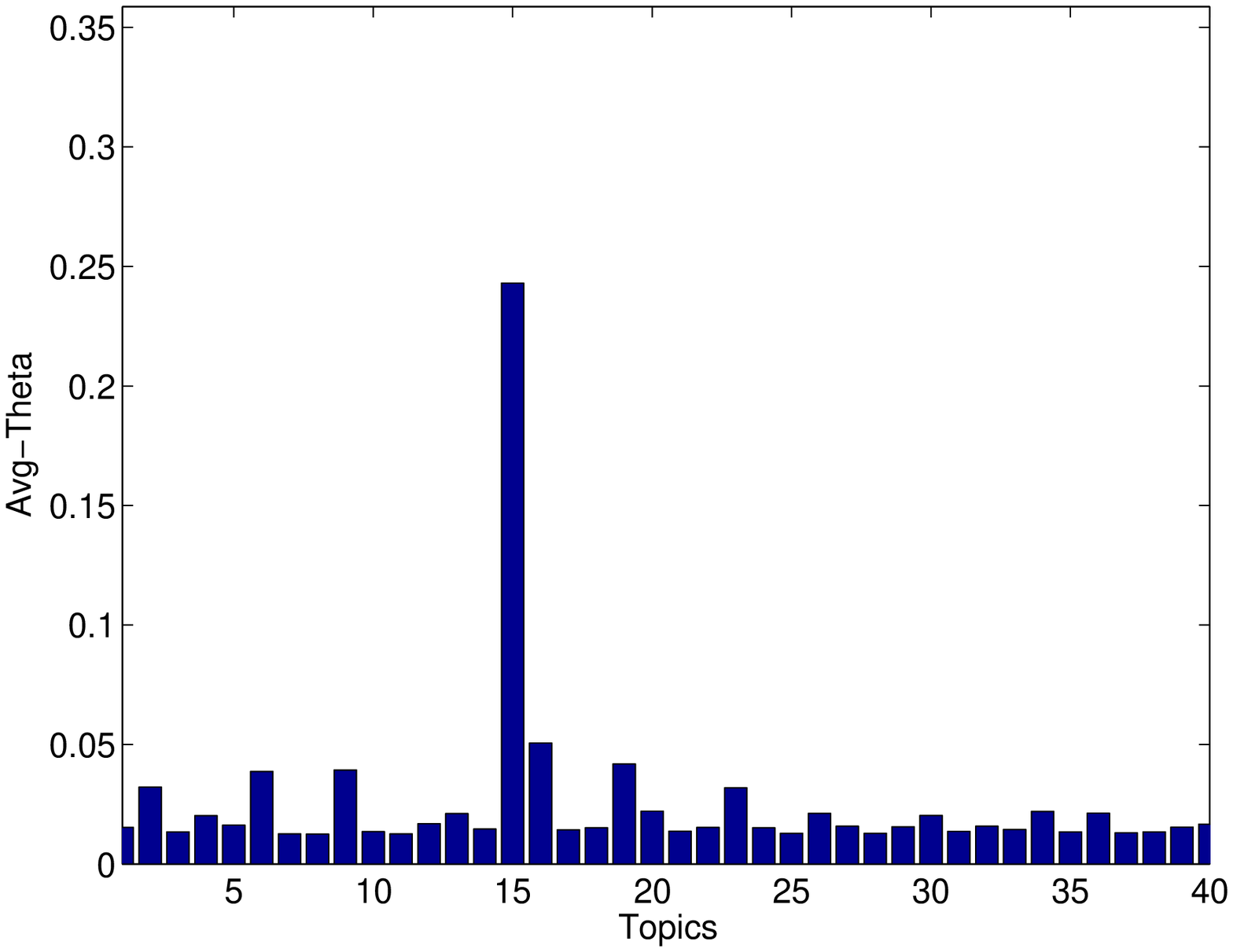}}\vspace{-.13cm}
\hfill \subfigure[windows.x]{\includegraphics[height=1.25in,width=1.4in]{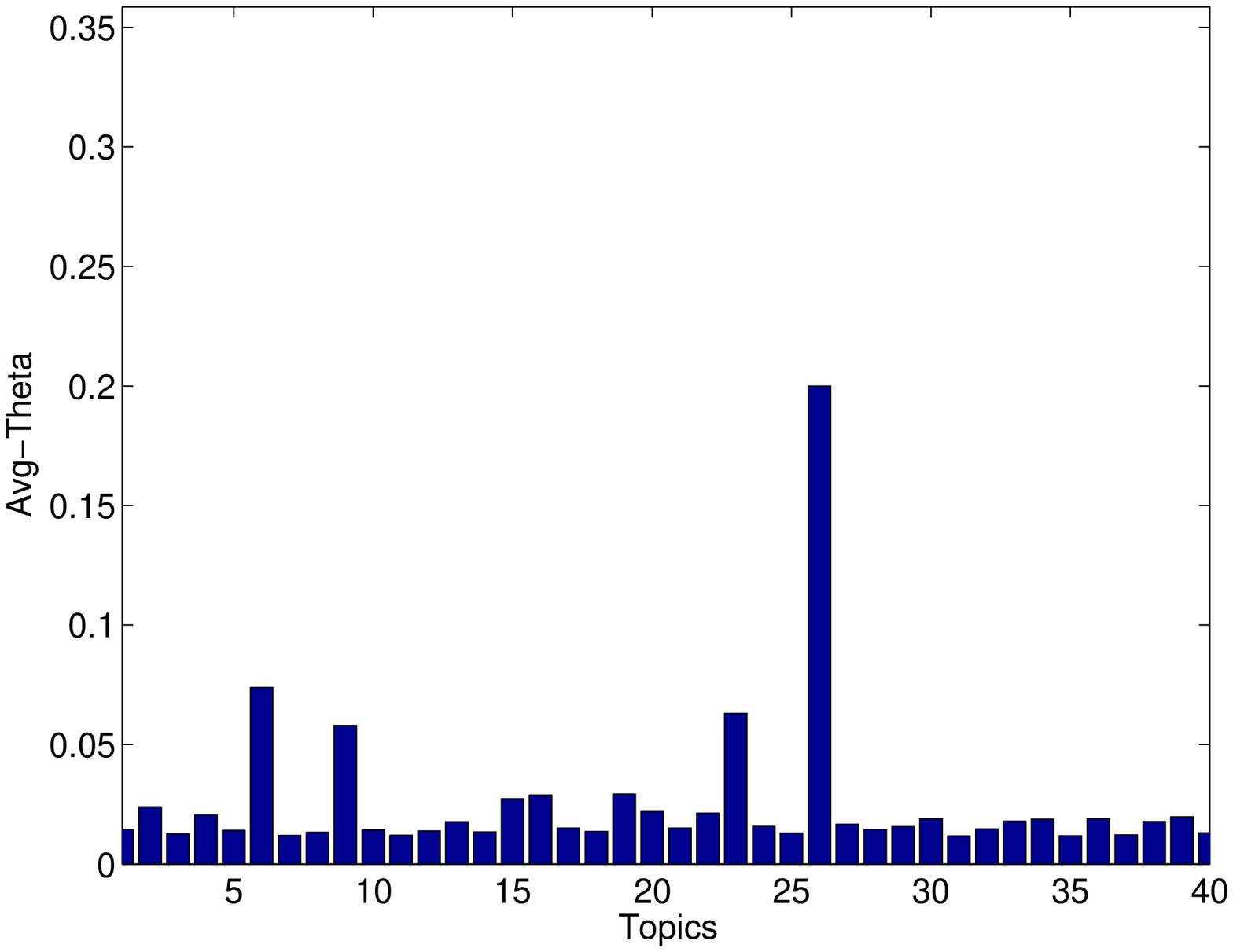}}\vspace{-.13cm}
\hfill \subfigure[misc.forsale]{\includegraphics[height=1.25in,width=1.4in]{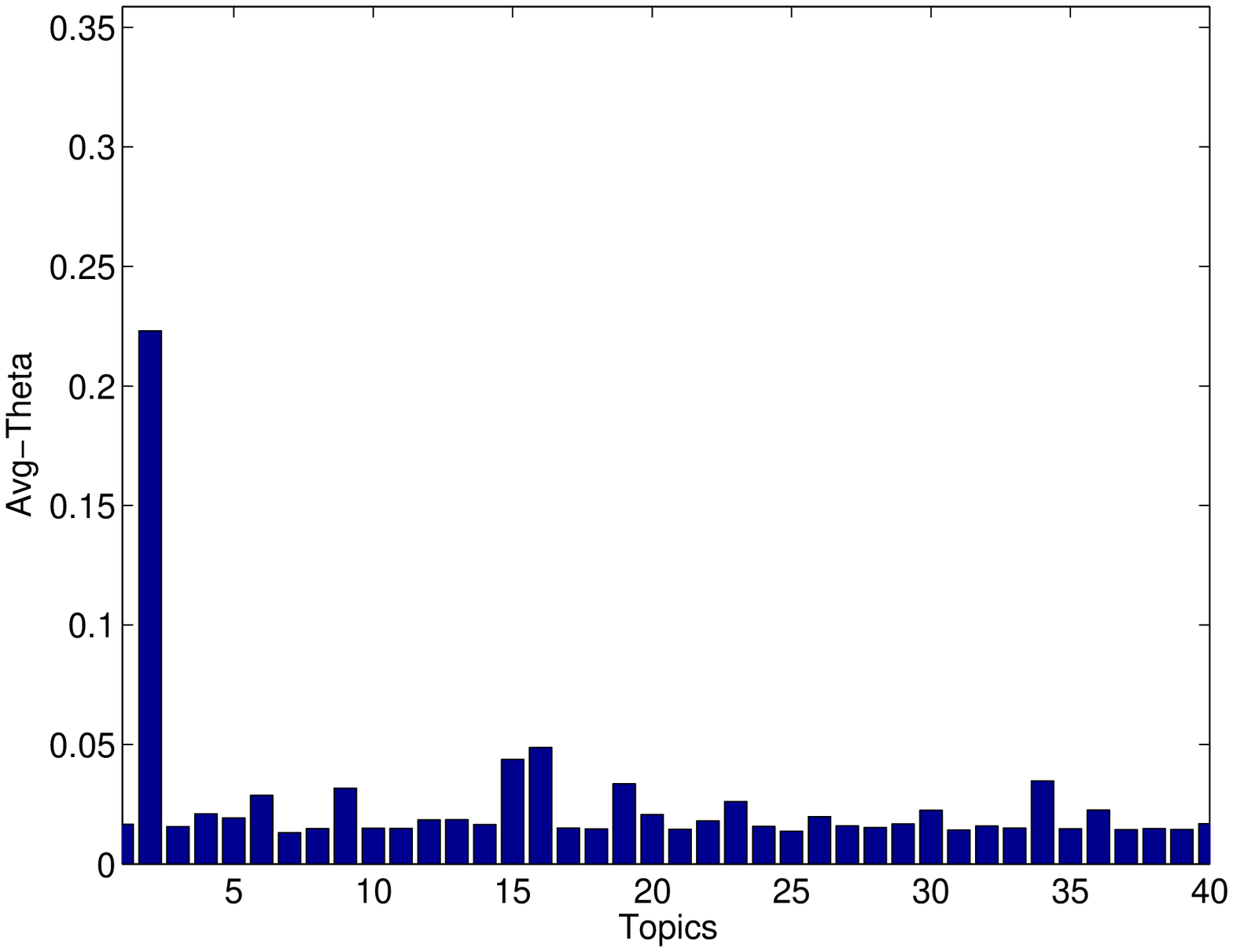}}\vspace{-.13cm}
\hfill \subfigure[rec.autos]{\includegraphics[height=1.25in,width=1.4in]{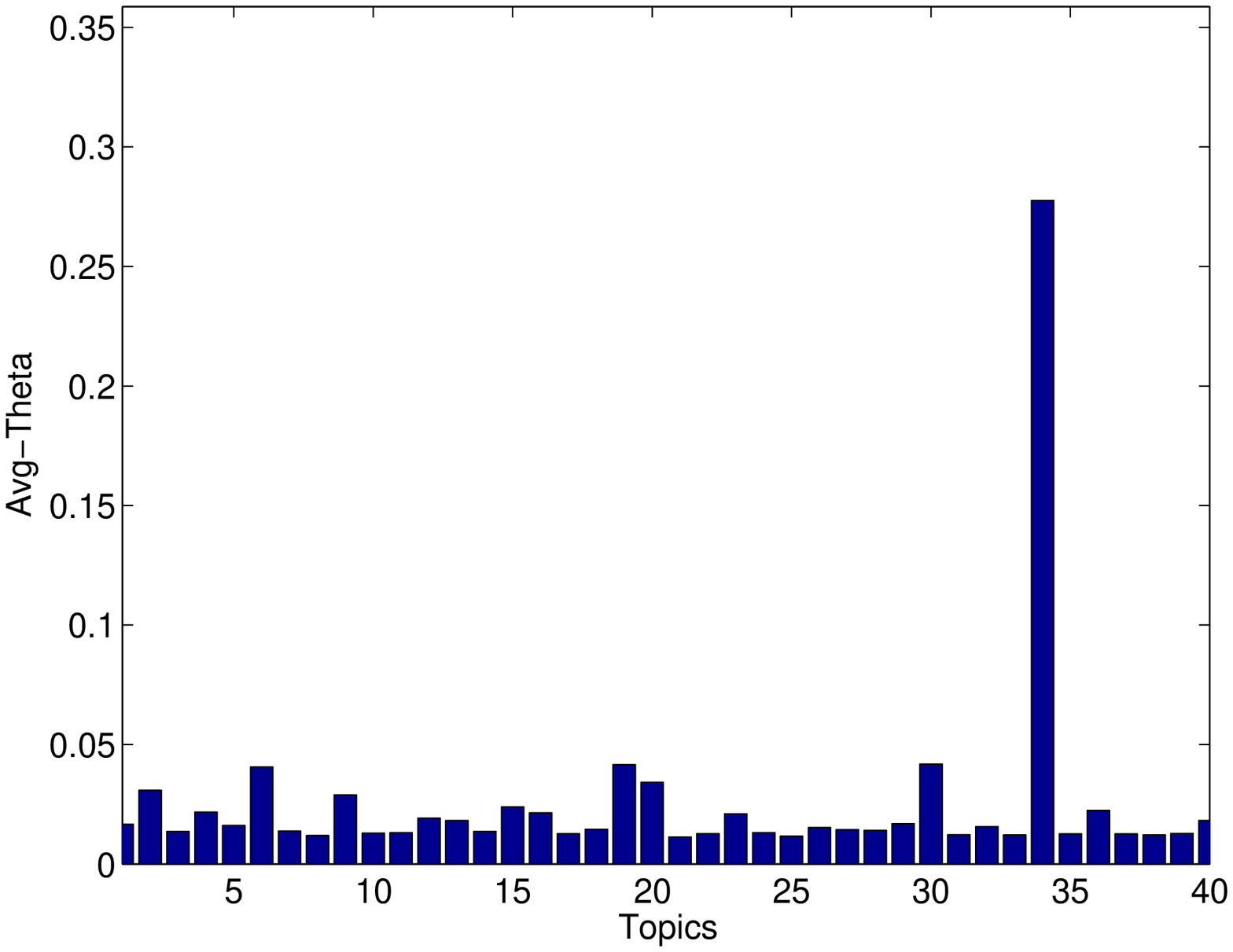}}\vspace{-.13cm}
\hfill \subfigure[rec.motorcycles]{\includegraphics[height=1.25in,width=1.4in]{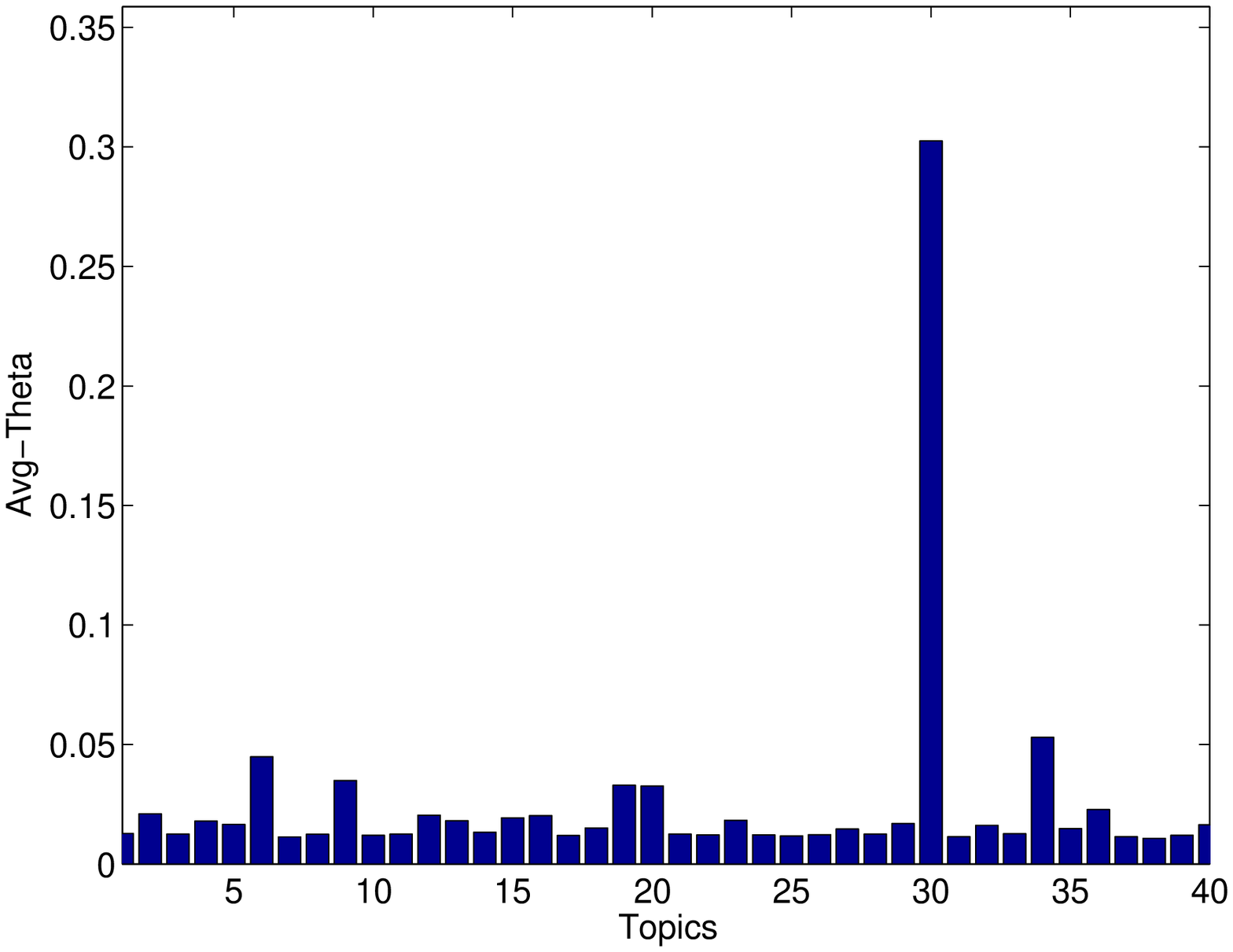}}\vspace{-.13cm}
\hfill \subfigure[rec.sport.baseball]{\includegraphics[height=1.25in,width=1.4in]{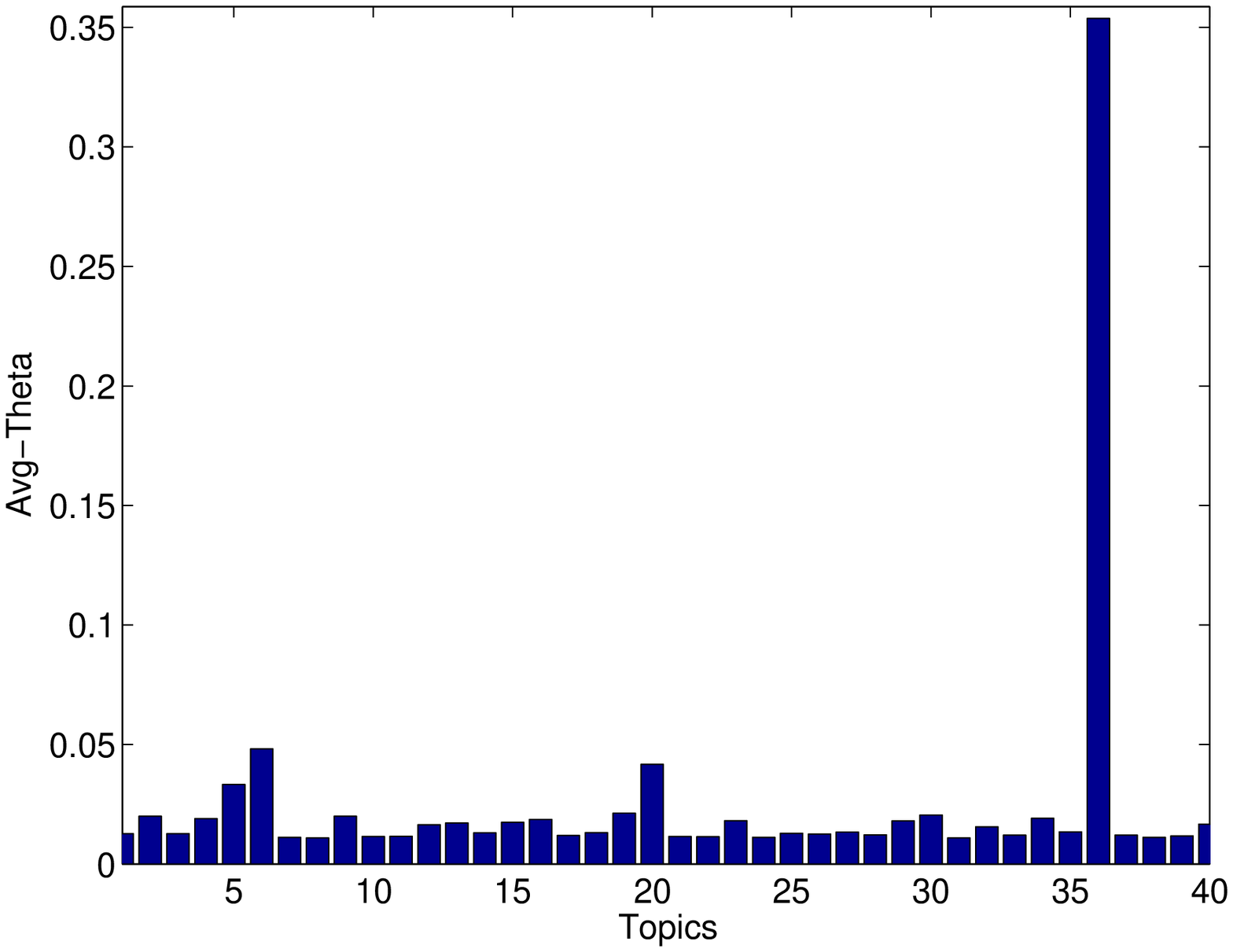}}\vspace{-.13cm}
\hfill \subfigure[rec.sport.hockey]{\includegraphics[height=1.25in,width=1.4in]{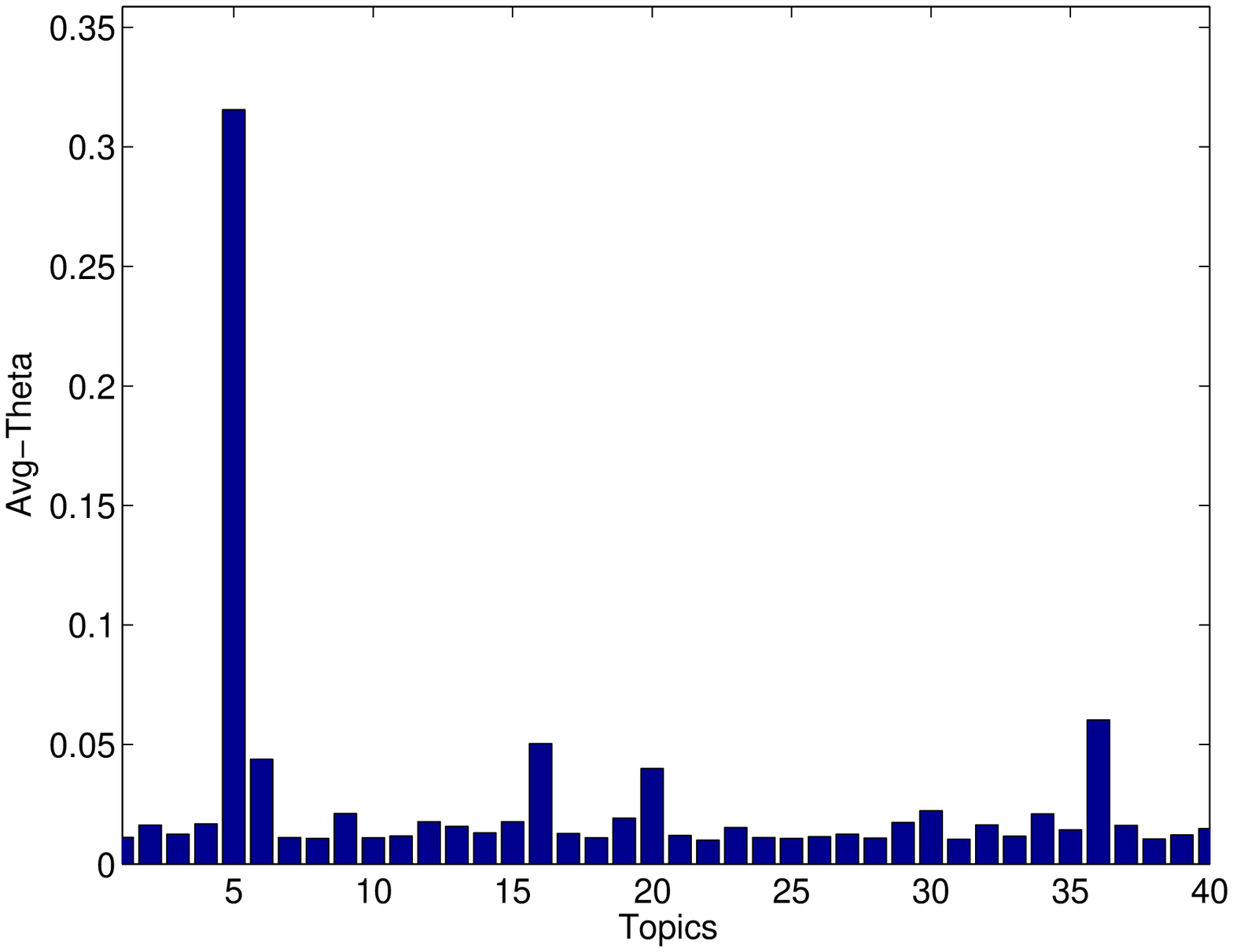}}\vspace{-.13cm}
\hfill \subfigure[sci.crypt]{\includegraphics[height=1.25in,width=1.4in]{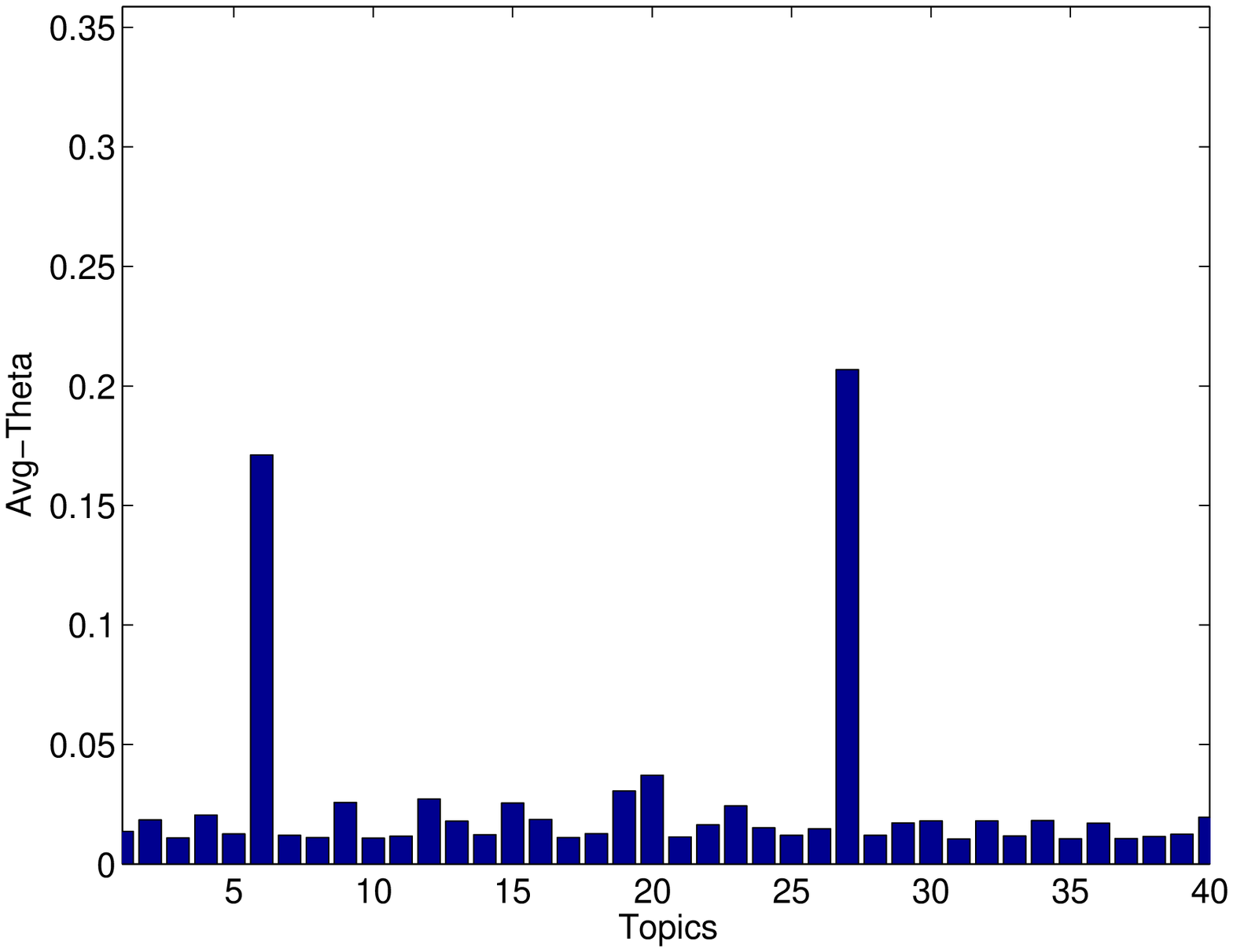}}\vspace{-.13cm}
\hfill \subfigure[sci.electronics]{\includegraphics[height=1.25in,width=1.4in]{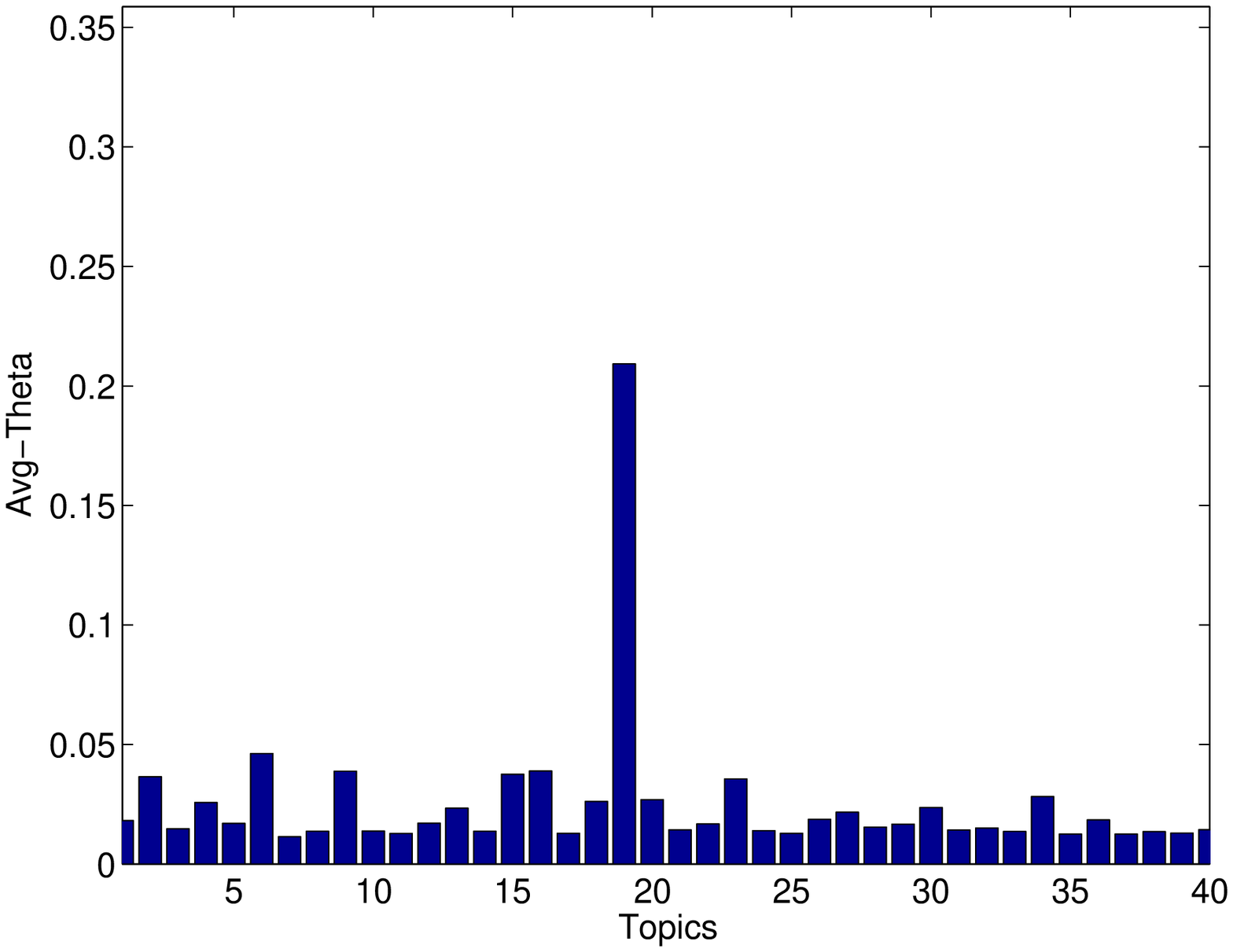}}\vspace{-.13cm}
\hfill \subfigure[sci.med]{\includegraphics[height=1.25in,width=1.4in]{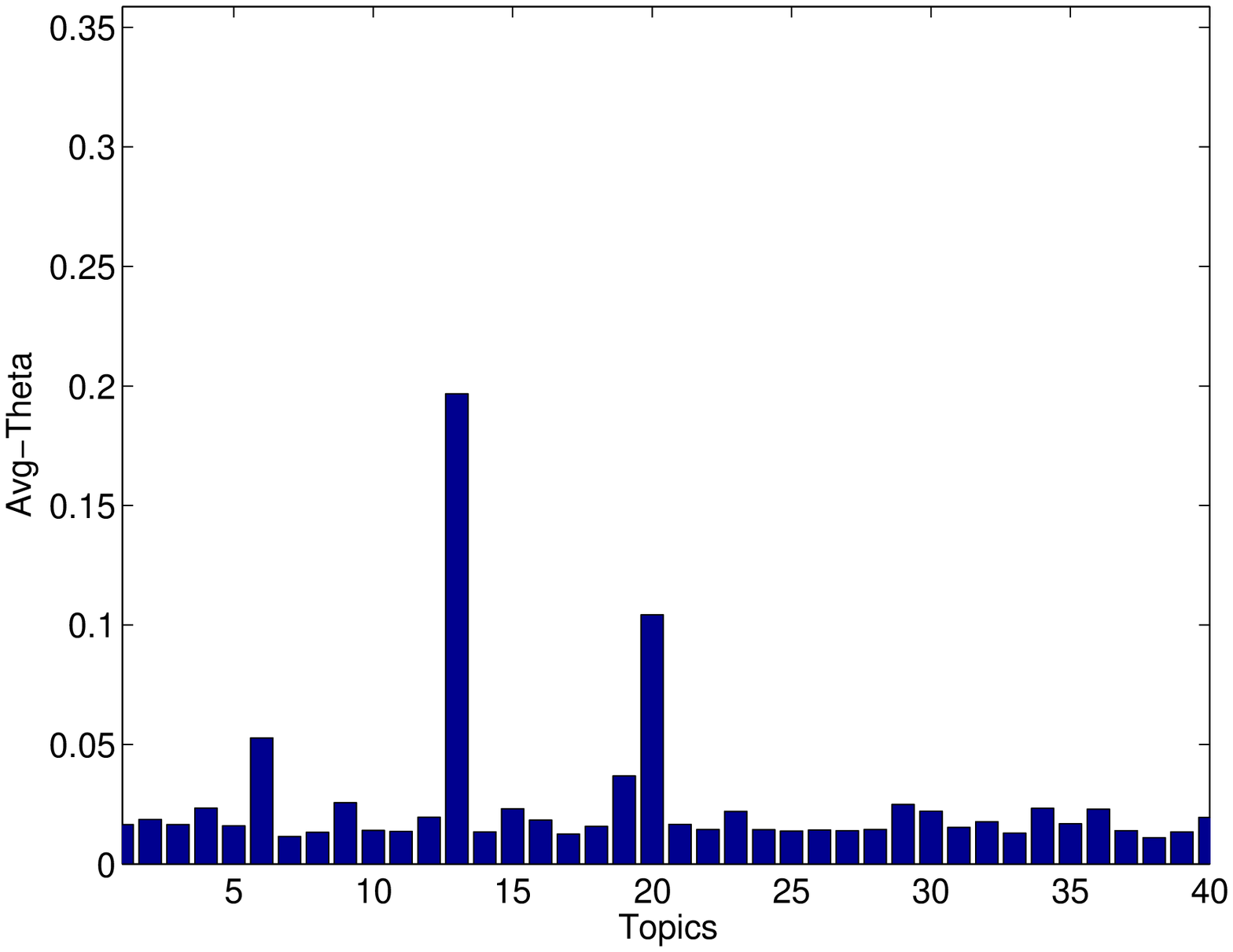}}\vspace{-.13cm}
\hfill \subfigure[sci.space]{\includegraphics[height=1.25in,width=1.4in]{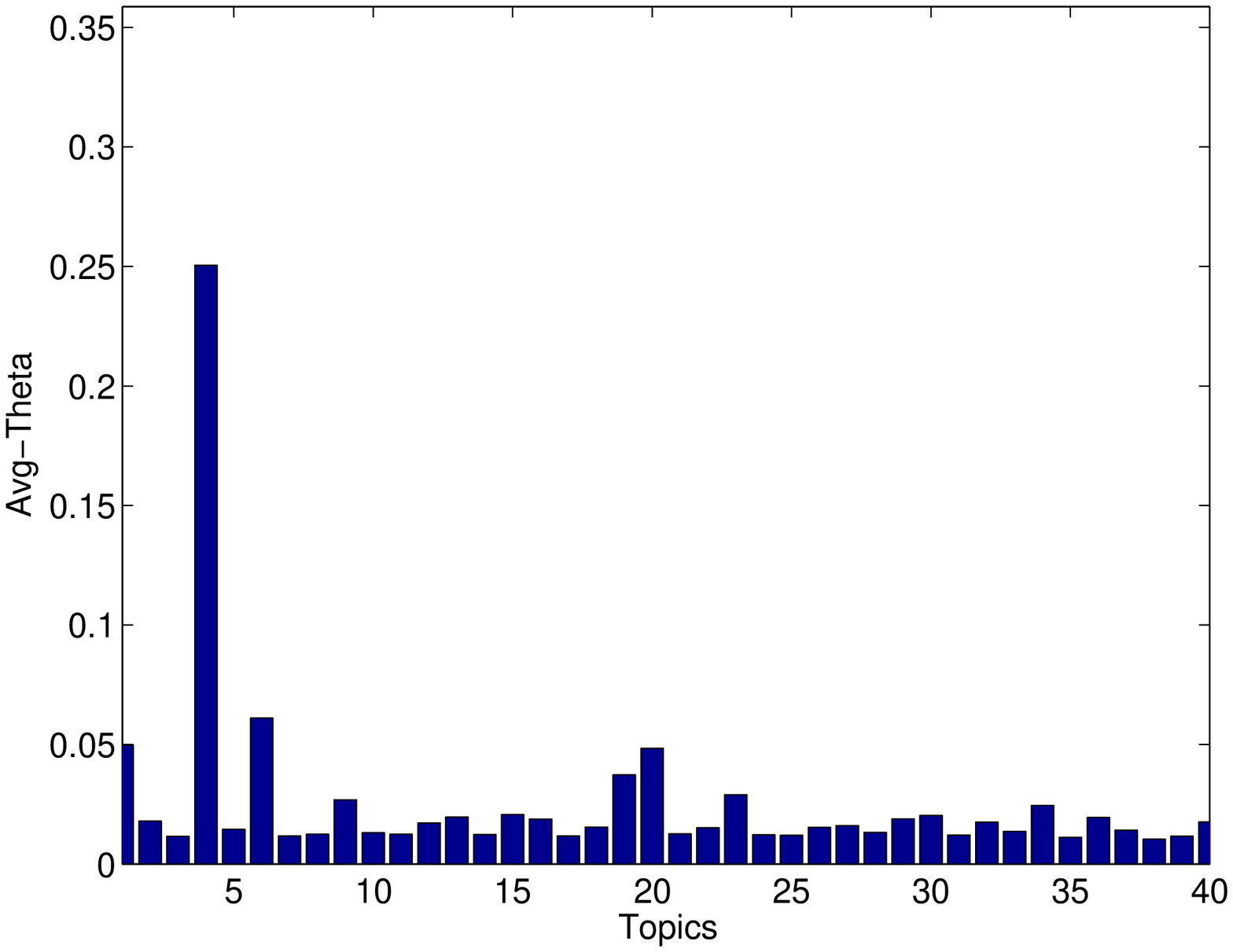}}\vspace{-.13cm}
\hfill \subfigure[religion.christian]{\includegraphics[height=1.25in,width=1.4in]{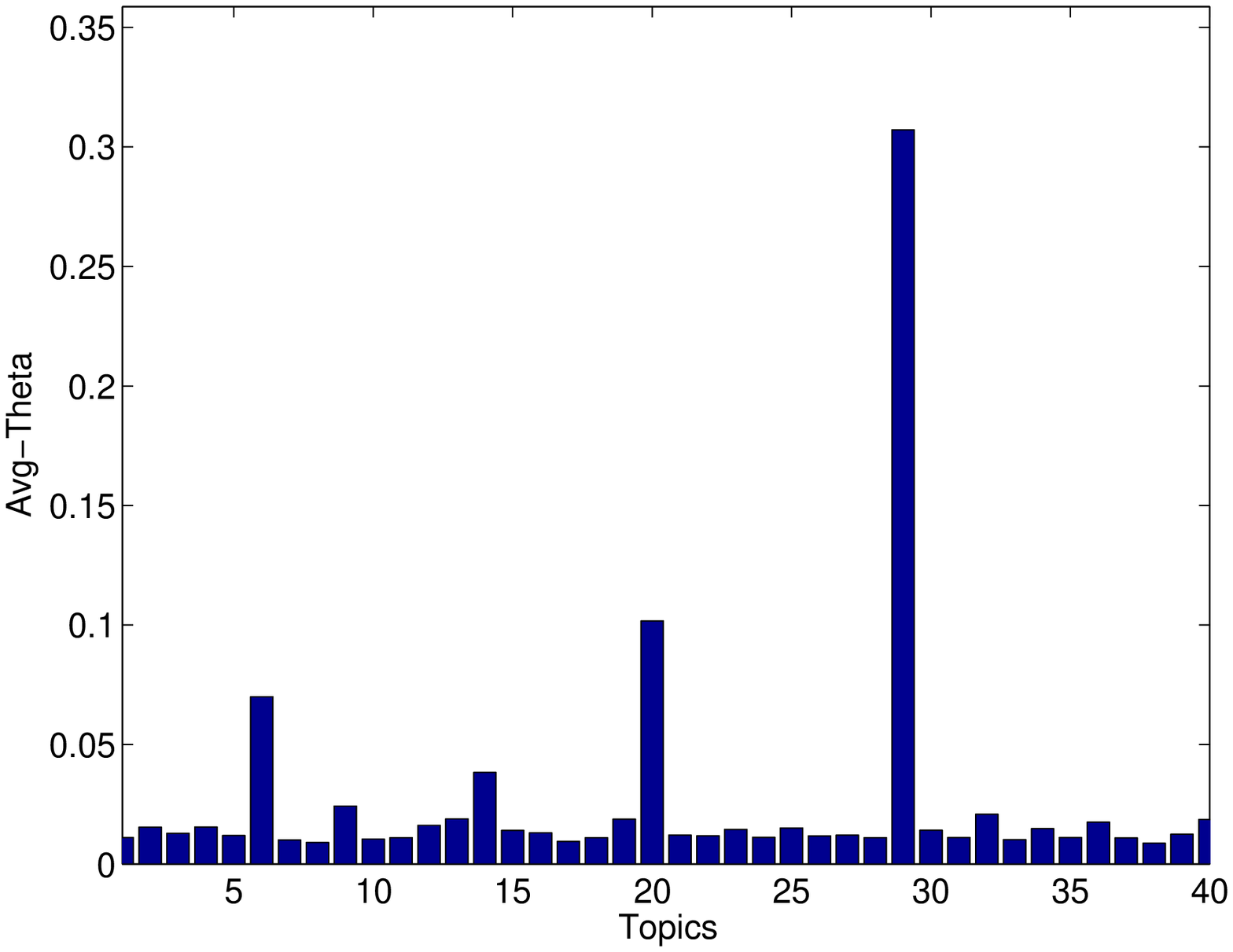}}\vspace{-.13cm}
\hfill \subfigure[politics.guns]{\includegraphics[height=1.25in,width=1.4in]{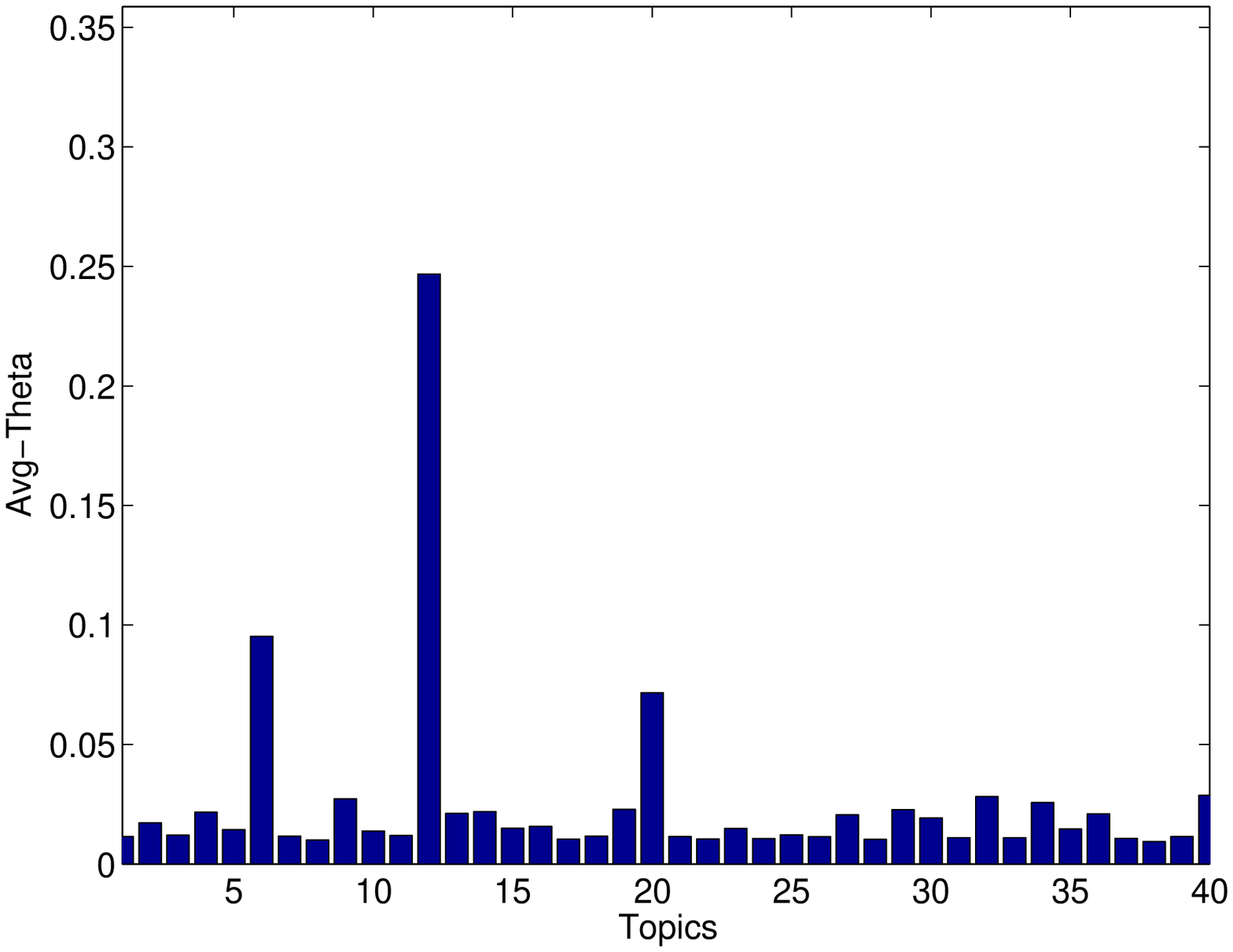}}\vspace{-.13cm}
\hfill \subfigure[politics.mideast]{\includegraphics[height=1.25in,width=1.4in]{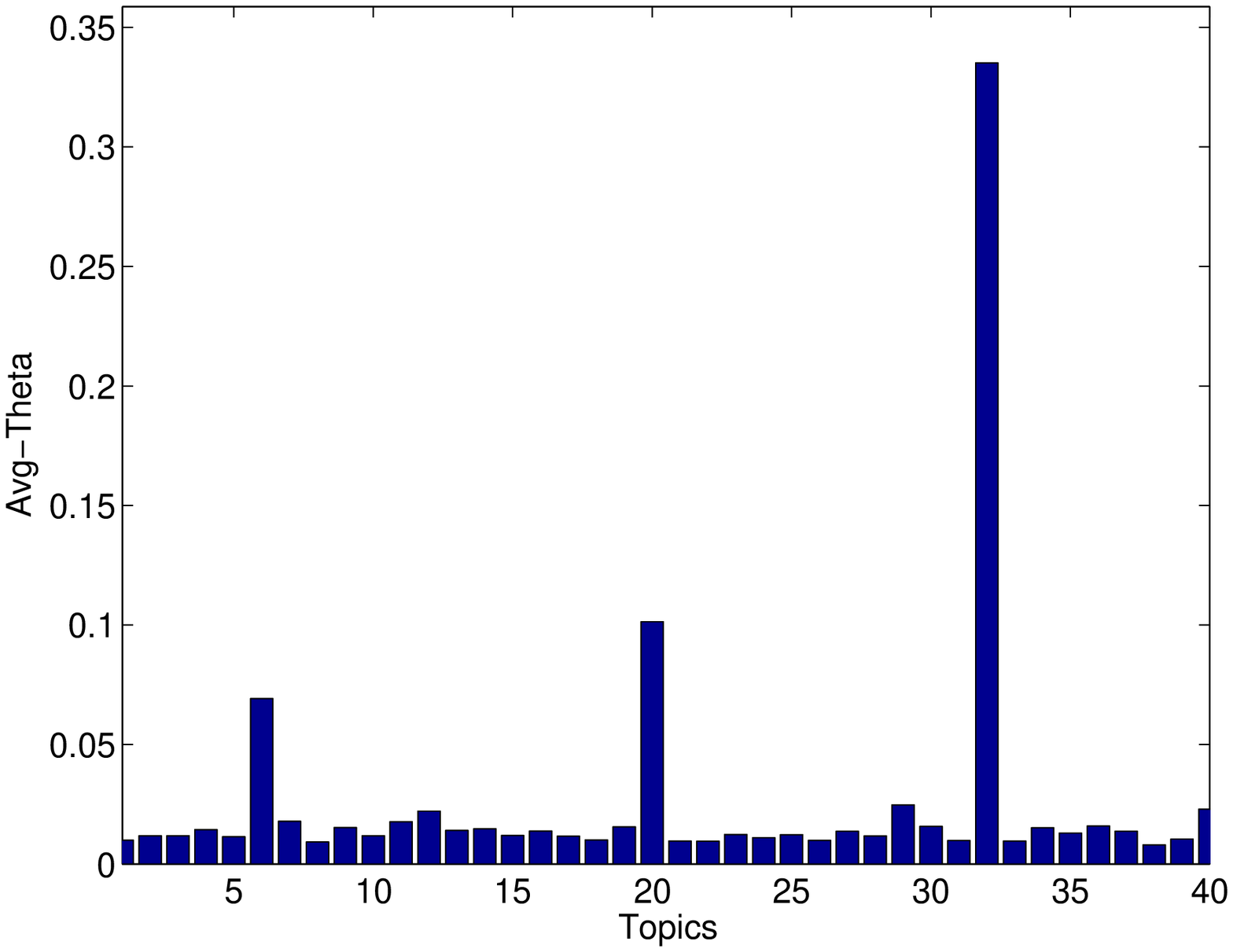}}\vspace{-.13cm}
\hfill \subfigure[politics.misc]{\includegraphics[height=1.25in,width=1.4in]{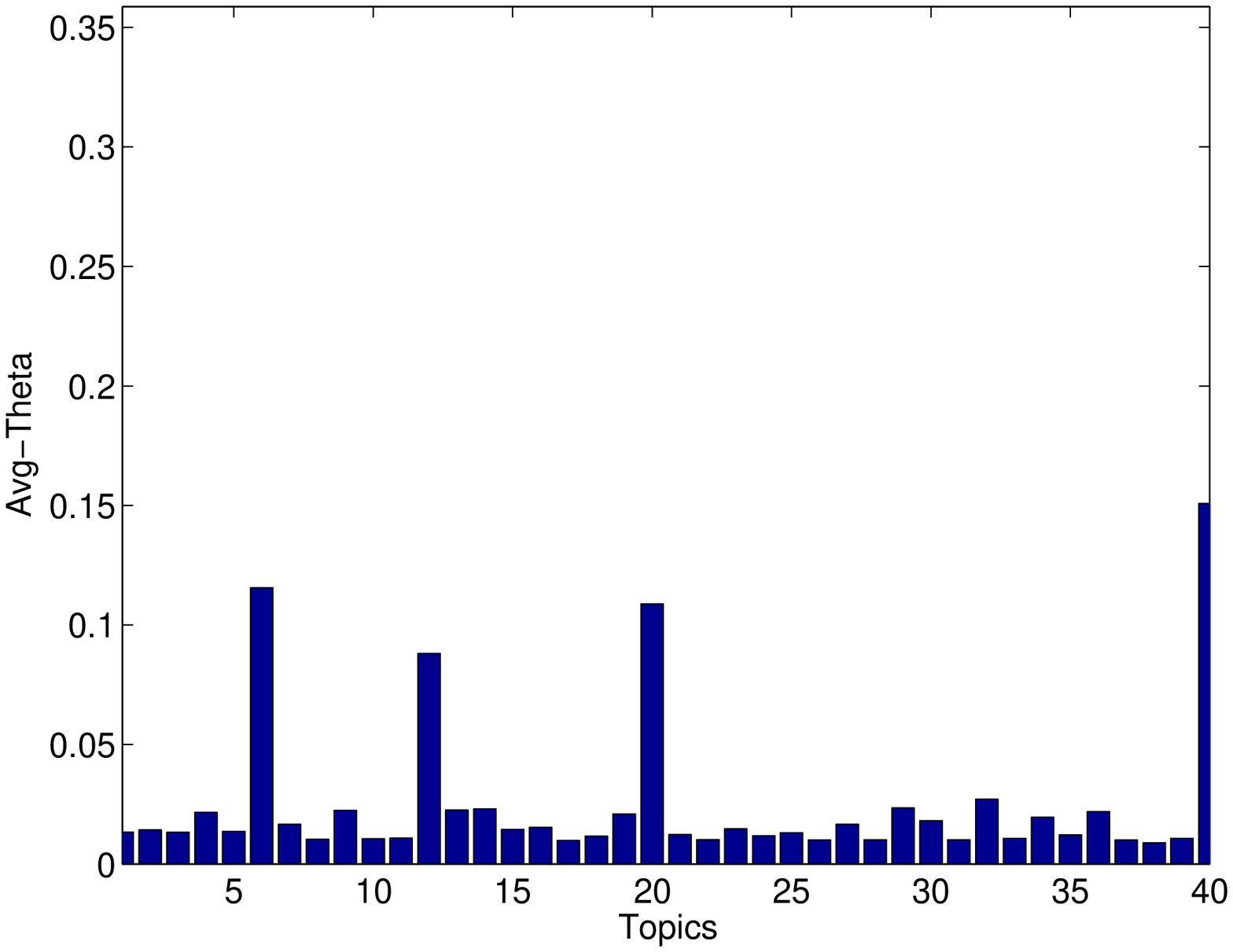}}\vspace{-.13cm}
\hfill \subfigure[religion.misc]{\includegraphics[height=1.25in,width=1.4in]{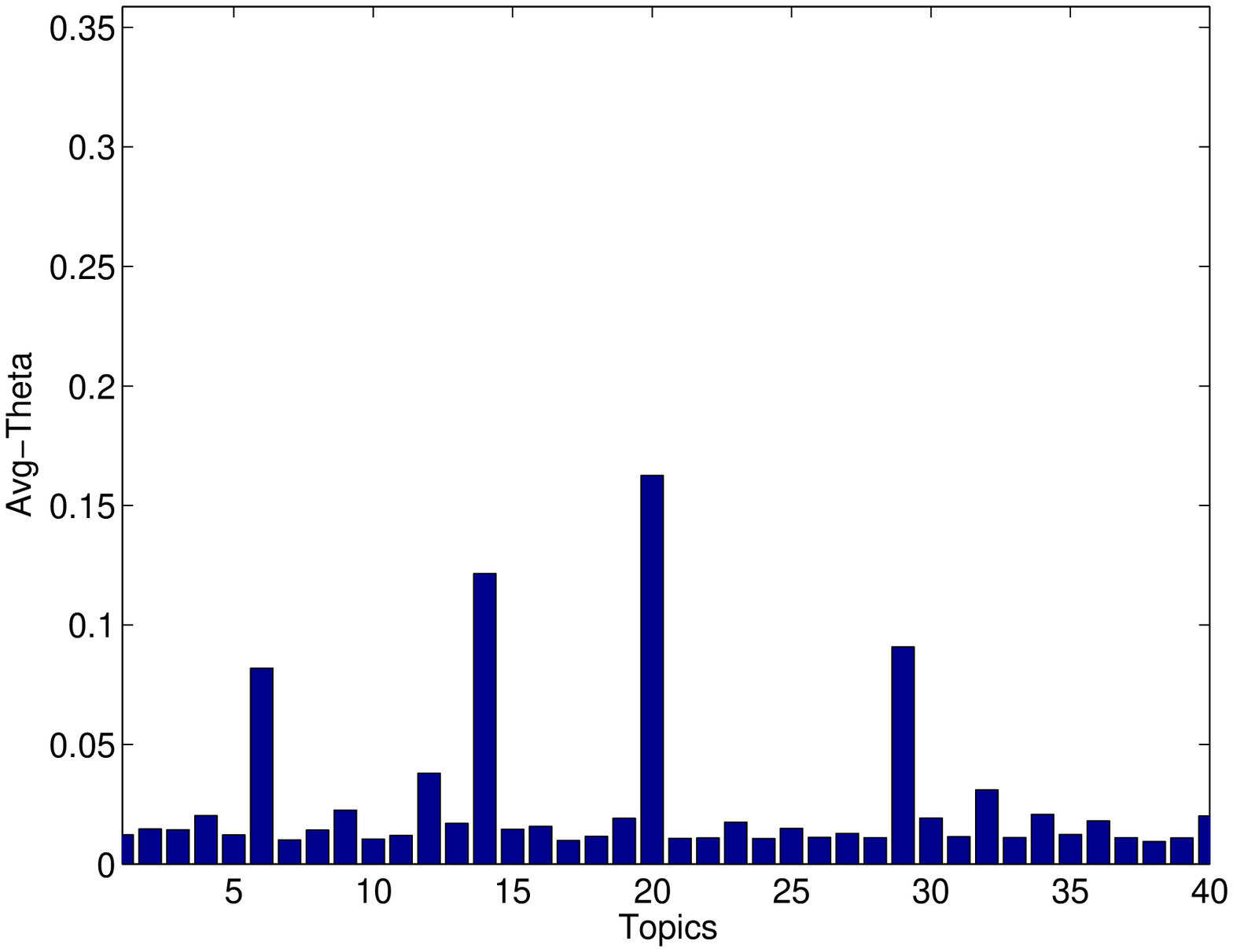}}\vspace{-.13cm}
\hfill}
\caption{(a-t) per-class average topic representations on the 20Newsgroups data set.}
\label{fig:20ng-topics}
\end{figure}

\begin{table*}[t]
\centering
\caption{The ten most probable words in the topics discovered by Multi-task Gibbs MedLDA ($K=40$) on the 20Newsgroups data set.}
\label{table:gibbs-medlda-topics}
\setlength\tabcolsep{1pt}
\scalebox{.78}
{
\begin{tabular}{!{\vrule width.5pt}
c|c|c|c|c|c|c|c|c|c
!{\vrule width.5pt}}
\hline\hline
 \textbf{Topic 1} & \textbf{Topic 2} & \textbf{Topic 3} & \textbf{Topic 4} & \textbf{Topic 5} & \textbf{Topic 6} & \textbf{Topic 7} & \textbf{Topic 8} & \textbf{Topic 9} & \textbf{Topic 10}\\
\hline
data & sale & woman & space & team & writes & mr & db & windows & file \\
mission & offer & told & nasa & game & don & president & cs & writes & congress \\
center & shipping & afraid & launch & hockey & time & stephanopoulos & mov & article & january \\
sci & dos & building & writes & play & article & jobs & bh & file & centers \\
jpl & mail & couldn & earth & season & information & russian & si & files & bill \\
planetary & price & floor & article & nhl & number & administration & al & problem & quotes \\
mass & condition & beat & orbit & ca & people & meeting & byte & dos & hr \\
probe & good & standing & moon & games & make & george & theory & don & states \\
ames & interested & immediately & shuttle & players & work & russia & bits & run & march \\
atmosphere & sell & crowd & gov & year & part & working & larson & win & included \\
\hline
 \textbf{Topic 11} & \textbf{Topic 12} & \textbf{Topic 13} & \textbf{Topic 14} & \textbf{Topic 15} & \textbf{Topic 16} & \textbf{Topic 17} & \textbf{Topic 18} & \textbf{Topic 19} & \textbf{Topic 20}\\
\hline
organizations & gun & msg & jesus & mac & drive & ma & wiring & writes & writes \\
began & people & health & god & apple & scsi & nazis & supply & power & people \\
security & guns & medical & people & writes & mb & mu & boxes & article & article \\
terrible & weapons & food & christian & drive & card & conflict & bnr & don & don  \\
association & firearms & article & bible & problem & system & ql & plants & ground & god \\
sy & writes & disease & sandvik & mb & controller & te & reduce & good & life \\
publication & article & patients & christians & article & bus & ne & corp & current & things \\
helped & government & writes & objective & system & hard & eu & relay & work & apr \\
organized & fire & doctor & ra & bit & ide & cx & onur & circuit & evidence \\
bullock & law & research & christ & don & disk & qu & damage & ve & time \\
\hline
  \textbf{Topic 21} & \textbf{Topic 22} & \textbf{Topic 23} & \textbf{Topic 24} & \textbf{Topic 25} & \textbf{Topic 26} & \textbf{Topic 27} & \textbf{Topic 28} & \textbf{Topic 29} & \textbf{Topic 30} \\
\hline
ms & pub & image & internet & fallacy & window & key & unit & god & bike \\
myers & ftp & graphics & anonymous & conclusion & server & encryption & heads & jesus & dod \\
os & mail & file & privacy & rule & motif & chip & master & people & writes \\
vote & anonymous & files & posting & innocent & widget & clipper & multi & church & article \\
votes & archive & software & email & perfect & sun & government & led & christians & ride \\
santa & electronic & jpeg & anonymity & assertion & display & keys & vpic & christ & don \\
fee & server & images & users & true & application & security & dual & christian & apr \\
impression & faq & version & postings & ad & mit & escrow & ut & bible & ca \\
issued & eff & program & service & consistent & file & secure & ratio & faith & motorcycle \\
voting & directory & data & usenet & perspective & xterm & nsa & protected & truth & good \\
\hline
  \textbf{Topic 31} & \textbf{Topic 32} & \textbf{Topic 33} & \textbf{Topic 34} & \textbf{Topic 35} & \textbf{Topic 36} & \textbf{Topic 37} & \textbf{Topic 38} & \textbf{Topic 39} & \textbf{Topic 40} \\
\hline
  matthew & israel & courtesy & car & didn & year & south & entry & open & people \\
  wire & people & announced & writes & apartment & writes & war & output & return & government \\
  neutral & turkish & sequence & cars & sumgait & article & tb & file & filename & article \\
  judas & armenian & length & article & started & game & nuclear & program & read & tax \\
  reported & jews & rates & don & mamma & team & rockefeller & build & input & make \\
  acts & israeli & molecular & good & father & baseball & georgia & um & char & health \\
  island & armenians & sets & engine & karina & good & military & entries & write & state \\
  safety & article & pm & speed & ll & don & bay & ei & enter & don \\
  hanging & writes & automatic & apr & guy & games & acquired & eof & judges & money \\
  outlets & turkey & barbara & oil & knew & runs & ships & printf & year & cramer \\
\hline\hline
\end{tabular}
}
\end{table*}

\section{Conclusions and Discussions}\label{section:conclusions}

\newcommand{\mytilde}{\raise.17ex\hbox{$\scriptstyle\mathtt{\sim}$}}

We have presented Gibbs MedLDA, an alternative approach to learning max-margin supervised topic models by minimizing an expected margin loss. We have applied Gibbs MedLDA to various tasks including text categorization, regression, and multi-task learning. By using the classical ideas of data augmentation, we have presented simple and highly efficient ``augment-and-collapse" Gibbs sampling algorithms, without making any restricting assumptions on posterior distributions. Empirical results on real data demonstrate significant improvements on time efficiency and classification accuracy over existing max-margin topic models. Our approaches are applicable to building other max-margin latent variable models, such as the max-margin nonparametric latent feature models for link prediction~\citep{Zhu:ICML12} and matrix factorization~\citep{Xu:NIPS12}. Finally, we release the code for public use\footnote{Available at: http://www.ml-thu.net/$\mytilde$jun/gibbs-medlda.shtml.}.

The new data augmentation formulation without any need to solve constrained sub-problems has shown great promise on improving the time efficiency of max-margin topic models. For future work, we are interested in developing highly scalable sampling algorithms (e.g., using a distributed architecture)~\citep{Newman08distributedinference,Smola:vldb10,Ahmed:wsdm12} to deal with large scale data sets. One nice property of the sampling algorithms is that the augmented variables are local to each document. Therefore, they can be effectively handled in a distributed architecture. But, the global prediction model weights bring in new challenges. Some preliminary work has been investigated in~\citep{Zhu:ParallelMedLDA13}. Another interesting topic is to apply the data augmentation technique to deal with the multiclass max-margin formulation, which was proposed by~\citet{Crammer:01} and used in MedLDA for learning multi-class max-margin topic models. Intuitively, it can be solved following an iterative procedure that infers the classifier weights associated with each category by fixing the others, similar as in polychomotous logistic regression~\citep{Holmes:BA06}, in which each substep may involve solving a binary hinge loss and thus our data augmentation techniques can be applied. A systematical investigation composes our future work.

\acks{This work is supported by National Key Foundation R\&D Projects (No.s 2013CB329403, 2012CB316301), Tsinghua Initiative
Scientific Research Program No.20121088071, and the 221 Basic Research Plan for Young Faculties at Tsinghua University.}

\bibliographystyle{plain}
\bibliography{medlda}
\end{document}